\documentclass{article}

    \usepackage[preprint]{neurips_2025}

\usepackage[utf8]{inputenc} % allow utf-8 input
\usepackage[T1]{fontenc}    % use 8-bit T1 fonts
\usepackage{hyperref}       % hyperlinks
\usepackage{url}            % simple URL typesetting
\usepackage{booktabs}       % professional-quality tables
\usepackage{amsfonts}       % blackboard math symbols
\usepackage{nicefrac}       % compact symbols for 1/2, etc.
\usepackage{microtype}      % microtypography
\usepackage{microtype}
\usepackage{subcaption}
\usepackage{graphicx}
\usepackage{float}
\usepackage{amsmath,amsfonts,amssymb}
\usepackage{graphicx}
\usepackage{booktabs} % for professional tables
\usepackage{array}
\usepackage{multirow}
\usepackage[table,xcdraw]{xcolor}
\usepackage{hyperref}
\usepackage{pdflscape}
\usepackage[skip=5pt]{caption} 
\usepackage{amsmath}
\usepackage{amssymb}
\usepackage{mathtools}
\usepackage{amsthm}
\usepackage{wrapfig} % 提供文字环绕功能
\usepackage[capitalize,noabbrev]{cleveref}

\theoremstyle{plain}
\newtheorem{theorem}{Theorem}[section]

\newtheorem{lemma}[theorem]{Lemma}

\theoremstyle{definition}

\newtheorem{assumption}[theorem]{Assumption}
\theoremstyle{remark}

\usepackage{algorithm}      % Required for algorithm environment
\usepackage{algorithmic}
\usepackage[textsize=tiny]{todonotes}

\title{Feature Fitted Online Conformal Prediction for Deep Time Series Forecasting Model}

\author{%
  Xiannan Huang\\
  College of Transportation\\
  Tongji University\\
  4800 Cao’an Road, Shanghai, 201804, China \\
  \texttt{huang\_xn@tongji.edu.cn} 
  \And
  Shuhan Qiu\thanks{Corresponding Author}  \\
  College of Transportation\\
  Tongji University\\
  4800 Cao’an Road, Shanghai, 201804, China \\
  \texttt{2011374@tongji.edu.cn} 
}

\begin{document}

\maketitle

\begin{abstract}
Time series forecasting is critical for many applications, where deep learning-based point prediction models have demonstrated strong performance. However, in practical scenarios, there is also a need to quantify predictive uncertainty through online confidence intervals. Existing confidence interval modeling approaches building upon these deep point prediction models suffer from key limitations: they either require costly retraining, fail to fully leverage the representational strengths of deep models, or lack theoretical guarantees. To address these gaps, we propose a lightweight conformal prediction method that provides valid coverage and shorter interval lengths without retraining. Our approach leverages features extracted from pre-trained point prediction models to fit a residual predictor and construct confidence intervals, further enhanced by an adaptive coverage control mechanism. Theoretically, we prove that our method achieves asymptotic coverage convergence, with error bounds dependent on the feature quality of the underlying point prediction model. Experiments on 12 datasets demonstrate that our method delivers tighter confidence intervals while maintaining desired coverage rates. Code, model and dataset in \href{https://github.com/xiannanhuang/FFDCI}{Github}

\end{abstract}

\section{Introduction}
Time series forecasting is critical in various real-world applications, including weather prediction \cite{gruca2023weather4cast}, traffic management \cite{kadiyala2014vector}, and disease control \cite{morid2023time}. Driven by these pressing needs, recent years have seen a surge in deep learning-based models \cite{han2024softs,liuitransformer,yurevitalizing,zhou2021informer} that significantly enhance forecasting accuracy. However, it is well-known that time series forecasting tasks are inherently influenced by random factors, making theoretically perfect predictions unattainable. As a result, alongside point forecasts, the inclusion of confidence intervals has become a meaningful area of research. When models cannot offer precise predictions, confidence intervals allow users to account for uncertainty. 
% For instance, numerous traffic management models \cite{zhao2025research,chen2023target} rely on prediction intervals for future traffic volume as inputs and employ robust optimization techniques to identify scheduling solutions that perform well across various scenarios. Similarly, in healthcare, when predicting patients' future conditions, high uncertainty may prompt doctors to conduct additional tests to mitigate the risk \cite{loftus2022uncertainty}. 
% 这两个例子不错，但是描述的方式不够好，跟下面接的不好，请改一下，例子（1）要体现出Efficiency，比如加一句“研究表明越short的区间，优化的效果越好”例子（2）的用词要体现出validity，比如把mitigate the risk改成提升置信度，然后就是（1）和（2）的顺序要对调
For instance, in healthcare, uncertainty in predicting patient outcomes may prompt clinicians to adopt more conservative diagnostic strategies \cite{loftus2022uncertainty}.
Moreover, in traffic management, numerous models \cite{zhao2025research,chen2023target} rely on confidence intervals of future traffic volume as inputs, and use robust optimization techniques to find scheduling solutions that perform well across a range of potential scenarios.
% \textcolor{red}{Research has shown} that shorter prediction intervals in such models can lead to better optimization outcomes, making efficiency a critical factor in their design.

Therefore, the requirements for confidence intervals are twofold: (1) \textbf{Validity}, \textbf{ensuring the empirical coverage asymptotically converges to the nominal level}, and (2) \textbf{Efficiency}, minimizing the interval width while maintaining validity.
% valid 不是emprirical converage to

Existing methods for generating confidence intervals for deep learning models, such as model ensembles \cite{lakshminarayanan2017simple} and quantile regression \cite{bailie2024quantile}, often rely on stringent conditions \textcolor{red}{(e.g., i.i.d. data or correct model specification)} to guarantee validity \cite{schweighofer2023quantification}. Additionally, in time series forecasting, distribution shifts frequently occur \cite{gibbs2024conformal}, which can render confidence intervals generated by models trained on historical data invalid when applied to future data. \textbf{Moreover, many existing methods \cite{lakshminarayanan2017simple,bailie2024quantile} including above mentioned ones require retraining, which can be computationally expensive.} This issue is further exacerbated with recent approaches that apply large language models to time series forecasting \cite{gruver2024large,ye2024survey}, where retraining can result in substantial computational burdens that may be unacceptable in practice.

\definecolor{customblue}{HTML}{4472C4}
As a result, our goal is to \textcolor{black}{\textit{provide valid and efficient confidence intervals for a given deep-learning-based time series forecasting model, even under distribution shift}}. To achieve this, we explore the use of conformal prediction \cite{shafer2008tutorial}, which constructs confidence intervals by treating errors in the validation set as proxies for errors in the test set. This approach has two key advantages: it does not require retraining, and its coverage is theoretically guaranteed \cite{lei2018distribution}. Therefore, this paper will focus on conformal prediction and its application in time series forecasting.

While conformal prediction has been applied to time series forecasting \cite{gibbs2024conformal,xu2021conformal}, existing work primarily targets classical statistical models (e.g., linear regression). There is a noticeable gap in research applying conformal prediction to more advanced deep learning-based forecasting models. 
\textcolor{black}{Directly applying these conformal prediction methods to deep learning models is feasible. However, such approaches \cite{gibbs2024conformal,xu2024conformal} do not leverage sample-specific information and may fail to account for their intrinsic characteristics. For instance, some samples are easier to predict and should have shorter confidence intervals, while others may require longer intervals. Besides, the features extracted by a deep learning model could potentially encapsulate such information \cite{feldman2023calibrated}, and more efficient confidence intervals might be achieved if these features are utilized appropriately. 
% Furthermore, there is limited research \cite{sun2022copula} focusing on more complex time series forecasting problems, such as multivariate, multi-step forecasting, and simply repeating univariate methods to handle multivariate or multi-step forecasting problems could ignore the correlations between different dimensions and different forecasting steps \cite{xu2024conformal}.
}
% 这段需要再优化一下，overlook具体是为啥？多变量多步是因为前面的conformal prediction methods都是单步单变量吗，前后逻辑不够紧凑。如果这个逻辑链很长，考虑将这一段观点放到文献综述里面。

As a result, we propose to 
% leverage the features extracted by the deep learning model as an additional source of information. Specifically, we
train a quantile prediction model on validation set, using features from the point prediction model as inputs and predicting quantiles of errors. 
% The advantages of this approach are twofold: 
% \begin{itemize}
%     \item By considering additional information, we can obtain more efficient confidence intervals by leveraging the representation power of deep learning models effectively.
%     \item Since most deep learning models ultimately use a fully connected layer to map features to prediction values \cite{han2024softs,liuitransformer,yurevitalizing}, we can reasonably assume that, once the features are given, the predictions for each dimension and prediction step can be considered independent. In other words, the feature extraction module decouples the prediction problem for each dimension and prediction step.
% \end{itemize}
In addition, to address the potential distribution shift in time series forecasting tasks, we considered using a method to dynamically adjust the confidence interval length, similar to \cite{gibbs2021adaptive}. Furthermore, we theoretically prove that the coverage of confidence intervals provided by our method can converge to the desired level as the deployment time increases, even in the presence of distribution shifts (Theo \ref{all cover}). Finally, we also theoretically demonstrate that the approach of using features to predict the confidence intervals offers certain advantages (Theo \ref{MSCE}).

Our key contributions are:

\textbf{Algorithmic Contributions.} We propose \textbf{F}eature \textbf{F}itted \textbf{D}ynamic \textbf{C}onfidence \textbf{I}nterval (\textbf{FFDCI}), a lightweight conformal prediction framework that leverages deep-learning-based features for adaptive confidence interval calibration. 
% Specifically, we utilize the features extracted by the deep learning model to calibrate a quantile prediction model for errors on the validation set and adaptively adjust predicted confidence intervals during deployment. Experimental results demonstrate that our method can provide valid and relatively shorter confidence intervals.
% And we release the code in anonymous github: https://anonymous.4open.science/r/FFDCI-2B10.

\textbf{Theoretical Contributions.} We prove that FFDCI guarantees asymptotic coverage convergence under mild assumptions. Additionally, we analyzed the following metric "Mean Absolute Coverage Error (MACE)”:
\begin{equation}
    MACE=\frac{\sum_{t=1}^T |P(y_t\in C_t)-(1-\alpha)|}{T}
\end{equation}
Where $C_t$ is the predicted confidence interval in time step $t$, $y_t$ is the true value in time $t$ and $1-\alpha$ is the target coverage we want to control (e.g., 90\%). This metric reflects the cumulative difference between the actual coverage rate and the target coverage rate that we want to control.
% It is somehow like $MAE$ (mean absolute error) in classical regression problem, with the true value becoming $1-\alpha$ and the prediction value becoming actual coverage.
% We have proven that the upper bound of this metric is related to the fitting performance of confidence interval prediction model. In other words,
We prove that this metric can be controlled at a small level if the features extracted by the deep learning models are relevant to the quantiles of prediction errors.

\section{Method}
\subsection{Key idea}
Many works frame online conformal prediction as an online learning problem and employ online (sub)gradient descent methods \cite{zhang2024benefit,zhangdiscounted,Angelopoulos2024OnlineCP}. Specifically, consider a simple one-dimensional, single-step problem: at each time $t$, the target is $y_t$, the predicted value is $\hat{y}_t$, and the prediction interval is $[\hat{y}_t - q_t, \hat{y}_t + q_t]$. Define $s_t = |y_t - \hat{y}_t|$ and the $(1-\alpha)$-quantile loss as
\begin{equation}
l_t(q_t) = \mathbb{E}_{s_t}\left[(s_t - q_t)\left(\mathbb{I}(s_t > q_t) - \alpha\right)\right].
\label{eq:onlineloss}
\end{equation}
The online (sub)gradient descent update rule then becomes:
\begin{equation}
q_{t+1} = q_t - \gamma \nabla l_t(q_t) = q_t - \gamma \mathbb{E}_{s_t}\left(\mathbb{I}(s_t > q_t) - \alpha\right).
\label{eq:update}
\end{equation}

Importantly, the optimal value minimizing the per-step $(1-\alpha)$-quantile loss at each $t$ is the $(1-\alpha)$-quantile of $s_t$. Thus, this gradient-based algorithm inherently aims to approximate the true $(1-\alpha)$-quantile of $s_t$ at each step (donated as $q_t^*$). Existing works focus on improving the online gradient method through adaptive learning rates \cite{gibbs2024conformal} or smoother surrogate loss functions \cite{wu2025errorquantified}.

\textbf{Our approach diverges by simplifying the online optimization problem.} Suppose we have a predictor for the $(1-\alpha)$-quantile of $s_t$, denoted as $\hat{q}_t$. Instead of directly learning $q_t^*$, we aim to learn the \emph{residual} between $q_t^*$ and its prediction $\hat{q}_t$, i.e., $\Delta_t = q_t^* - \hat{q}_t$. Comparing these two optimization problems:
\begin{enumerate}
    \item Original problem: Optimal solution is $q_t^*$, the $(1-\alpha)$-quantile of $s_t$.
    \item Residual problem: Optimal solution is $\Delta_t$, the difference between $q_t^*$ and its prediction $\hat{q}_t$.
\end{enumerate}

The residual problem is inherently simpler \textbf{as long as} the predictor $\hat{q}_t$ is reasonably accurate. When $\hat{q}_t$ closely approximates $q_t^*$, the residual $\Delta_t$ exhibits smaller fluctuations than $q_t^*$ itself. This reduced variability makes the optimization easier.

To obtain $\hat{q}_t$, we leverage the features extracted by deep learning models, which can be regarded as informative representations of the input data. By combining these features with quantile regression techniques, we can estimate $q_t^*$ efficiently.
\subsection{Problem definition}
Suppose the data is of dimension $p$, and there is a model to predict the data in the next $d_1$ time steps. Therefore, in each step $t$, the prediction $\hat{y}_t$ is a matrix in $\mathbb{R}^{p\times s}$. And $y_{t,i,j}$ is the value of $i$ dimension, $j$ prediction step in time $t$. We can further assume that the point prediction model consists of two parts: the first part can be regarded as a feature extraction module, while the second part generates predictions for each dimension and each prediction step based on the features. Specifically, if the input is historical data from the past $l$ time steps, represented as an \(p \times l\) matrix, the feature extraction module outputs a matrix $z$ of shape $p\times d_2$, $d_2$ means the dimension of feature. Then second part of model uses feature to output prediction $\hat{y}$. For many models \cite{han2024softs,liuitransformer,yurevitalizing}, we can regard the last MLP layer as prediction head and the former layers as feature extractor.

If the algorithm is deployed for $T$ steps, we want to provide confidence interval $C_{t,i,j}$ of $\hat{y}_{t,i,j}$ such that:

\begin{equation}
    \frac{\sum_{t,i,j}{\mathbb{I}(y_{t,i,j}\in C_{t,i,j})}}{T\times s\times p} \to 1-\alpha\  \ when\ T \to \infty
    \label{average_cov_required}
\end{equation}
Where $\mathbb{I}(y_{t,i,j}\in C_{t,i,j})$ is indicator function, if $y_{t,i,j}\in C_{t,i,j}$, the value of it is 1, and it equals to 0 otherwise. The meaning of this equation is that, the proportion of true values covered by confidence intervals will converge to the target level as time progresses. Furthermore, based on previous studies \cite{lin2022conformal,batra2023conformal}, it is also required that the coverage rate for each dimension and each prediction step should converge to the target level, as the following equation shows:
\begin{equation}
    \frac{\sum_{t}{\mathbb{I}(y_{t,i,j}\in C_{t,i,j})}}{T} \to 1-\alpha\ when\ T \to \infty, for\ any\ i,j \label{ave_cov}
\end{equation}

It is easy to show that if (\ref{ave_cov}) holds, then (\ref{average_cov_required}) follows.

Besides, (\ref{ave_cov}) is not enough. For example, for specific $i$ and $j$, an algorithm that provides intervals with infinite length for 90\% cases and provides empty set for the other 10\% cases can also satisfy (\ref{ave_cov}) when $\alpha = 0.1$. But this algorithm is obviously inappropriate. Therefore we also want to control the following $MACE$ for any $i$ and $j$ to be small.
\begin{equation}
     MACE_{i,j}=\frac{\sum_{t=1}^T |P(y_{t,i,j}\in C_{t,i,j})-(1-\alpha)|}{T}
\end{equation}
This requirement means the probability of each $y_{t,i,j}$ being in the predicted confidence interval is close to $1-\alpha$.

\begin{figure}[H] % h表示尽量放在这里
    \centering
    \includegraphics[width=0.75\linewidth]{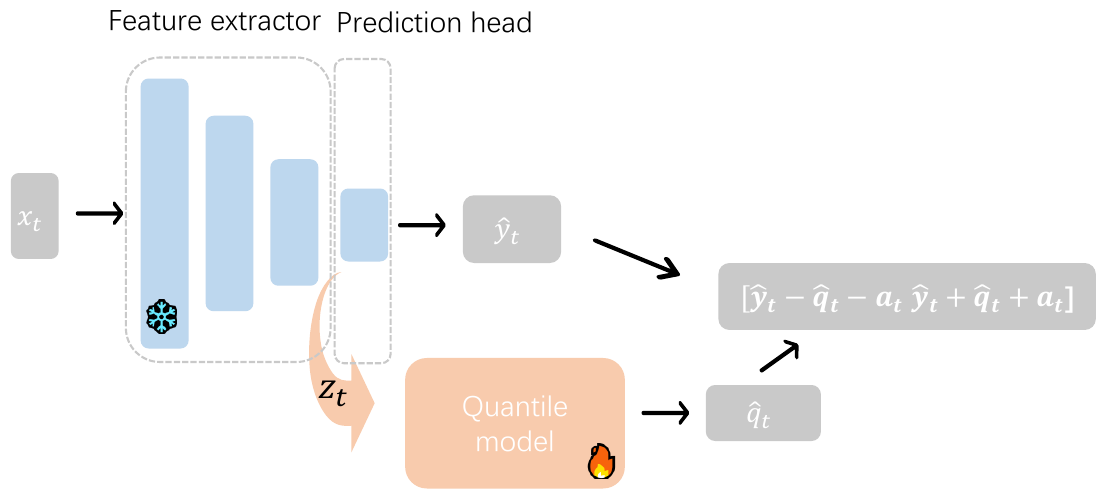}
    \caption{The work flow of our method}
    \label{fig:workflow}
\end{figure}
\subsection{Fit quantile model for errors}
The workflow of our method is illuminated in Figure \ref{fig:workflow}. First, we need to fit a model to predict the quantiles of errors using features. To accomplish that, the features and errors in the validation set are collected. Suppose the validation set contains \(n_1\) samples, and the prediction for each sample is a \(p \times d_1\) matrix. Then, the error for each sample can be expressed as: 
\begin{equation}
    s=|\hat{y}-y|
\end{equation}
and it is also a matrix.

Since most models \cite{wu2022timesnet,han2024softs,liuitransformer,yurevitalizing} use a fully connected layer to transform the feature matrix into the prediction matrix, the extracted feature can be considered crucial for the final prediction result, as it summarizes information from all previous time steps and dimensions. Therefore, it is reasonable to learn a function that maps the feature to the quantile of prediction error $s$. Specifically, we aim to fit a model, where the input is a feature matrix $z$ with shape \(p \times d_2\), and the output is a \(p \times d_1\) matrix. Each element of this matrix corresponds to the \(1-\alpha\) quantile of the prediction error for each dimension and prediction step. For example, an MLP can be use to complete this task.

Therefore, we can use the features and errors from the validation set to train the prediction model for error quantiles based on the pinball loss, which is:
\begin{equation}
    l_{\alpha}(s,\hat{q})=max((1-\alpha)(s-\hat{q}),\alpha(\hat{q}-s))
\end{equation}
Where $\hat{q}$ is predicted quantile and this loss is computed element-wise. 

% The rationale behind this approach is that, in multi-step time series prediction, the uncertainties at different steps are recursively related. Specifically, if entropy $H(y_j)=\int log(y_j)p(y_j)dy$ is used to represent the uncertainty at each step, the following recursive formula holds:
% \begin{theorem}
%     suppose a feature is $z$, and $y_j$ is the value in the next $j$ step, then:
%     \begin{equation}
%         H(y_j|z)=H(y_{j-1}|z)+E_{y_{j-1}\sim p(y_{j-1}|z)}H(y_j|z,y_{j-1})
%     \end{equation}
%    The above equation holds for any dimension, and we omit the index $i$. 
%    \label{chian_rule}
% \end{theorem}
% The above theorem uncover a structure like:
% \begin{equation}
%     H(y_j|x)=f(x,p(y_{j-1}|x))
% \end{equation}
% Where $f$ is a map.

% Therefore, the hidden states of the recurrent neural network can be considered as representation of \(p(y_{j-1} \mid x)\), which is updated at each time step. As a result, the recursive structure of the recurrent neural network can be regarded as an approximate of $f$, which means the recursive relationship of uncertainties between each prediction step.

\subsection{Update confidence interval during deployment}
Suppose the point prediction model outputs point estimator $\hat{y}_t$ and the quantile prediction model output the quantile of error as $\hat{q}_t$, where $\hat{y}_t,\hat{q}_t$ are both $p\times d_1$ matrix at each time step $t$. Then the most native way to obtain prediction interval is that:
\begin{equation}
C_{t,i,j}=[\hat{y}_{t,i,j}-\hat{q}_{t,i,j},\hat{y}_{t,i,j}+\hat{q}_{t,i,j}]
\label{eq:ab}
\end{equation}
However, when the quantile prediction model can not output the true quantile of error, which often occurs when there is a distribution shift or the quantile model is not correct, the above interval will be invalid. Therefore, we refer to ACI \cite{gibbs2021adaptive} and propose to dynamically adjust the interval as:
\begin{equation}C_{t,i,j}=[\hat{y}_{t,i,j}-\hat{q}_{t,i,j}-a_{t,i,j},\hat{y}_{t,i,j}+\hat{q}_{t,i,j}+a_{t,i,j}]
\end{equation}
where $a_{t,i,j}$ can be determined recurrently, as the following equation shows:
\begin{equation}
    a_{t+1,i,j}=a_{t,i,j}+\gamma (1-\mathbb{I}(y_{t,i,j}\in C_{t,i,j})-\alpha)
    \label{update a}
\end{equation}
Where $\gamma$ is a constant greater than 0. The meaning of this equation is that, if the output interval in former time step can not cover the true value, then the interval in the next step will be elongated, and if the formal interval covers the true value, the interval in the next step will be shortened. And it can regarded as lagged-online gradient decent for $1-\alpha$-quantile loss:
$$
l_t(a_t) = \mathbb{E}_{s_t}\left[(s_{t-j} - \hat{q}_{t-j}-a_{t-j})\left(\mathbb{I}(s_{t-j} > q_{t-j}+a_{t-j}) - \alpha\right)\right].
$$
In the original online conformal prediction algorithm \cite{gibbs2021adaptive,gibbs2024conformal}, the confidence interval is updated according to $\mathbb{I}(y_{t,i,j}\in C_{t,i,j})$. But in our problem, this indicator can not be observed until $t+j$ step, as a result, only $\mathbb{I}(y_{t,i,j}\in C_{t-j,i,j})$ can be used to obtain $a_{t+1,i,j}$. Besides, we set $a_{t+1,i,j}=0$ when $t<j+1$, and if $\hat{q}_{t,i,j}+a_{t,i,j}\leq0$, $C_{t,i,j}$ is empty set. The algorithm can be summarized as Algorithm \ref{alg:ci_prediction}.
\begin{algorithm}[ht]
\caption{Feature Fitted Dynamic Confidence Interval (FFDCI)}
\label{alg:ci_prediction}
\begin{algorithmic}
\REQUIRE Point prediction model $f$, validation set with $n_1$ samples, target confidence level $1-\alpha$, adjustment rate $\gamma > 0$.
\ENSURE Confidence intervals $C_{t,i,j}$ for future $T$ time steps.

\STATE \textbf{Fit Quantile Prediction Model:}
\STATE Extract feature $z \in \mathbb{R}^{p \times d_1}$ for each sample in the validation set.
\STATE Compute prediction error $s = |\hat{y} - y|$ for each sample in the validation set.
\STATE Fit a quantile prediction model for errors using pinball loss.

\STATE \textbf{Deploy Algorithm:}
\STATE Initialize $a_{t,i,j} = 0$ for $t < j+1$.
\FOR{$t = 1$ \textbf{to} $T$}
    \STATE Obtain point prediction $\hat{y}_t$ and feature $z_t$ using $f$.
    \STATE Compute error quantile $\hat{q}_t$ using quantile prediction model and $z_t$.
    \FOR{each $i \in \{1, \dots, p\}$ and $j \in \{1, \dots, d_1\}$}
        \STATE Output confidence interval:
        \STATE $C_{t,i,j} = [\hat{y}_{t,i,j} - \hat{q}_{t,i,j} - a_{t,i,j}, \hat{y}_{t,i,j} + \hat{q}_{t,i,j} + a_{t,i,j}].$
        \IF{$t \geq j+1$}
            \STATE Update adjustment term:
            \STATE $a_{t+1,i,j} = a_{t,i,j} + \gamma \big(\mathbb{I}(y_{t,i,j}\in C_{t,i,j}) - \alpha\big).$
        \ENDIF
    \ENDFOR
\ENDFOR
\end{algorithmic}
\end{algorithm}
\subsection{Benefits of the proposed approach}
\textbf{Against model based methods}:
There are some methods to provide confidence interval based on probabilistic assumption \cite{salinas2020deepar}, model ensembles \cite{lakshminarayanan2017simple}, Monte-Carlo Dropout \cite{Gal2015DropoutAA} or quantile regression \cite{bailie2024quantile}. Compared with them, there are two main benefits of our method. The first is our method can be used to any point prediction model without the need of retraining. The second is that the coverage rate of our method is guaranteed, which will be elaborated in the following section.

\textbf{Against other time series conformal prediction methods}: Compared with these methods, our \textbf{FFDCI} integrates the features extracted by deep learning models as additional information. Therefore, our method can take advantage of the powerful representational ability of deep learning models and provide better confidence intervals.
\section{Theoretical Results}
Similar to most theoretical guarantees in conformal prediction, our theoretical results below just make a few week assumptions about the data distribution.
\subsection{Bound of average coverage}
\begin{assumption}
    The errors and predicted quantiles of errors should be bounded. We assume $|y_{t,i,j}-\hat{y}_{t,i,j}|\leq M$ and $\hat{q}_{t,i,j}\leq M$ for any $t,i,j$.
    \label{bounded assumption}
\end{assumption}
This assumption is common in researches about online conformal prediction, such as  \cite{gibbs2021adaptive,gibbs2024conformal}.
\begin{theorem}[Coverage guarantee]
   For any dimension $i$ and prediction step $j$, if $C_{t,i,j}$ is the confidence interval provided by our algorithm, then we have:
    \begin{equation}
        |\frac{\sum_{t=1}^T{\mathbb{I}(y_{t,i,j}\in C_{t,i,j})}}{T}-(1-\alpha)|\leq  2(\frac{M+\gamma}{T\gamma}+\frac{j+1}{T})
    \end{equation}
    \label{all cover}
\end{theorem}

This theorem implies that as the deployment time increases, the proportion of true values covered by the confidence intervals provided by our algorithm will converge to the desired coverage rate for any prediction step and any dimension. As a result, requirement (\ref{ave_cov}) could be met.

This theorem is similar to coverage guarantee provided by ACI \cite{gibbs2021adaptive}. But because it is not able to update $a_{t,i,j}$ using $\mathbb{I}(y_{t,i,j}\in C_{t,i,j})$, there is an error term related to $j$ in addition.
\subsection{Bound of $MACE$}
As we previously mentioned, we propose a simpler online learning problem where we only need to learn the residual between the predicted quantile and the true quantile. Therefore, when this residual is smaller (indicating better fitting performance), the problem should be easier to learn - meaning that we can more accurately approximate the true desired quantile at each step. Consequently, we can achieve smaller $MACE$ values. This intuition suggests that $MACE$ should be related to the goodness-of-fit of our quantile model. Our following theoretical result will formally establish this relationship:

% \begin{assumption}
%    if ${q}_{t,i,j}$ is the actual $1-\alpha$ quantile of $|y_{t,i,j}-\hat{y}_{t,i,j}|$, and $\hat{q}_{t,i,j}$ is the estimated quantile of it, then we assume that, for any $i,j$, $(q_{t,i,j}-\hat{q}_{t,i,j})$ is independent in all time steps $t$. 
%    \label{independent}
% \end{assumption}
\begin{assumption}
    $P(y_{t,i,j}\in C_{t,i,j})$ can be regard as a function of $a_{t,i,j}$, and we assume it is a Lipschitz function with Lipschitz constant $L$.
\end{assumption}

\begin{theorem}[Bound of $MACE$]
    Under the above assumptions, for any dimension $i$ and prediction step $j$, if we use $q_{t,i,j}^*$ to express to true $1-\alpha$ quantile of $|\hat{y}_{t,i,j}-y_{t,i,j}|$, then there exists a best learning rate $\gamma$ such that:
    % \begin{equation}
    %      MSCE_{i,j}\leq\sqrt{c\sum_{t=j+1}^T{|(a_{t,i,j}-\hat{a}_{t,i,j})-(a_{t-1,i,j}-\hat{a}_{t-1,i,j})|}}+M\frac{j}{T}
    % \end{equation}
    % where ${a}_{t,i,j}$ is the value statifing: $P(y_{t,i,j}\in [\hat{y}_{t,i,j}-{a}_{t,i,j},\hat{y}_{t,i,j}+{a}_{t,i,j}])=1-\alpha$. If we further assume that, given feature $x$,$(a_{t,i,j}-\hat{a}_{t,i,j})$ is independent in all time steps, then we will have:
    \begin{equation}
         MACE_{i,j}\leq c\sqrt{\sigma(q_{i,j}^*-\hat{q}_{i,j})+\frac{M(j+1)}{T}}
    \end{equation}
    Where $\sigma(q_{i,j}-\hat{q}_{i,j})$ means the standard deviation of $q_{i,j}^*-\hat{q}_{i,j}$, defined as:
    \begin{equation}
        \sigma(q_{i,j}^*-\hat{q}_{i,j})=\sqrt{\frac{\sum_t(q_{t,i,j}^*-\hat{q}_{t,i,j})^2}{T}} \label{stander_d}
    \end{equation}
    \label{MSCE}
  \end{theorem}
  It should be noted that $\sigma(q_{i,j}^*-\hat{q}_{i,j})$ could be regarded as '$RMSE$' (root mean square error) of quantile regression. This theorem states that as the deployment time increases, the $MACE$ will converge to the goodness-of-fit of the quantile regression model. In another word, If the model is deployed for a long period and the quantile regression model performs well, it can be guaranteed that the average coverage error will be relatively small with suitable $\gamma$.
  
  Besides, in traditional conformal prediction method, the feature of data is not used and $\hat{q}_{i,j}$ is regarded as a constant, then the upper bound of $MACE$ could be related to $\sigma(q_{i,j}^*)$. Therefore, as long as the feature is related to $q_{i,j}^*$ and the quantile prediction model can capture this relationship in some extend, our algorithm could be regarded as an improvement.
  
  % Besides, if a constant is used for $\hat{q}_{i,j}$, the upper bound of $MACE$ could be related to $\sigma(q_{i,j})$. Therefore, if $\sigma(q_{i,j}-\hat{q}_{i,j})<\sigma(q_{i,j})$, our method will perform better than traditional ACI. As a result, as long as the feature is related to $q_{i,j}$ and the quantile predction model can capture this relationship in some extend, our algorithm could be regarded as an improvement of ACI.
  
  The proofs about these two theorems can be found in Appendix \ref{section:proof}.

\section{Experiments}

\subsection{Protocols}
\textbf{Datasets} We conducted experiments on 12 datasets, which have been widely used in previous time series forecasting researches \cite{wu2022timesnet,liuitransformer,han2024softs}. The schemes for splitting the datasets into training, validation, and test sets follow the approach described in the iTransformer paper \cite{liuitransformer}. 

\textbf{Base models} To validate the applicability of our algorithm across different time series forecasting models, we selected three recently proposed models: iTransformer \cite{liuitransformer}, Leddam \cite{yurevitalizing}, SOFTS \cite{han2024softs}, as the base point forecasting models. Following the setup of prior work on time series forecasting, the base models were trained to predict the next 96 time steps based on data from the past 96 time steps.

\textbf{Baselines} For baseline methods, we considered ECI \cite{wu2025errorquantified} (applied independently to each prediction step and dimension), along with two specialized methods for multivariate time series conformal prediction: TQA-E (Temporal Quantile Adjustment-Error) \cite{lin2022conformal} and LPCI ( Longitudinal Predictive Conformal Inference )\cite{batra2023conformal}. Additionally, CF-SST (Conformalized Time Series
 with Semantic Features ) \cite{chen2024conformalized}, a prior work performing feature-level conformal prediction for time series, was included as a baseline.

We also conducted experiments with other conformal prediction algorithms including PID \cite{angelopoulos2023conformal}, QCP (Quantile Conformal Prediction) \cite{romano2019conformalized}, CP (Classical Conformal Prediction), ACI (Adaptive Conformal Inference) \cite{gibbs2021adaptive}, FCP (Feature Conformal Prediction) \cite{teng2023predictive}, and SPCI (Sequential Predictive Conformal Inference) \cite{xu2023sequential}. Due to space constraints, these results are omitted from the main text but provided in the Appendix \ref{section:full result}. While these methods show varied performance, their exclusion from the main analysis does not affect our core conclusions.

Details and discussions about baseline methods can be found in Appendix \ref{section:baseline}. Besides, we also conduct experiments using quantile regression and MC dropout, the results are in Appendix \ref{section:MC}.

\textbf{Setup} For our own method, MLP was considered as quantile prediction model. The learning rate \(\gamma\) is a critical parameter. In the main experiments, we fixed \(\gamma\) at 0.002. We set two hidden layers in MLP model with hidden size 512 and 256. We split the validation set randomly by 8:2, and use 80\% of data to train the quantile prediction model for 100 epochs. If loss in the remaining 20\% does not decrease for 5 epochs, we will stop training. Adam was used as optimization with learning rate 0.001. We set $\alpha$ as 0.1, which means we want to obtain confidence interval with 90\% coverage. 

\textbf{Evaluation} We used four metrics to evaluate the performance of confidence intervals: coverage ($Cov$), average interval length ($l$), coverage on the worst-performing dimension ($min\_d$), and coverage on the worst-performing prediction step ($min\_t$). These four metrics are defined as:
$$
Cov=\frac{1}{T\times p\times d_1}\sum_{t=1}^T\sum_{i=1}^p\sum_{j=1}^{d_1}\mathbb{I}(y_{t,i,j}\in C_{t,i,j})
$$
$$
l=\frac{1}{T\times p\times d_1}\sum_{t=1}^T\sum_{i=1}^p\sum_{j=1}^{d_1}l(C_{t,i,j})
$$
$$
Min\_d=min_i\sum_{t=1}^T\sum_{j=1}^{d_1}\frac{1}{T\times d_1}\mathbb{I}(y_{t,i,j}\in C_{t,i,j})
$$
$$
Min\_t=min_j\sum_{t=1}^T\sum_{i=1}^p\frac{1}{T\times p}\mathbb{I}(y_{t,i,j}\in C_{t,i,j})
$$
\subsection{Results}
\renewcommand{\arraystretch}{1.5}
\setlength{\tabcolsep}{1pt}
\begin{table}[!h]
\caption{Forecast results}
\label{tab:result1}
\fontsize{4.8}{5}\selectfont 
\centering
\begin{tabular}{@{}c|cccc|cccc|cccc|cccc|cccc|c@{}}
\toprule
Method      & \multicolumn{4}{c|}{ECI}                                              & \multicolumn{4}{c|}{TQA-E}                                   & \multicolumn{4}{c|}{LPCI}                                                          & \multicolumn{4}{c|}{CF-SST}                                                        & \multicolumn{4}{c|}{DDFCI}                                            & \multicolumn{1}{c}{Improve} \\ \midrule
Dataset  & $Cov$    & $l$                                     &$ Min\_d $                      & $Min\_t$ & $Cov$    & $l$                                     &$ Min\_d $                      & $Min\_t$& $Cov$    & $l$                                     &$ Min\_d $                      & $Min\_t$& $Cov$    & $l$                                     &$ Min\_d $                      & $Min\_t$& $Cov$    & $l$                                     &$ Min\_d $                      & $Min\_t$ & $l$             \\\midrule
weather     & 89.8\% & {\color[HTML]{0070C0} 1.147}          & 87.3\%   & 89.4\%   & 85.8\% & 1.108                        & 82.3\%   & 82.8\%   & 84.7\% & 0.918                        & \cellcolor[HTML]{E2EFDA}60.8\% & 76.4\%   & 90.7\% & 1.233                        & 83.0\%                         & 86.3\%   & 89.3\% & {\color[HTML]{FF0000} \textbf{1.062}} & 87.6\%   & 88.9\%   & 7.4\%                \\
traffic     & 88.4\% & 1.094                                 & 77.7\%   & 87.8\%   & 88.5\% & 1.146                        & 77.7\%   & 87.9\%   & 88.4\% & {\color[HTML]{0070C0} 1.088} & \cellcolor[HTML]{E2EFDA}64.4\% & 87.5\%   & 86.7\% & 1.094                        & \cellcolor[HTML]{E2EFDA}68.9\% & 84.0\%   & 88.6\% & {\color[HTML]{FF0000} \textbf{1.047}} & 78.9\%   & 87.6\%   & 3.7\%                \\
electriticy & 88.4\% & 1.094                                 & 82.9\%   & 88.1\%   & 89.0\% & 1.170                        & 84.0\%   & 88.2\%   & 89.0\% & {\color[HTML]{0070C0} 1.061} & \cellcolor[HTML]{C6E0B4}59.1\% & 86.5\%   & 87.1\% & 1.061                        & 80.7\%                         & 82.5\%   & 89.3\% & {\color[HTML]{FF0000} \textbf{1.054}} & 85.2\%   & 88.8\%   & 0.6\%                \\
solar       & 88.4\% & 1.333                                 & 86.9\%   & 87.4\%   & 89.3\% & 1.684                        & 88.9\%   & 88.9\%   & 82.3\% & 1.106                        & 74.5\%                         & 77.9\%   & 88.1\% & {\color[HTML]{0070C0} 1.322} & 85.4\%                         & 82.4\%   & 89.8\% & {\color[HTML]{FF0000} \textbf{1.054}} & 89.4\%   & 89.6\%   & 20.3\%               \\
ETTh1       & 92.3\% & 2.125                                 & 86.7\%   & 90.3\%   & 89.4\% & {\color[HTML]{0070C0} 1.891} & 87.0\%   & 88.5\%   & 87.8\% & 1.684                        & 79.4\%                         & 85.8\%   & 92.4\% & 1.977                        & 86.2\%                         & 89.8\%   & 89.2\% & {\color[HTML]{FF0000} \textbf{1.683}} & 87.5\%   & 87.5\%   & 11.0\%               \\
ETTh2       & 89.6\% & {\color[HTML]{FF0000} \textbf{1.273}} & 85.2\%   & 88.6\%   & 87.4\% & 1.377                        & 85.1\%   & 85.6\%   & 85.1\% & 1.144                        & \cellcolor[HTML]{E2EFDA}68.0\% & 82.4\%   & 89.4\% & {\color[HTML]{0070C0} 1.456} & 82.4\%                         & 86.3\%   & 87.3\% & 1.194                                 & 83.8\%   & 86.0\%   & 6.2\%                    \\
ETTm1       & 90.6\% & {\color[HTML]{0070C0} 1.696}          & 89.3\%   & 90.3\%   & 88.2\% & 1.840                        & 85.7\%   & 87.1\%   & 86.0\% & 1.508                        & 76.0\%                         & 84.0\%   & 90.2\% & 1.757                        & 84.8\%                         & 87.5\%   & 89.4\% & {\color[HTML]{FF0000} \textbf{1.623}} & 88.2\%   & 89.1\%   & 4.3\%                \\
ETTm2       & 88.1\% & {\color[HTML]{0070C0} 1.192}          & 86.8\%   & 87.5\%   & 87.5\% & 1.161                        & 86.2\%   & 85.6\%   & 83.5\% & 0.972                        & \cellcolor[HTML]{E2EFDA}65.6\% & 80.7\%   & 87.4\% & 1.070                        & 73.2\%                         & 82.6\%   & 88.9\% & {\color[HTML]{FF0000} \textbf{1.114}} & 87.8\%   & 88.3\%   & 7.0\%                \\
PEMS03      & 87.7\% & {\color[HTML]{0070C0} 1.221}          & 80.9\%   & 87.2\%   & 88.8\% & 1.272                        & 81.7\%   & 88.3\%   & 86.9\% & 1.255                        & \cellcolor[HTML]{C6E0B4}59.0\% & 81.2\%   & 84.8\% & 1.199                        & \cellcolor[HTML]{E2EFDA}69.0\% & 74.5\%   & 88.7\% & {\color[HTML]{FF0000} \textbf{1.179}} & 82.0\%   & 88.0\%   & 3.5\%                \\
PEMS04      & 88.7\% & {\color[HTML]{0070C0} 1.299}          & 81.8\%   & 88.3\%   & 89.7\% & 1.338                        & 82.7\%   & 89.4\%   & 88.0\% & 1.329                        & \cellcolor[HTML]{C6E0B4}56.3\% & 82.9\%   & 84.8\% & 1.466                        & 73.1\%                         & 79.8\%   & 88.9\% & {\color[HTML]{FF0000} \textbf{1.196}} & 82.0\%   & 88.3\%   & 8.6\%                \\
PEMS07      & 89.1\% & {\color[HTML]{0070C0} 1.136}          & 77.7\%   & 88.9\%   & 89.9\% & 1.169                        & 78.0\%   & 89.5\%   & 89.1\% & 1.201                        & \cellcolor[HTML]{A9D08E}43.8\% & 85.5\%   & 90.0\% & 1.228                        & 75.3\%                         & 85.9\%   & 89.6\% & {\color[HTML]{FF0000} \textbf{1.075}} & 80.5\%   & 89.4\%   & 5.7\%                \\
PEMS08      & 88.7\% & {\color[HTML]{0070C0} 1.236}          & 76.8\%   & 88.3\%   & 89.7\% & 1.288                        & 77.5\%   & 89.4\%   & 88.6\% & 1.341                        & \cellcolor[HTML]{C6E0B4}57.3\% & 82.9\%   & 87.5\% & 1.290                        & 73.7\%                         & 83.9\%   & 88.7\% & {\color[HTML]{FF0000} \textbf{1.143}} & 77.6\%   & 88.0\%   & 8.2\%                \\ \bottomrule
\end{tabular}
\end{table}
In Table \ref{tab:result1}, we report the main experimental results. The values of the evaluation metric represent the average performance across the three base point prediction models. The full forecasting results are listed in Appendix \ref{section:full result}. Methods with an overall coverage rate exceeding 88\% are considered valid. Among the valid methods, the shortest length is highlighted in \textcolor{red}{red}, while the second shortest is highlighted in \textcolor{customblue}{blue}. 
In addition, we use green background colors to indicate coverage rates below 70\%, where darker shades represent lower coverage levels: light green for rates below 70\%, medium green for below 60\%, and dark green for below 50\%. The last column of the table shows the percentage reduction in confidence interval length achieved by our method compared to the best-performing baseline method.

It can be concluded that our method achieves valid coverage in most cases while maintaining relatively shorter interval lengths. For example, on the weather dataset, our method reduces confidence interval length by 7.4\%. For the solar dataset, the reduction exceeds 20\%. As for electricity dataset, where the reduction is only 0.6\%, our method significantly outperforms the best baseline (LPCI) in terms of worst-dimension coverage($Min\_d$). Specifically, LPCI achieves only 59\% coverage in the worst dimension, while our method reaches 85\%.

The only exception occurs on the ETTh2 dataset, where our method slightly misses the target coverage (87.3\% vs. 90\%). However, as shown in the sensitivity analysis of learning rates in Appendix \ref{section:sensitive of gamma}, adjusting the learning rate allows our method to achieve the target coverage on ETTh2 without significantly increasing the interval length.

Regarding worst-dimension coverage, our method consistently delivers high results: most datasets exceed 80\%, with only two exceptions at 78\% and 77\%. These results are superior to other methods.

Additionally, we would like to present experimental results related to our proposed $MACE$ metric. However, since $MACE$ is defined as the difference between the actual coverage rate of our confidence intervals and the target 90\% coverage rate at each timestep — while the true coverage rate of the intervals we produce at each timestep remains inherently unknown — this metric cannot be directly measured. However, we can approximate $MACE$ using some local coverage metrics. Detailed experimental results supporting this analysis are provided in the Appendix \ref{section:MACE}.
\subsection{Ablation Experiments}
\renewcommand{\arraystretch}{2}
\setlength{\tabcolsep}{1.5pt}
\begin{table}[!h]
\fontsize{6}{5}\selectfont 
\centering
\caption{Results of ablation experiments}
\label{tab:a_result1}
\begin{tabular}{@{}c|cccc|ccccc|ccccc@{}}
\toprule
Method      & \multicolumn{4}{c|}{Ori}              & \multicolumn{5}{c|}{w/o update}                                         & \multicolumn{5}{c}{w/o feature}                \\ \midrule
dataset     & $Cov$    & $l$                                     &$ Min\_d $                      & $Min\_t$& $Cov$    & $l$                                     &$ Min\_d $                      & $Min\_t$& $C\_loss$ &  $Cov$    & $l$                                     &$ Min\_d $                      & $Min\_t$& $l\_loss$ \\ \midrule
weather     & 89.3\% & 1.062 & 87.6\%   & 88.9\%   & 79.5\% & 0.804 & \cellcolor[HTML]{E2EFDA}65.5\% & 75.9\%   & 9.8\%     & 90.5\% & 1.137 & 89.1\%   & 90.2\%   & 7.0\%   \\
traffic     & 88.6\% & 1.047 & 78.9\%   & 87.6\%   & 84.6\% & 0.959 & \cellcolor[HTML]{C6E0B4}56.1\% & 79.9\%   & 4.0\%     & 89.3\% & 1.150 & 82.6\%   & 88.9\%   & 9.8\%   \\
electricity & 89.3\% & 1.054 & 85.2\%   & 88.8\%   & 86.9\% & 0.988 & \cellcolor[HTML]{A9D08E}39.8\% & 84.7\%   & 2.4\%     & 89.6\% & 1.102 & 86.7\%   & 89.4\%   & 4.5\%   \\
solar       & 89.8\% & 1.054 & 89.4\%   & 89.6\%   & 88.7\% & 1.032 & 84.9\%                         & 81.7\%   & 1.1\%     & 89.7\% & 1.495 & 89.2\%   & 89.5\%   & 41.8\%  \\
ETTh1       & 89.2\% & 1.683 & 87.5\%   & 87.5\%   & 87.7\% & 1.630 & 85.0\%                         & 84.2\%   & 1.4\%     & 93.1\% & 2.234 & 88.4\%   & 90.3\%   & 32.7\%  \\
ETTh2       & 87.3\% & 1.194 & 83.8\%   & 86.0\%   & 83.2\% & 1.040 & \cellcolor[HTML]{E2EFDA}69.5\% & 77.8\%   & 4.0\%     & 89.6\% & 1.290 & 86.9\%   & 89.0\%   & 8.1\%   \\
ETTm1       & 89.4\% & 1.623 & 88.2\%   & 89.1\%   & 81.8\% & 1.310 & 75.0\%                         & 79.2\%   & 7.7\%     & 90.3\% & 1.691 & 89.4\%   & 90.1\%   & 4.2\%   \\
ETTm2       & 88.9\% & 1.114 & 87.8\%   & 88.3\%   & 81.5\% & 0.808 & 76.0\%                         & 77.4\%   & 7.4\%     & 89.7\% & 1.160 & 89.2\%   & 89.4\%   & 4.1\%   \\
PEMS03      & 88.7\% & 1.179 & 82.0\%   & 88.0\%   & 81.6\% & 0.989 & \cellcolor[HTML]{C6E0B4}51.3\% & 76.6\%   & 7.1\%     & 89.4\% & 1.323 & 85.8\%   & 89.2\%   & 12.2\%  \\
PEMS04      & 88.9\% & 1.196 & 82.0\%   & 88.3\%   & 84.9\% & 1.085 & \cellcolor[HTML]{E2EFDA}62.5\% & 81.7\%   & 4.0\%     & 89.7\% & 1.356 & 85.3\%   & 89.6\%   & 13.4\%  \\
PEMS07      & 89.6\% & 1.075 & 80.5\%   & 89.4\%   & 87.8\% & 1.014 & \cellcolor[HTML]{A9D08E}37.5\% & 86.0\%   & 1.9\%     & 90.1\% & 1.178 & 84.2\%   & 90.0\%   & 9.6\%   \\
PEMS08      & 88.7\% & 1.143 & 77.6\%   & 88.0\%   & 83.9\% & 1.025 & \cellcolor[HTML]{C6E0B4}59.9\% & 80.5\%   & 4.7\%     & 89.8\% & 1.303 & 82.0\%   & 89.7\%   & 14.0\%  \\ \bottomrule
\end{tabular}
\end{table}

\definecolor{custompurple}{HTML}{7030A0}
Our method consists of two main components: adaptive updates of confidence interval lengths and the use of features to predict error quantiles. We conducted ablation experiments to verify the effectiveness of these two components. Specifically, in the first experiment, we did not update the confidence interval lengths but directly used predicted quantiles as error quantiles to construct confidence intervals. In the second experiment, instead of using features to predict error quantiles, we fixed the output of the quantile prediction model to the 90\% quantile of validation set errors.

The results are shown in Table \ref{tab:a_result1}, where green shading indicates coverage rates below 70\% (The results in this table represent the average across three point prediction models, with complete results available in the Appendix \ref{section:ae}.). For the experiment without interval length updates, the "$C\_Loss$" column shows the coverage degradation caused by disabling updates. For the experiment without feature-based quantile fitting, the "$l\_loss$" column quantifies the percentage increase in confidence interval length due to the lack of quantile prediction.

These ablation studies demonstrate that the two components serve distinct purposes: updating interval lengths ensures valid coverage, while feature-based quantile fitting reduces interval lengths. Without interval length updates, coverage degrades—for example, by 9.8\% on the weather dataset, with other datasets showing declines ranging from 1\% to 7\%. Without feature-based fitting, coverage remains valid, but interval lengths increase substantially, e.g., 7\% on the weather dataset and 41.8\% on the solar dataset.

% \subsection{Sensitive analysis of $\gamma$}

% The update rate of the interval length, \(\gamma\), is a very important parameter. Therefore, we aim to demonstrate that our method performs reasonably well for different values of \(\gamma\) through additional experiments. To achieve this, we conducted experiments where \(\gamma\) equals to 0.001, 0.005, and 0.01, and obtained the results in Figures \ref{fig:lr-MLP}. Figure \ref{fig:lr-MLP} shows the algorithm's performance for different values of \(\gamma\) when using MLP. The left part of Figure \ref{fig:lr-MLP} illustrates the relationship between different $\gamma$ and average coverage, while the right part of it shows the relationship between $\gamma$ and interval length.

% It can be observed from Figures \ref{fig:lr-MLP} that, when $\gamma$ increases, the coverage and interval length increase. But the increase for coverage become inconspicuous when $\gamma$ is greater than 0.002. As for interval lengths, even though they increase with $\gamma$, these lengths are generally shorter than baseline method. For example, the largest length for ETTm1 dataset is around 1.7 in our method, but the length of baseline method is above 1.8. And for weather dataset, the largest length is 1.15 for our method, but in baseline method, the length is 1.47. Finally, the confidence intervals provided by our method become valid in ETTh2 dataset when $\gamma$ equals to 0.01. 

% Overall, for different values of \(\gamma\), our method performs relatively stable and provides valid confidence interval with shorter length in most cases. 

\section{Conclusion}

In this work, we propose a novel algorithm to provide confidence intervals for deep learning-based time series forecasting models. Our approach leverages the powerful feature extraction capabilities of deep learning models and does not require any modifications to the original forecasting model. Experimental results demonstrate that our approach not only produces valid confidence intervals but also reduces their length compared to existing methods. We also show that the coverage rate of the confidence intervals provided by our method converges to the target coverage rate. Besides, we introduced Mean Absolute Coverage Error (MACE) and proves that the upper bound of it is directly related to the fitting performance of the quantile prediction model. In the future, we aim to improve the chosen of $\gamma$ and make the features more representative.

Overall, this work highlights the effectiveness of considering features when modeling confidence interval for time series forecasting, and we holp it could help provide reliable time series predictions for real-world application.
% \section*{Impact Statements}
% This paper presents work whose goal is to advance the field of Machine Learning. There are many potential societal consequences of our work, none of which we feel must be specifically highlighted here.
\bibliography{example_paper}
\bibliographystyle{icml2025}

%%%%%%%%%%%%%%%%%%%%%%%%%%%%%%%%%%%%%%%%%%%%%%%%%%%%%%%%%%%%%%%%%%%%%%%%%%%%%%%
%%%%%%%%%%%%%%%%%%%%%%%%%%%%%%%%%%%%%%%%%%%%%%%%%%%%%%%%%%%%%%%%%%%%%%%%%%%%%%%
% APPENDIX
%%%%%%%%%%%%%%%%%%%%%%%%%%%%%%%%%%%%%%%%%%%%%%%%%%%%%%%%%%%%%%%%%%%%%%%%%%%%%%%
%%%%%%%%%%%%%%%%%%%%%%%%%%%%%%%%%%%%%%%%%%%%%%%%%%%%%%%%%%%%%%%%%%%%%%%%%%%%%%%
\newpage
\appendix
\section{Related Work}
\subsection{Conformal prediction and its single variable time series version}
The simplest version of conformal prediction is just regarding errors in validation set as the errors in test set \cite{shafer2008tutorial}. Therefore, if the 90 percentage quintile of errors in the validation set is $q$, and $\hat{y}$ is a point prediction in the test set, then the 90 percentage confidence interval of this prediction is $[\hat{y}-q,\hat{y}+q]$. Moreover, there is a theorem \cite{lei2018distribution} that shows if the data is exchangeable, the confidence intervals derived in this way can guarantee coverage at the target level. 

However, for time series data, the values at adjacent time points are usually dependent, and the distribution of data may also change over time. Therefore, considerable research has been done to improve the basic conformal prediction method to develop algorithms for these non-exchangeable time series data. One foundational work in this area is ACI \cite{gibbs2021adaptive}. Its core idea is to adjust the confidence interval during deployment. If the coverage rate of past step is too large, the confidence interval will be shortened, and vice versa. Later, some works committed to improve ACI, specifically focusing on how to determine the adjusting speed of interval \cite{xu2024conformal}. They treated the problem as an online convex optimization problem, and provided algorithms form this framework \cite{bhatnagar2023improved,zhang2024benefit,zhangdiscounted,Angelopoulos2024OnlineCP}. 
Additionally, some studies proposed methods to construct confidence intervals by continuously updating validation set \cite{xu2021conformal} (EnbPI) or giving distinctive weight to each data point in validation set \cite{barber2023conformal} or by predicting future errors using the errors from some past time steps (SPCI) \cite{xu2023sequential}.
\subsection{Conformal prediction for multi-variable or multi-step time series}
In recent work about multi-step time series prediction \cite{sun2022copula}, different steps were just regarded as different variables, so we list works about multi-variable and multi-step as a single class and use variable to represent both step and variable in the following part of this section. In terms of constructing confidence intervals for multivariate time series, there are two types of coverage requirements. 

Suppose the prediction is of dimension $p$. The first type treats the 
$p$-dimensional data as a single data point and then constructs a region in the $p$-dimensional space such that the true value falls within this region with a given probability \cite{xu2024conformal,stankeviciute2021conformal}. The core issue here is how to consider the relationships between different dimensions. Some methods enlarged the given confidence level $1-\alpha$ \cite{stankeviciute2021conformal,cleaveland2024conformal,lopes2024conforme}, while others used copula methods \cite{sun2022copula}. Besides, others considered the covariance of errors in different dimensions and then constructed ellipsoidal prediction regions \cite{xu2024conformal}. The second type treats the data as $p$ data points, and provides a confidence interval for each dimension individually, requiring each dimension falling within its interval with a given probability. Some methods extended the ACI method to high dimensional data \cite{lin2022conformal,yang2024bellman}, for example, executed ACI in each dimension. And some works also considered using past errors to predict future errors \cite{batra2023conformal}, similar to SPCI. Besides, EnbPI \cite{xu2021conformal} and ACI \cite{gibbs2021adaptive} have been combined in \cite{sousa2024general}.

Both definitions of coverage are valuable, but in this paper, we focus on the second one. This is because in real life, it may be more intuitive to provide a confidence interval for each dimension separately. For example, when predicting the temperature in various regions, residents in a specific region might be more concerned about the upper and lower bounds of future temperature in their own region. But if we follow the first method, the confidence interval would be an ellipsoidal region about temperature vectors in all regions, which may not be very meaningful to the residents in a specific region.

\subsection{Better intervals for conformal prediction}
Another category of methods relevant to our research is how to construct more efficient conformal prediction intervals. The main issue with traditional conformal prediction is that it gives confidence interval with the same length for every point, which is not suitable for many problems. Therefore, some studies proposed to construct confidence intervals based on local sample points \cite{lee2024kernel,lei2014distribution} or use the predicted variance to construct the confidence intervals \cite{lei2018distribution}. Other works transformed traditional predictive models into quantile regression models and adjusting the predicted quantiles using validation set \cite{romano2019conformalized}. Additionally, some researchers considered modifying point prediction models to  distribution prediction models and then construct confidence intervals \cite{chernozhukov2021distributional}. In recent studies, the features of data have been considered, for example, by dividing the feature space into different regions and providing a distinctive confidence interval length for each region \cite{kiyani2024conformal,kiyani2024length}. Some papers also define errors at the feature level and then use feature-level errors to derive errors at the prediction level \cite{teng2023predictive,chen2024conformalized}.
\section{Implementation details}
\subsection{Datasets}
We conduct experiments on the following real-world datasets to evaluate the performance of proposed method. The datasets include: 
\begin{itemize}
    \item \textbf{Weather}: 21 meteorological factors were collected every 10 minutes from the Weather Station of the Max Planck Biogeochemistry Institute in 2020.
    \item \textbf{Traffic}: Hourly road occupancy rates were recorded from 862 sensors installed on freeways in the San Francisco Bay Area, covering the period from January 2015 to December 2016.
    \item \textbf{Electricity}: Hourly electricity consumption data from 321 clients was recorded.
    \item \textbf{Solar}: Solar power production data from 137 photovoltaic plants in 2006 was collected every 10 minutes.
    \item \textbf{ETT}: Data on 7 factors of electricity transformers was collected from July 2016 to July 2018. Four subsets are available: ETTh1 and ETTh2, recorded every hour, and ETTm1 and ETTm2, recorded every 15 minutes.
    \item \textbf{PEMS}: Public traffic network data from California was sampled in 5-minute windows.
\end{itemize}
For more information about datasets, please refer to \cite{liuitransformer}.

\subsection{Base point prediction models}
We did not conduct hyperparameter search suggested by the original papers \cite{liuitransformer,han2024softs,yurevitalizing} because finding the best point prediction model is not our focus. As a result, we fix the dimension of feature as 256 and the number of layer as 3. Besides, we used Adam to train these models for 50 epochs with early stop if loss in validation set did not decrease for 5 epochs. The learning rate was 0.001. Other hyperparameters were set as default values according to their github repositions.

% All models were trained using an NVIDIA RTX 4090D GPU.
\section{Baseline methods}
In this section, we first introduce each baseline method and finally provide some discussion on several baselines most relevant to our approach.
\label{section:baseline}
\subsection{ECI \cite{wu2025errorquantified} }
This method employs a smoother loss function as an alternative to the traditional quantile loss used in online conformal prediction. Specifically, the original quantile loss function (Equation \ref{eq:onlineloss}) contains an indicator function that determines whether $s_t$
exceeds $q_t$. The authors argue that this indicator function is excessively non-smooth and propose replacing it with a sigmoid function. Consequently, they derive the following updated formulation for $q_t$:
\begin{equation}
    q_{t+1}=q_t+\gamma(err_t-\alpha-(s_t-q_t)\nabla f(s_t-q_t))
    \label{eq:eci}
\end{equation}
where $err_t=\mathbb{I}(s_t>q_t)$, $f$ means sigmoid function,i.e., $f(x)=\frac{1}{1+cexp(-x)}$. Although the original paper used $c=1$ in their experiments, we found this setting failed to achieve the target coverage rate in our experiments. Through empirical testing, we determined that $c=0.2$ yields better results. Consequently, all reported results in our study correspond to this configuration.
\subsection{TQA-E \cite{lin2022conformal} }
This method is designed to tackle multi-variable problem and we modified it in our experiments for multi-step problem. We use the $1-a_{t,i,j}$ quantile of errors in the past to obtain confidence intervals. That is
\begin{equation}
    C_{t,i,j}=[\hat{y}_{t,i,j}-Q_{1-a_{t,i,j}}(E_{t-j,i,j}),\hat{y}_{t,i,j}+Q_{1-a_{t,i,j}}(E_{t-j,i,j})]
\end{equation}
Where $E_{t,i,j}=\{e_{k,i,j}\ | \ k\in[0,1,...,t]\}$ is a bag \cite{shafer2008tutorial} of errors in the past times. And $a_{t,i,j}=\alpha-\delta_{t,i,j}$, $\delta_{t,i,j}$ is obtained in an iteration way:
\begin{equation}
     \delta_{t+1,i,j}=
    \begin{cases}
        \delta_{t,i,j}+\gamma(I_{y_{t-j,i,j}\notin C_{t-j,i,j}}-\alpha) &\delta_{t,i,j} \geq1-\alpha\\
        (1-\gamma)\delta_{t,i,j} &otherwise
    \end{cases}
\end{equation}
$\gamma$ was set as 0.005 according to the original paper.

\subsection{LPCI \cite{batra2023conformal}} 
The main idea of this method is using $[e_{t-k},...,e_{t-1}]$ to predict the quantile of $e_{t}=y_t-\hat{y}$ and the confidence interval is $$[\hat{y}_t+q_{\beta}(e_t),\hat{y}_t+q_{1-\alpha+\beta}(e_t)]$$ where $\beta$ is choosed to minimize to interval length. To handle multi-variable situation,  \cite{batra2023conformal} propose to use data in all dimensions to fit one model but using additional variable identifies $ide_i$ to signify different dimensions. As a result, the quantile prediction model regards $[e_{t-k},...,e_{t-1}, ide_i]$ as input and predict the quantile of future errors. 

The original algorithm requires to retrain quantile prediction in every step using the most recent data. However, it is vary time consuming to do that and we choose to retrain this model every 200 steps. And $[e_{t-k-j,i,j},...,e_{t-1-j,i,j}, ide_i]$ is regarded as input because  $e_{t-1,i,j}$ can not be observed until time step $t+j$. The lag $k$ was 20, the same as the original paper.
\subsection{ACI \cite{gibbs2021adaptive}}
This method  represents the most classical online conformal prediction approach. Specifically, it adjusts the confidence interval based on whether the current interval covers the true value. Rather than directly modifying the interval length, it adaptively selects which quantile of the test set errors to use for constructing the next interval. This process can be equivalently viewed as performing online gradient descent on the quantile parameter \cite{wu2025errorquantified}, where the loss function is the standard quantile loss. 
\subsection{SPCI \cite{xu2023sequential}}
The core idea of this method involves using historical residuals to predict future residual quantiles, which closely resembles the approach of LPCI. However, the key distinction lies in the fact that SPCI was originally designed for univariate time series forecasting problems. In our experiments, we adapted this approach by calibrating separate future error quantile prediction models for each variable and each prediction horizon. All other experimental settings remained identical to those used in LPCI.
\subsection{PID \cite{angelopoulos2023conformal}}
This method applies control theory to the construction of confidence intervals for time series, consisting of three main components: The first part uses an algorithm similar to ACI to adaptively update the length of confidence intervals. Unlike traditional ACI which updates the interval length for the next time step based solely on whether the current time step is covered, this method employs PID control to incorporate coverage statistics from past periods when updating the interval length. And this procedure is the second part. The third part utilizes a model that predicts future errors based on past errors.

\subsection{FCP \cite{teng2023predictive}}

The primary objective of this method is to shift traditional conformal prediction from computing residuals at the prediction level to the feature level. Specifically, for a deep learning model $ f $, it can be decomposed into a feature extractor $ g $ and a prediction head $ h $. Given an input-output pair $(X,Y)$, its latent feature representation is obtained through the feature extractor:
\begin{equation}
    z = g(X)
\end{equation}

Then, the optimal perturbed feature $ z' $ that minimizes the distance from the original feature while satisfying the prediction constraint can be calculated as follows:
\begin{equation}
    \min_{z'} \|z' - z\| \quad \text{subject to} \quad h(z') = Y
\end{equation}
This optimization is achieved through gradient descent in the feature space.

During calibration, the feature-space distances $\|z' - z\|$ for all validation samples are computed. Let $\Delta$ denote the $(1-\alpha)$-quantile of these distances across the validation set:
\begin{equation}
    \Delta = \text{Quantile}_{1-\alpha}\left(\{\|z'_i - z_i\|\}_{i=1}^N\right)
\end{equation}

For a test sample $ X_{N+1} $ with extracted feature $ z_{N+1} = g(X_{N+1}) $, we construct the prediction interval as:
\begin{equation}
    \mathcal{S} = \left\{ y \in \mathcal{Y} \, \big| \, \exists z' \text{ s.t. } y = h(z'), \, \|z' - z_{N+1}\| < \Delta \right\}
\end{equation}
\subsection{CF-SST \cite{chen2024conformalized}}
This method extends FCP. Since FCP was not designed for time series data, it only guarantees confidence interval coverage under exchangeable data assumptions. The CF-SST method extends the FCP to ensure coverage for time series data. Specifically, for the $(N+1)$-th data point: Traditional FCP uses the $1-\alpha$ quantile of feature-space errors from the validation set, but CFSST employs a weighted quantile approach that first calculates the similarity between the $(N+1)$-th sample's errors and those of each validation sample, then weights the validation set feature-space errors based on these similarity scores to compute a weighted quantile. 

\subsection{QCP \cite{romano2019conformalized}}
This method first converts the point prediction model into a quantile prediction model. It then evaluates the accuracy of this quantile prediction model by computing the errors between the predicted quantiles and the actual dependent variables on the validation set.

For test samples, the method constructs confidence intervals by combining two the quantile predictions from the quantile prediction model and the model's observed errors on the validation set
\subsection{Discussion on baseline methods}
All the baseline methods mentioned above are highly insightful and have made significant contributions to this field. Our method builds upon these existing approaches, and in the following discussion, we will focus on the methods most relevant to our work while highlighting the key differences between them and our proposed approach.

\textbf{PID} This method is insightful and kind of similar to our approach. The paper about this method also mentions that using certain models to predict future error quantiles can reduce the variability of these quantiles, thereby improving prediction accuracy and efficiency. Our work advances beyond this in two key aspects. First, we demonstrate that features extracted by deep learning models can predict future error quantiles more effectively than relying solely on past error data. Second, through $MACE$-related theorems, we theoretically establish why using a model to predict future error quantiles yields superior confidence intervals. We further prove that better fitting of future error quantiles directly leads to improved confidence interval performance.

\textbf{SPCI and LPCI} These two methods also attempt to fit a predictive model for future error quantiles. And they rely on combining past errors to predict future error quantiles, which is similar to the PID approach. However, their coverage guarantees require this quantile prediction model to be accurate, thus imposing certain assumptions about the data distribution. Our experiments also show that these two methods often fail to achieve the target coverage rate. In the future, research on deep learning-based confidence intervals could consider both past errors and the features of the current sample to construct an error quantile prediction model. This would effectively combine the strengths of these existing methods and our proposed approach.

\textbf{FCP} This method also utilizes deep learning features to construct confidence intervals, but it was not specifically designed for time series data. Consequently, it cannot provide coverage guarantees for time series applications, a limitation that has been empirically confirmed in our experiments.

\textbf{CF-SST} This method modifies FCP to adapt it for time series data. However, our experiments reveal two key limitations: first, it still struggles to guarantee proper coverage rates, and second, the resulting confidence intervals tend to be unnecessarily wide. To theoretically explain these observations, let us first present the coverage guarantee theorem from the original paper. Suppose $\Delta Cov$ represents the discrepancy between the achieved coverage rate and the target coverage level:
\begin{equation}
    \Delta Cov\geq-\sum_{i=1}^n{w^i}d_{TV}(R(Z),R(Z)^i)
\end{equation}
Where $R(Z)$ denotes the residual sequence from the first to the $N$-th data point
$R(Z)^i$ represents the modified sequence obtained by swapping the residuals of the $N$-th and $i$-th data points in $R(Z)$,$d_{TV}$ is the total variation norm between distributions. However, our task involves multi-step forecasting, where the output is a high-dimensional vector and consequently each residual is also high-dimensional. For high-dimensional random vectors, the total variation distance between them will be larger than the TV distance between their individual dimensions. This makes it difficult to control the TV distance between $R(z)$ and $R(z)^i$, ultimately leading to challenges in achieving the desired coverage rate.

Another challenge lies in CF-SST's objective to find the optimal feature $z'$ in the latent space such that when passed through the prediction head, it closely approximates the true output $y$. For one-dimensional prediction problems, it may be feasible to find such $z'$ within a small neighborhood around the original feature $z$. However, for high-dimensional prediction tasks, finding $z'$ becomes significantly more difficult – it is likely that $z'$ may end up being very far from $z$, which would result in wide confidence intervals.

Certainly, there exist numerous other conformal prediction algorithms for time series, such as EnbPI \cite{xu2021conformal}, NexCP \cite{barber2023conformal}, HopCPT \cite{auer2023conformal}, DtACI \cite{gibbs2024conformal}, AgACI \cite{zaffran2022adaptive}, MVP \cite{bastani2022pra}, decay-OGD \cite{Angelopoulos2024OnlineCP}, SAOCP \cite{bhatnagar2023improved} , etc. While our experiments do not cover all possible time series conformal prediction methods, we have included the most relevant baselines given our approach's key components: Feature-based confidence interval prediction (represented by FCP and CF-SST), online optimization for interval width adjustment (represented by PID, ACI, and ECI), we believe our experimental comparisons are sufficiently comprehensive for evaluating our method's effectiveness.

While all the aforementioned baseline methods represent groundbreaking contributions to the field of time series conformal prediction, our approach demonstrates particular suitability for the specific challenges posed by deep learning-based multivariate multi-step forecasting scenarios. The comparative advantages observed in our experiments should be interpreted as context-dependent refinements rather than universal superiority, as different methodological designs naturally excel in different application contexts. We further acknowledge the complementary value of advanced online optimization strategies from recent work, such as the expert aggregation approaches exemplified by DtACI \cite{gibbs2024conformal} and AgACI \cite{zaffran2022adaptive}. These innovations offer promising directions for enhancing our framework.
\newpage
\section{Proofs}
\label{section:proof}
\begin{lemma}
    Under assumption \ref{bounded assumption}, for any $t,i,j$, we have:
    \begin{equation}
        -M-\gamma\leq a_{t,i,j}\leq M+\gamma
        \label{bound a}
    \end{equation}
\end{lemma}
\begin{proof}
    We use inductive method to proof that. First, as we defined in method, $a_{t,j}=0$ when $t<j+1$, so they conform \ref{bound a}. Then we proof, if $-M-\gamma\leq a_{k,i,j}\leq M+\gamma$ holds for all $k\leq t$, then $-M-\gamma\leq a_{t+1,i,j}\leq M+\gamma$.
    
    \textbf{Case 1:} If $-M-\gamma \leq a_{t-j,i,j}\leq-M$, then:
    $\hat{q}_{t-j,i,j}+a_{t-j,j}\leq0$, and the prediction set is empty set. As a result:
    $$
     a_{t+1,i,j}=a_{t,i,j}+\gamma (1-\mathbb{I}_{y_{t-j,i,j}\in C_{t-j,i,j}}-\alpha)=a_{t,i,j}+\gamma(1-\alpha)> a_{t-j,j}\geq -M-\gamma
    $$
    And $a_{t+1,i,j}$ also less then $M+\gamma$.
    
    \textbf{Case 2:} If $-M< a_{t-j,i,j}<M$, then:
    $$
    a_{t+1,i,j}=a_{t,i,j}+\gamma (1-\mathbb{I}_{y_{t-j,i,j}\in C_{t-j,i,j}}-\alpha)\leq a_{t,i,j}+\gamma(1-\alpha)\leq M+\gamma(1-\alpha)
    $$
    $$
     a_{t+1,i,j}=a_{t,i,j}+\gamma (1-\mathbb{I}_{y_{t-j,i,j}\in C_{t-j,i,j}}-\alpha)\geq a_{t,i,j}-\gamma \alpha \geq -M-\gamma
    $$
    
    \textbf{Case 3:} If $M< a_{t-j,i,j}<M+\gamma$, then, because  assumption \ref{bounded assumption}, $\mathbb{I}_{y_{t-j,i,j}\in C_{t-j,i,j}}$ must be 1. As a result:
    $$
    a_{t+1,i,j}=a_{t,i,j}-\gamma \alpha\leq M+\gamma
    $$
    And $a_{t+1,i,j}$ also greater then $-\gamma$.

    Combine the three cases, we can proof $a_{t+1,i,j}$ obeys (\ref{bound a}), and complete the proof.
\end{proof}
% \subsection{Proof of Theorem \ref{chian_rule}}
% \begin{proof}
% \begin{align}
%     H(y_k|z) &=-\int p(y_k|z)log(p(y_k|z))dy_k\\
%     &=-\int p(y_k|z)log(\frac{p(y_k,y_{k-1}|z)}{p(y_{k-1}|y_{k},z)})dy_k
% \end{align}

% \begin{align}
%     -H(y_k,y_{k-1}|z) &= \int{p(y_k,y_{k-1}|z)log(p(y_k,y_{k-1}|z))}dy_{k-1}y_{k} \\
%      &= \int{p(y_k|y_{k-1},z)p(y_{k-1}|z)(log(p(y_k|y_{k-1},z)+log(y_{k-1}|z))}dy_{k-1}y_{k} \\
%      &= \int{p(y_k|y_{k-1},z)p(y_{k-1}|z)(log(p(y_k|y_{k-1},z)}dy_{k-1}y_{k}+\int{p(y_k|y_{k-1},z)p(y_{k-1}|z)log(y_{k-1}|z)}dy_{k-1}y_{k}
% \end{align}
% The first equation is the definition of entropy and the second equation is because of the property of log.
% The first term equals to:
% \begin{align}
%     \int{p(y_k|y_{k-1},z)p(y_{k-1}|z)(log(p(y_k|y_{k-1},z)}dy_{k-1}y_{k}&=  \int{p(y_{k-1}|z)[\int{p(y_k|y_{k-1},z)log(p(y_k|y_{k-1},z))}dy_k}]dy_{k-1}\\
%     &= -E_{y_{k-1}\sim p(y_{k-1}|z)}H(y_k|y_{k-1},z)
% \end{align}
% The second term equals to:
% \begin{align}
%     \int{p(y_k|y_{k-1},z)p(y_{k-1}|z)log(y_{k-1}|z)}dy_{k-1}dy_{k}&=\int{p(y_k|y_{k-1},z)[\int{p(y_{k-1}|z)log(y_{k-1}|z)}dy_{k-1}]}dy_k\\
%     &=-\int{p(y_k|y_{k-1},z)H(y_{k-1}|z)}dy_k\\
%     &=-H(y_{k-1}|z)
% \end{align}
% Therefore, combing then above two terms, we obtain the final conclusion:
% $$
% H(y_k,y_{k-1}|z)=\int{p(y_k,y_{k-1}|z)log(p(y_k.y_{k-1}|x))}
% $$
% \end{proof}
\subsection{Proof of Theorem \ref{all cover}}
\begin{proof}
    \begin{align}
        |\frac{\sum_{t=1}^T{\mathbb{I}_{y_{t,i,j}\in C_{t,i,j}}}}{T}-(1-\alpha)|&=|\frac{\sum_{t=T-j}^T{\mathbb{I}_{y_{t,i,j}\in C_{t,i,j}}}}{T}+\frac{\sum_{t=1}^{T-j-1}{\mathbb{I}_{y_{t,i,j}\in C_{t,i,j}}}}{T}-(1-\alpha)|\\
        &\leq |\frac{\sum_{t=1}^{T-j-1}{\mathbb{I}_{y_{t,i,j}\in C_{t,i,j}}}}{T}-(1-\alpha)|+\frac{j+1}{T}
        \label{coverall1}
    \end{align}
    According to (\ref{update a}), $$\mathbb{I}_{y_{t,i,j}\in C_{t,i,j}}=\frac{a_{t+j+1,i,j}-a_{t+j,i,j}}{\gamma}+(1-\alpha)$$
    Therefore,
    \begin{equation}
        \frac{\sum_{t=1}^{T-j-1}{\mathbb{I}_{y_{t,i,j}\in C_{t,i,j}}}}{T}=\frac{a_{T,i,j}-a_{j+1,i,j}}{T\gamma}+\frac{(1-\alpha)(T-j-1)}{T}
    \end{equation}
    substitute into (\ref{coverall1}).
    \begin{align}
        |\frac{\sum_{t=1}^T{\mathbb{I}_{y_{t,i,j}\in C_{t,i,j}}}}{T}-(1-\alpha)|&\leq
        |\frac{2(M+\gamma)}{T\gamma}+\frac{(j+1)(1-\alpha)}{T}|+\frac{j+1}{T}\\&\leq 2(\frac{M+\gamma}{T\gamma}+\frac{j+1}{T})
    \end{align}
    The first inequality is because $a_{t,i,j}$ is bounded according to lemma \ref{bound a}.
\end{proof}
\subsection{Proof of Theorem \ref{MSCE}}

\begin{proof}
    Similar to to \cite{gibbs2024conformal}, we first define $\beta_{t,i,j}$ as:
    $$
    \beta_{t,i,j}=min:\{\beta:y_{t,i,j}\in [\hat{y}_{t,i,j}-\beta_{t,i,j}-\hat{q}_{t,i,j},\hat{y}_{t,i,j}+\beta_{t,i,j}+\hat{q}_{t,i,j}]\}
    $$
    and we define $a^{\ast}_{t,i,j}+\hat{q}_{t,i,j}$ as the true $1-\alpha$ quantile of $y_{t,i,j}-\hat{y}_{t,i,j}$, therefore,
    $$
    P(y_{t,i,j}\in[\hat{y}_{t,i,j}-a^{\ast}_{t,i,j}-\hat{q}_{t,i,j},\hat{y}_{t,i,j}+a^{\ast}_{t,i,j}+\hat{q}_{t,i,j}])=1-\alpha
    $$
    Then $q_{t,i,j}=a^{\ast}_{t,i,j}+\hat{q}_{t,i,j}$. Besides, define pinball loss as:
    $$
    l_\alpha(\beta_{t,i,j},a_{t,i,j})=max(\alpha(a_{t,i,j}-\beta_{t,i,j}),(1-\alpha)(\beta_{t,i,j}-a_{t,i,j})))
    $$
    Therefore.
    $$
    -\frac{d}{da_{t,i,j}}l_\alpha(\beta_{t,i,j},a_{t,i,j})=1-\alpha-\mathbb{I}_{y_{t-j,i,j}\in C_{t-j,i,j}}
    $$
    and the iterate equation (\ref{update a}) can be expressed as:
    $$
    a_{t+1,i,j}=a_{t-1,i,j}-\gamma\frac{d}{da_{t-j,i,j}}l_\alpha(\beta_{t-j,i,j},a_{t-j,i,j})
    $$
    The proof can be divided as three parts. 
    
    First we bound $\sum_t{l_\alpha(\beta_{t,i,j},a_{t,i,j})}-l_\alpha(\beta_{t,i,j},a^{\ast}_{t,i,j})$ using similar calculations in Theorem 10.1 of \cite{Hazan2016IntroductionTO}. Then, we can use this bound to bound $\sum_t{(a_{t,i,j}-a^{\ast}_{t,i,j})^2}$. Finally, the bound of $MACE$ can be obtained.
    
    \textbf{Step 1}
     
    To make symbols more simple, we omit dimension index $i$ and use $l(a_{t,j}),l(a^{\ast}_{t,j})$ to represent $l_\alpha(\beta_{t,i,j},a_{t,i,j}),l_\alpha(\beta_{t,i,j},a^{\ast}_{t,i,j})$. Because pinball loss in a convex function, then, for every step $t$:
    \begin{equation}
        l(a_{t,j})-l(a_{t,j}^*)\leq l'(a_{t,j})\times(a_{t,j}-a_{t,j}^*)
        \label{convex}
    \end{equation}
    And for all $t$ such that $t>j$:
    \begin{align}
        (a_{t+1,j}-a^*_{t,j})^2&=(a_{t,j}-\gamma l'(a_{t-j,j})-a^*_{t,j})^2\\&=(a_{t,j}-a^*_{t,j})^2-2\gamma l'(a_{t-j,j})(a_{t,j}-a^*_{t,j})+(\gamma l'(a_{t-j,j}))^2
    \end{align}
    substitute into (\ref{convex}), we have the following equation for each $t>j$
    $$
    2(l(a_{t,j})-l(a_{t,j}^*))\leq\frac{(a_{t,j}-a^*_{t,j})^2-(a_{t+1,j}-a^*_{t+1,j})^2}{\gamma}+\gamma(l'(a_{t-j,j}))^2
    $$
    Then:
    \begin{align}
        \sum_{t=1}^T{[l(a_{t,j})-l(a_{t,j}^*)]}&=\sum_{t=1}^j{l(a_{t,j})-l(a_{t,j}^*)}+\sum_{t=j}^T{l(a_{t,j})-l(a_{t,j}^*)}\\
        &\leq M(j+1)+\sum_{t=j}^{T-1}{l(a_{t,j})-l(a_{t,j}^*)}\\ \label{32}
        &\leq M(j+1)+ \frac{1}{2}\sum_{t=j}^{T-1}{[\frac{(a_{t,j}-a^*_{t,j})^2-(a_{t+1,j}-a^*_{t+1,j})^2}{\gamma}+\gamma(l'(a_{t-j,j}))^2
    ]}\\
    & \leq M(j+1)+\frac{1}{2}\sum_{t=j}^{T-1}{[\frac{(a_{t,j}-a^*_{t,j})^2-(a_{t+1,j}-a^*_{t+1,j})^2}{\gamma}}]+\gamma(T-j-1)\\
    &= M(j+1)+\frac{a_{j,j}^2-a_{T,j}^2+\sum_{t=j}^{T-1}{[2a_{t+1,j}(a^*_{t+1,j}-a^*_{t,j})]}}{2\gamma}\\
    &+\frac{-2a_{j,j}a^*_{j,j}+2a_{T,j}a_{T,j}^*}{2\gamma}+\gamma(T-j-1)\\ 
    &\leq M(j+1)+\frac{c\sum_{t=j}^{T-1}{2a_{t+1,j}(a^*_{t+1,j}-a^*_{t,j})}}{2\gamma}+\gamma(T-j-1)\\ \label{37}
    & \leq M(j+1)+\sqrt{c(T-j-1)\sum_{t=j}^{T-1}{(a^*_{t+1,j}-a^*_{t,j})}}\\
    &= M(j+1)+\sqrt{c_2(T-j-1)\sum_{t=j}^{T-1}{(q_{t+1,j}-\hat{q}_{t+1,j})-(q_{t,j}-\hat{q}_{t,j})}}\\
    \end{align}
   Inequality (\ref{32}) is because of Assumption \ref{bounded assumption}. Inequality (\ref{37}) is because $a/\gamma+\gamma b\leq \sqrt{ab}$. $c$ is a constant, but can be of different values in different equations. And:
   \begin{align}
       \sum_{t=j}^{T-1}|{(q_{t+1,j}-\hat{q}_{t+1,j})-(q_{t,j}-\hat{q}_{t,j})}|&=\sum_{t=j}^{T-1}\sqrt{[(q_{t+1,j}-\hat{q}_{t+1,j})-(q_{t,j}-\hat{q}_{t,j})]^2}\\
       \end{align}
Then it equals to:
\begin{equation}
    \sum_{t=j}^{T-1}\sqrt{[(q_{t+1,j}-\hat{q}_{t+1,j})^2+(q_{t,j}-\hat{q}_{t,j})^2-2(q_{t+1,j}-\hat{q}_{t+1,j})(q_{t,j}-\hat{q}_{t,j})]}
\end{equation}
Therefore:
\begin{equation}
    \sum_{t=j}^{T-1}|{(q_{t+1,j}-\hat{q}_{t+1,j})-(q_{t,j}-\hat{q}_{t,j})}|=c(T-j-1)\sigma(q_j-\hat{q}_j) \label{38}
\end{equation}

   $\sigma$ means standard deviation, defined in \ref{MSCE}. Therefore, we have:
   \begin{equation}
       \sum_{t=1}^T{[l(a_{t,j})-l(a_{t,j}^*)]}\leq M(j+1)+c(T-j-1)\sigma(q_{j}-\hat{q}_{j})
       \label{39}
   \end{equation}
   
   \textbf{Step 2}
   
   In this step we first want to show the following equation holds for all $t,i,j$.
   \begin{equation}
       E_{\beta}[l_\alpha(\beta_{t,i,j},a_{t,i,j})-l_\alpha(\beta_{t,i,j},a^{\ast}_{t,i,j})]\leq\frac{p_2}{2}(a_{t,i,j}-a_{t,i,j}^*)^2
   \end{equation}
   where $p_2$ is the minimum value of the density function of $\beta$. Again, we omit $t,i,j$ for simplicity. We proof this inequality when $a\leq a^*$, and if $a> a^*$, the proof is the same.
   \begin{align}
       E[l_\alpha(\beta,a)-l_\alpha(\beta,a^*)]&=E\{\alpha(a-a^*)\mathbb{I}_{\beta<a}+[(1-\alpha)(\beta-\alpha)-\alpha(a^*-\beta)]\mathbb{I}_{a<\beta<a^*}\\&+(1-\alpha)(a^*-a)\mathbb{I}_{a^*\leq\beta}\}\\   \label{42}
       &=E[\alpha(a-a^*)]+(\beta-a)I_{a<\beta<a^*}+(a^*-a)I_{a^*<\beta}]\\ \label{43}
       &=E(\beta-a)\mathbb{I}_{a<\beta<a^*}\\
       &=\int_a^{a^*}(\beta-a)p(\beta)d\beta\geq\frac{p_1}{2}(a-a^*)^2
       \label{square bound}
   \end{align}
   (\ref{42}) is because $\mathbb{I}_{\beta<a}=1-I_{a<\beta<a^*}-I_{\beta\geq a^*}$ and (\ref{43}) is because $E\mathbb{I}_{\beta \geq a^*}=\alpha$. Insert this inequality into (\ref{39}), we have:
   \begin{equation}
       \sum_{t=1}^T{(a_{t,j}-a^*_{t,j})^2}\leq \frac{p_1}{2} [ M(j+1)+c(T-j-1)\sigma(q_{j}-\hat{q}_{j})]
   \end{equation}
   
   \textbf{Step 3}

   \begin{align}
         \sum_{t=1}^T{|a_{t,j}-a^*_{t,j}|}&\leq \sqrt{T\sum_{t=1}^T{(a_{t,j}-a^*_{t,j})^2}}\\
         & \leq \sqrt{c[TM(j+1)+T^2\sigma(q_j-\hat{q}_j)]}
   \end{align}
     The first inequality is because of Cauchy-Schwarz inequality. Then:
     \begin{equation}
         \frac{1}{T}\sum_{t=1}^T|(P(y_t\in C_t)-(1-\alpha))|\leq \frac{1}{T}\sum_{t=1}^TL{|a_{t,j}-a^*_{t,j}|}\leq c\sqrt{\sigma(q_j-\hat{q_j})+\frac{M(j+1)}{T}}
     \end{equation}
     The first inequality is because $P(y_t\in[\hat{y}_{t,j}-\hat{q}_{t,j}-a_{t,j},\hat{y}_{t,j}+\hat{q}_{t,j}+a_{t,j}])$ is $L$-Lipschitz to $a_{t,j}$ and $P(y_t\in[\hat{y}_{t,j}-\hat{q}_{t,j}-a_{t,j}^*,\hat{y}_{t,j}+\hat{q}_{t,j}+a_{t,j}^*])=1-\alpha$

\end{proof}

% You can have as much text here as you want. The main body must be at most $8$ pages long.
% For the final version, one more page can be added.
% If you want, you can use an appendix like this one.  

% The $\mathtt{\backslash onecolumn}$ command above can be kept in place if you prefer a one-column appendix, or can be removed if you prefer a two-column appendix.  Apart from this possible change, the style (font size, spacing, margins, page numbering, etc.) should be kept the same as the main body.
%%%%%%%%%%%%%%%%%%%%%%%%%%%%%%%%%%%%%%%%%%%%%%%%%%%%%%%%%%%%%%%%%%%%%%%%%%%%%%%
%%%%%%%%%%%%%%%%%%%%%%%%%%%%%%%%%%%%%%%%%%%%%%%%%%%%%%%%%%%%%%%%%%%%%%%%%%%%%%%
\clearpage
\section{Full Results}

\subsection{Full forecast results}
We report the full results in Table \ref{tab:full_result_1}, Table \ref{tab:full_result_2} and Table \ref{tab:full_result_3}.
\label{section:full result}
% \newgeometry{top=2cm, bottom=2cm, left=1.5cm, right=1.5cm}
\renewcommand{\arraystretch}{1.75}
\setlength{\tabcolsep}{1.5pt}
\begin{table}[!h]
\caption{Full forecast results I}
\label{tab:full_result_1}
\fontsize{6}{5}\selectfont 
\centering
\begin{tabular}{@{}c|c|cccc|cccc|cccc|cccc@{}}
\toprule
                         \multicolumn{2}{c|}{Dataset}              & \multicolumn{4}{c|}{weather}                                                                & \multicolumn{4}{c|}{traffic}                                                                & \multicolumn{4}{c|}{electricity}                                                           & \multicolumn{4}{c}{solar}                                            \\ \midrule
               Method          &      Base model        & $Cov$    & $l$                                     &$ Min\_d $                      & $Min\_t$ & $Cov$    & $l$                                     &$ Min\_d $                      & $Min\_t$& $Cov$    & $l$                                     &$ Min\_d $                      & $Min\_t$& $Cov$    & $l$                                     &$ Min\_d $                      & $Min\_t$ \\ \midrule
                         & itransformer & 94.6\% & 1.510                                 & 87.3\%                         & 93.6\%   & 87.9\% & 1.074                                 & \cellcolor[HTML]{C6E0B4}51.4\% & 87.0\%   & 86.3\% & 0.970                                 & \cellcolor[HTML]{C6E0B4}52.5\% & 85.6\%   & 87.6\% & 1.271                                 & 84.1\%   & 85.8\%   \\
                         & leddam       & 94.1\% & 1.439                                 & 83.4\%                         & 93.5\%   & 87.9\% & 1.082                                 & \cellcolor[HTML]{C6E0B4}56.0\% & 87.3\%   & 85.7\% & 0.940                                 & \cellcolor[HTML]{A9D08E}24.5\% & 84.6\%   & 86.9\% & 1.212                                 & 83.3\%   & 85.7\%   \\
\bottomrule
\end{tabular}
\end{table}
\renewcommand{\arraystretch}{1.75}
\setlength{\tabcolsep}{1.5pt}
\clearpage
\begin{table}[!h]
\fontsize{6}{5}\selectfont 
\caption{Full forecast results II}
\label{tab:full_result_2}
\centering
\begin{tabular}{@{}c|c|cccc|cccc|cccc|cccc@{}}
\toprule
                         \multicolumn{2}{c|}{Dataset}              & \multicolumn{4}{c|}{ETTh1}                                                                & \multicolumn{4}{c|}{ETTh2}                                                                & \multicolumn{4}{c|}{ETTm1}                                                           & \multicolumn{4}{c}{ETTm2}                                            \\ \midrule
                  Method       &       Base model       & $Cov$    & $l$                                     &$ Min\_d $                      & $Min\_t$ & $Cov$    & $l$                                     &$ Min\_d $                      & $Min\_t$& $Cov$    & $l$                                     &$ Min\_d $                      & $Min\_t$& $Cov$    & $l$                                     &$ Min\_d $                      & $Min\_t$ \\ \midrule
                            & itransformer & 94.3\% & 2.527                                 & 87.2\%   & 93.0\%   & 90.9\% & 1.413                                 & 82.4\%   & 89.8\%   & 92.2\% & 1.862                                 & 87.0\%   & 91.3\%   & 87.8\% & 0.997                                 & 83.6\%   & 86.6\%   \\
                            & leddam       & 94.5\% & 2.603                                 & 87.4\%   & 92.4\%   & 91.1\% & 1.354                                 & 84.4\%   & 89.6\%   & 92.9\% & 1.812                                 & 88.6\%   & 91.7\%   & 87.6\% & 0.985                                 & 83.2\%   & 87.5\%   \\
                            & softs        & 94.2\% & 2.484                                 & 86.9\%   & 92.2\%   & 91.0\% & {\color[HTML]{FF0000} \textbf{1.334}} & 84.7\%   & 89.5\%   & 93.4\% & 1.820                                 & 88.9\%   & 92.0\%   & 87.3\% & 0.973                                 & 82.7\%   & 86.8\%   \\
\bottomrule
\end{tabular}
\end{table}
\clearpage
\renewcommand{\arraystretch}{1.75}
\setlength{\tabcolsep}{1.5pt}
\begin{table}[!h]
\fontsize{6}{5}\selectfont 
\caption{Full forecast results III}
\label{tab:full_result_3}
\centering
\begin{tabular}{@{}c|c|cccc|cccc|cccc|cccc@{}}
\toprule
                         \multicolumn{2}{c|}{Dataset}                 & \multicolumn{4}{c|}{PEMS03}                                                                & \multicolumn{4}{c|}{PEMS04}                                                                & \multicolumn{4}{c|}{PEMS07}                                                           & \multicolumn{4}{c}{PEMS08}                                            \\ \midrule
                  Method       &     Base model         & $Cov$    & $l$                                     &$ Min\_d $                      & $Min\_t$ & $Cov$    & $l$                                     &$ Min\_d $                      & $Min\_t$& $Cov$    & $l$                                     &$ Min\_d $                      & $Min\_t$& $Cov$    & $l$                                     &$ Min\_d $                      & $Min\_t$ \\ \midrule
                        & itransformer & 84.4\%                     & 1.080                                 & \cellcolor[HTML]{A9D08E}42.4\% & 82.3\%               & 89.4\%                     & 1.296                                 & 72.3\%                         & 88.4\%               & 90.2\%                     & 1.113                                 & \cellcolor[HTML]{A9D08E}18.0\% & 89.6\%               & 89.3\%                     & {\color[HTML]{0070C0} 1.163}          & 84.0\%                         & 88.8\%               \\
\bottomrule
\end{tabular}
\end{table}
\subsection{Full ablation experiment results}
\label{section:ae}
We report the full result of ablation experiments in Table \ref{tab:a_result_ab}.
\renewcommand{\arraystretch}{1.75}
\setlength{\tabcolsep}{1.5pt}
\begin{table}[]
\fontsize{6}{5}\selectfont 
\caption{Forecast results of ablation experiments}
\label{tab:a_result_ab}
\centering
\begin{tabular}{@{}c|c|cccc|ccccc|ccccc@{}}
\toprule
\multicolumn{2}{c|}{Method}                                     & \multicolumn{4}{c|}{Ori}              & \multicolumn{5}{c|}{w/o update}                                       & \multicolumn{5}{c}{w/o feature}                \\ \midrule
\multicolumn{1}{c|}{Dataset}   & \multicolumn{1}{c|}{Base model} & $Cov$ & $l $                  &$ Min\_d        $      & $Min\_t$& $Cov$ & $l $                  &$ Min\_d        $      & $Min\_t$ & $C\_loss$ & $Cov$ & $l $                  &$ Min\_d        $      & $Min\_t$& $l\_loss$ \\ \midrule
                              & itransformer                   & 89.4\% & 1.061 & 88.2\%   & 89.1\%   & 81.5\% & 0.842 & 70.6\%                         & 76.6\%   & 7.9\%   & 90.5\% & 1.170 & 89.2\%   & 90.2\%   & 10.3\%  \\
                              & leddam                         & 89.3\% & 1.040 & 87.4\%   & 88.8\%   & 78.8\% & 0.776 & \cellcolor[HTML]{E2EFDA}64.6\% & 75.1\%   & 10.5\%  & 90.5\% & 1.095 & 89.1\%   & 90.2\%   & 5.2\%   \\
                              & SOFTS                          & 89.2\% & 1.086 & 87.2\%   & 88.9\%   & 78.3\% & 0.795 & \cellcolor[HTML]{E2EFDA}61.4\% & 76.0\%   & 10.9\%  & 90.4\% & 1.144 & 89.0\%   & 90.2\%   & 5.4\%   \\
\multirow{-4}{*}{weather}     & ave                            & 89.3\% & 1.062 & 87.6\%   & 88.9\%   & 79.5\% & 0.804 & \cellcolor[HTML]{E2EFDA}65.5\% & 75.9\%   & 9.8\%   & 90.5\% & 1.137 & 89.1\%   & 90.2\%   & 7.0\%   \\\midrule
                              & itransformer                   & 88.9\% & 1.067 & 81.5\%   & 87.8\%   & 85.7\% & 0.996 & \cellcolor[HTML]{C6E0B4}59.5\% & 80.4\%   & 3.2\%   & 89.4\% & 1.137 & 83.1\%   & 88.9\%   & 6.5\%   \\
                              & leddam                         & 88.7\% & 1.026 & 77.2\%   & 87.6\%   & 85.6\% & 0.953 & \cellcolor[HTML]{C6E0B4}59.0\% & 80.7\%   & 3.1\%   & 89.2\% & 1.141 & 82.4\%   & 88.9\%   & 11.2\%  \\
\bottomrule 
\end{tabular}
\end{table}

We can conclude that not dynamically updating confidence interval typically results in a decline in coverage. For example, on the weather dataset, the coverage rate without updating confidence interval is less than 80\%, whereas updating it increases the coverage to 89\%. Moreover, updating the confidence interval can significantly improve coverage in the worst-performing dimension and the worst-performing prediction time step. For instance, in electricity dataset, updating confidence interval raises the coverage in the worst-performing dimension from around 40\% to  85\%. 

However, we must acknowledge that for certain datasets, such as the solar dataset, updating the confidence interval occasionally results in worse outcomes. This may be related to the inherent characteristics of the dataset, possibly because the solar dataset has relatively minor distribution shifts.

Our experimental results also demonstrate that disabling feature-based quantile fitting leads to degraded performance across all datasets. While the coverage rates remain relatively stable , we observe 4-40\% increases in confidence interval lengths. Notably, the magnitude of this degradation directly correlates with our method's performance advantage over baselines.

A representative case is the solar dataset, where our method achieves its most significant improvement over baselines – here, feature-based error fitting reduces interval widths by over 40\%. The degree of improvement varies across datasets, suggesting this effect is closely tied to dataset-specific characteristics. 
\section{Sensitive analysis}
\subsection{Sensitive analysis of learning rate}
\label{section:sensitive of gamma}
The update rate of the interval length, \(\gamma\), is a very important parameter. Therefore, we aim to demonstrate that our method performs reasonably well for different values of \(\gamma\) through additional experiments. To achieve this, we conducted experiments where \(\gamma\) equals to 0.001, 0.005, and 0.01, and obtained the results in Figures \ref{fig:lr-MLP}. The left part of Figure \ref{fig:lr-MLP} illustrates the relationship between different $\gamma$ and average coverage, while the right part of it shows the relationship between $\gamma$ and interval length.

It can be observed from Figures \ref{fig:lr-MLP} that, when $\gamma$ increases, the coverage and interval length increase. But the increase for coverage become inconspicuous when $\gamma$ is greater than 0.005. Finally, the confidence intervals provided by our method become valid in ETTh2 dataset when $\gamma$ equals to 0.01 with length around 1.344. This is competitve result, as the best result of baselines is 1.357. 

\begin{figure*}[!h]
    \centering
    \includegraphics[width=\linewidth]{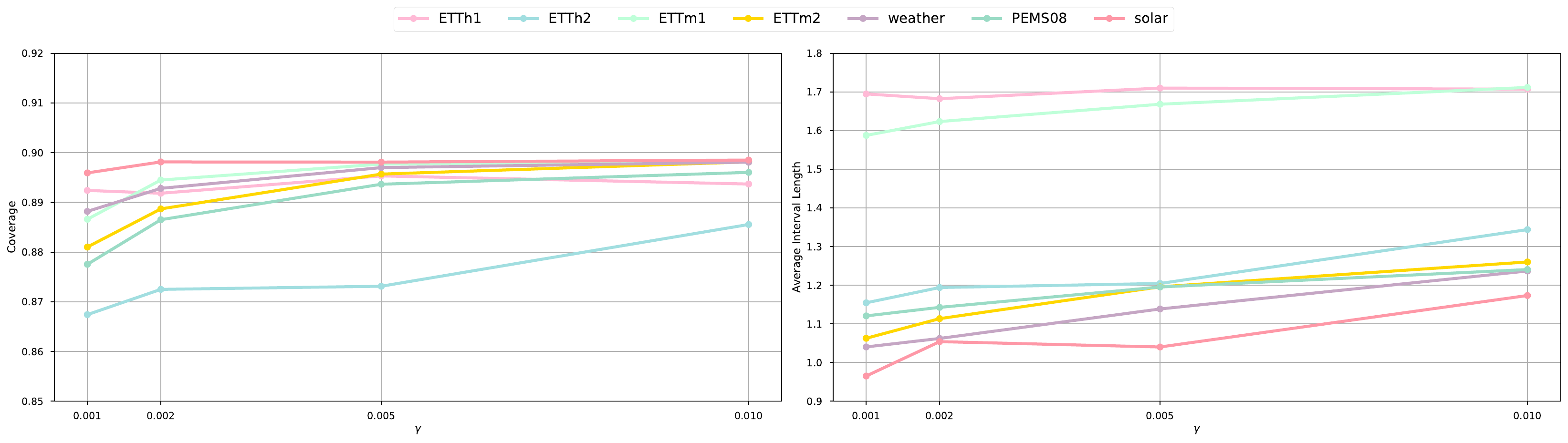}
    \caption{Results under different learning rates}
    \label{fig:lr-MLP}
\end{figure*}
\subsection{Sensitive analysis of hidden size}
Besides, we replaced the hidden size in quantile prediction model by 256,1024 and conducted additional experiments, resulting in Table \ref{tab:sensitive_hidden}. We also color the best performance \textcolor{red}{red}. 
\begin{table}[!h]

\renewcommand{\arraystretch}{2.5}
\fontsize{8}{5}\selectfont 
\centering
\caption{Results of sensitive analysis experiments of hidden size}
\label{tab:sensitive_hidden}
\begin{tabular}{@{}c|c|cccc|cccc|cccc@{}}
\toprule
\multicolumn{2}{c|}{Hidden size}          & \multicolumn{4}{c|}{256}                                                       & \multicolumn{4}{c|}{512}                                                       & \multicolumn{4}{c|}{1024}                                                      \\ \midrule
Dataset                   & Base model   & $Cov$ & $l $                  &$ Min\_d        $      & $Min\_t$& $Cov$ & $l $                  &$ Min\_d        $      & $Min\_t$ & $Cov$ & $l $                  &$ Min\_d        $      & $Min\_t$\\ \midrule
                          & iTransformer & 89.3\%       & {\color[HTML]{FF0000} 1.681} & 86.6\%         & 87.9\%         & 89.9\%       & 1.758                        & 87.5\%         & 88.3\%         & 89.8\%       & 1.703                        & 87.4\%         & 88.4\%         \\
                          & Leddam       & 89.3\%       & 1.684                        & 87.4\%         & 87.2\%         & 88.8\%       & 1.666                        & 87.7\%         & 87.6\%         & 88.5\%       & {\color[HTML]{FF0000} 1.648} & 87.3\%         & 85.7\%         \\
\multirow{-3}{*}{ETTh1}   & SOFTS        & 89.7\%       & 1.691                        & 87.6\%         & 88.2\%         & 88.8\%       & {\color[HTML]{FF0000} 1.624} & 87.3\%         & 86.6\%         & 89.2\%       & 1.653                        & 87.3\%         & 87.6\%         \\\midrule
                          & iTransformer & 87.7\%       & 1.357                        & 84.3\%         & 85.4\%         & 87.9\%       & 1.336                        & 85.7\%         & 86.9\%         & 86.5\%       & 1.241                        & 82.6\%         & 84.4\%         \\
                          & Leddam       & 86.9\%       & 1.145                        & 83.0\%         & 85.7\%         & 86.7\%       & 1.133                        & 82.8\%         & 85.4\%         & 86.2\%       & 1.117                        & 82.4\%         & 84.9\%         \\
\multirow{-3}{*}{ETTh2}   & SOFTS        & 87.3\%       & 1.120                        & 83.0\%         & 86.2\%         & 87.1\%       & 1.113                        & 82.7\%         & 85.9\%         & 86.9\%       & 1.100                        & 82.5\%         & 85.5\%         \\\midrule
                          & iTransformer & 89.6\%       & 1.698                        & 88.7\%         & 89.2\%         & 89.6\%       & 1.718                        & 88.5\%         & 89.3\%         & 89.5\%       & {\color[HTML]{FF0000} 1.690} & 88.3\%         & 89.1\%         \\
                          & Leddam       & 89.5\%       & 1.613                        & 88.3\%         & 89.0\%         & 89.4\%       & 1.588                        & 88.2\%         & 89.1\%         & 89.2\%       & {\color[HTML]{FF0000} 1.582} & 88.0\%         & 88.9\%         \\
\multirow{-3}{*}{ETTm1}   & SOFTS        & 89.5\%       & 1.591                        & 88.2\%         & 89.1\%         & 89.3\%       & 1.565                        & 87.9\%         & 89.0\%         & 89.3\%       & {\color[HTML]{FF0000} 1.555} & 87.9\%         & 88.8\%         \\\midrule
                          & iTransformer & 89.2\%       & {\color[HTML]{FF0000} 1.163} & 88.3\%         & 88.6\%         & 89.1\%       & 1.122                        & 88.0\%         & 88.4\%         & 89.3\%       & 1.155                        & 88.1\%         & 88.5\%         \\
                          & Leddam       & 88.9\%       & 1.116                        & 87.9\%         & 88.3\%         & 88.8\%       & 1.104                        & 87.8\%         & 88.3\%         & 88.7\%       & {\color[HTML]{FF0000} 1.104} & 87.7\%         & 88.1\%         \\
\multirow{-3}{*}{ETTm2}   & SOFTS        & 88.8\%       & 1.121                        & 87.8\%         & 88.3\%         & 88.7\%       & {\color[HTML]{FF0000} 1.114} & 87.7\%         & 88.1\%         & 88.7\%       & 1.115                        & 87.7\%         & 88.1\%         \\\midrule
                          & iTransformer & 89.1\%       & 1.126                        & 80.5\%         & 88.6\%         & 89.1\%       & {\color[HTML]{FF0000} 1.094} & 80.3\%         & 88.6\%         & 88.8\%       & 1.097                        & 79.7\%         & 88.3\%         \\
                          & Leddam       & 88.4\%       & 1.112                        & 76.4\%         & 87.5\%         & 88.3\%       & 1.100                        & 76.2\%         & 87.6\%         & 88.1\%       & {\color[HTML]{FF0000} 1.092} & 76.1\%         & 87.1\%         \\
\multirow{-3}{*}{PEMS08}  & SOFTS        & 88.7\%       & 1.243                        & 75.5\%         & 87.8\%         & 88.6\%       & {\color[HTML]{FF0000} 1.235} & 76.3\%         & 87.8\%         & 88.7\%       & 1.239                        & 75.6\%         & 88.0\%         \\\midrule
                          & iTransformer & 89.9\%       & 1.084                        & 89.6\%         & 89.6\%         & 89.8\%       & 1.054                        & 89.4\%         & 89.6\%         & 89.7\%       & {\color[HTML]{FF0000} 1.040} & 89.4\%         & 89.4\%         \\
                          & Leddam       & 89.8\%       & {\color[HTML]{FF0000} 0.984} & 89.5\%         & 89.5\%         & 89.9\%       & 1.037                        & 89.5\%         & 89.6\%         & 89.8\%       & 1.005                        & 89.5\%         & 89.7\%         \\
\multirow{-3}{*}{solar}   & SOFTS        & 89.7\%       & {\color[HTML]{FF0000} 0.967} & 89.2\%         & 89.3\%         & 89.8\%       & 1.072                        & 89.4\%         & 89.5\%         & 89.8\%       & 1.074                        & 89.3\%         & 89.5\%         \\\midrule
                          & iTransformer & 89.4\%       & 1.075                        & 88.3\%         & 89.0\%         & 89.4\%       & {\color[HTML]{FF0000} 1.061} & 88.2\%         & 89.1\%         & 89.5\%       & 1.096                        & 88.3\%         & 89.0\%         \\
                          & Leddam       & 89.3\%       & 1.067                        & 87.6\%         & 88.9\%         & 89.3\%       & {\color[HTML]{FF0000} 1.040} & 87.4\%         & 88.8\%         & 89.3\%       & 1.046                        & 87.2\%         & 88.9\%         \\
\multirow{-3}{*}{weather} & SOFTS        & 89.3\%       & 1.093                        & 87.3\%         & 88.8\%         & 89.2\%       & {\color[HTML]{FF0000} 1.086} & 87.2\%         & 88.9\%         & 89.3\%       & 1.110                        & 87.2\%         & 88.8\%         \\ \bottomrule
\end{tabular}
\end{table}

It can be concluded from this table that the performance of our method is generally stable with different hidden sizes. And in many cases, increasing the hidden size to 1024 might slightly improve the results.

\section{Results of other methods}
\label{section:MC}
We also conducted experiments about Monte-Carlo Dropout \cite{Gal2015DropoutAA} because it is also able to provide confidence interval without changing the original model. In these experiments, the dropout rate is fixed as 0.3. And we use the point prediction models to predict 100 times, regarding the 5\% and 95\% quantiles of these predictions as the lower and upper bounds of confidence intervals. The results are shown in Table \ref{tab:MC}.
\begin{table}[]
\centering
\renewcommand{\arraystretch}{3}
\fontsize{8}{5}\selectfont 
\caption{Results of Monte-Carlo Dropout}
\label{tab:MC}
\begin{tabular}{@{}c|cccc|cccc|cccc@{}}
\toprule
            & \multicolumn{4}{c}{iTransformer} & \multicolumn{4}{|c}{Leddam}       & \multicolumn{4}{|c}{SOFTS}        \\ \midrule
dataset     &$ Cov$    & $l $    &$ Min\_d$ & $Min\_t$ & $Cov$    & $l$     & $Min\_d$ & $Min\_t$ & $Cov $   & $l $    & $Min\_d$ & $Min\_t $\\ \hline
weather     & 19.9\% & 0.144 & 4.9\%  & 12.9\% & 27.1\% & 0.132 & 10.8\% & 23.5\% & 26.3\% & 0.117 & 13.0\% & 21.8\% \\
electricity & 28.7\% & 0.186 & 14.3\% & 25.9\% & 40.1\% & 0.257 & 13.8\% & 37.4\% & 27.9\% & 0.170 & 13.0\% & 24.2\% \\
solar       & 40.4\% & 0.200 & 36.2\% & 35.7\% & 53.3\% & 0.221 & 45.1\% & 46.7\% & 55.3\% & 0.182 & 53.8\% & 48.9\% \\
traffic     & 48.1\% & 0.245 & 7.5\%  & 44.8\% & 59.7\% & 0.347 & 23.9\% & 56.8\% & 52.5\% & 0.250 & 23.1\% & 49.5\% \\
ETTh1       & 28.3\% & 0.290 & 19.8\% & 26.3\% & 25.9\% & 0.265 & 17.0\% & 23.8\% & 24.5\% & 0.244 & 12.5\% & 22.3\% \\
ETTh2       & 19.0\% & 0.172 & 10.3\% & 15.2\% & 23.1\% & 0.169 & 18.8\% & 21.0\% & 25.6\% & 0.180 & 20.2\% & 22.1\% \\
ETTm1       & 33.5\% & 0.359 & 18.2\% & 30.9\% & 30.7\% & 0.315 & 19.9\% & 28.3\% & 31.0\% & 0.329 & 18.9\% & 28.9\% \\
ETTm2       & 25.3\% & 0.185 & 17.4\% & 23.4\% & 27.5\% & 0.169 & 23.1\% & 25.2\% & 29.8\% & 0.167 & 23.3\% & 27.1\% \\
PEMS08      & 34.2\% & 0.254 & 4.9\%  & 31.0\% & 39.5\% & 0.326 & 12.0\% & 38.0\% & 37.3\% & 0.332 & 19.0\% & 35.8\% \\
PEMS03      & 28.2\% & 0.244 & 11.2\% & 24.2\% & 34.9\% & 0.298 & 20.6\% & 32.3\% & 31.9\% & 0.303 & 21.6\% & 29.0\% \\
PEMS04      & 27.4\% & 0.220 & 4.4\%  & 24.3\% & 42.3\% & 0.336 & 15.5\% & 40.5\% & 33.5\% & 0.339 & 18.3\% & 31.3\% \\
PEMS07      & 38.4\% & 0.252 & 0.4\%  & 36.9\% & 46.4\% & 0.334 & 16.8\% & 44.5\% & 39.9\% & 0.327 & 15.8\% & 38.3\% \\ \bottomrule
\end{tabular}
\end{table}

It is obvious that the coverages of Monte-Carlo Dropout are far away from 90\%. This pattern is consistent with experiments in \cite{lin2022conformal,xu2021conformal}, where the coverages are also vary low.

Although our work primarily focuses on how to obtain confidence intervals without altering the time series forecasting model, we still conducted additional experiments where we modified the point prediction model to a interval prediction model. Specifically, we added two linear layers to predict the 5\% and 95\% quantiles of the prediction, and then used quantile loss for training. The dataset division, optimizer, hyperparameters, and other settings were consistent with those used in the previous experiments. After training the model, we deployed it on the test set and observed the prediction results, which are presented in Table \ref{tab:QR}. And the coverage below 50\% is colored in \textcolor{custompurple}{purple}.

\begin{table}[!h]
\centering
\renewcommand{\arraystretch}{3}
\fontsize{8}{5}\selectfont 
\caption{Results of quantile regression}
\label{tab:QR}
\begin{tabular}{@{}c|cccc|cccc|cccc@{}}

\toprule
            & \multicolumn{4}{c}{iTransformer}                        & \multicolumn{4}{|c}{Leddam}                              & \multicolumn{4}{|c}{SOFTS}        \\ \midrule
dataset     & $Cov$    & $l$     & $Min\_d$                         & $Min\_t$  & $Cov $   & $l$     & $Min\_d      $                   & $Min\_t$  & $Cov$    &$ l $    & $Min\_d$  & $Min\_t $ \\ \hline
weather     & 76.3\% & 0.633 & 61.7\%                        & 72.0\% & 75.9\% & 0.638 & 58.6\%                        & 71.6\% & 78.3\% & 0.715 & 68.3\% & 73.4\% \\
electricity & 87.2\% & 0.935 & 76.5\%                        & 85.9\% & 85.5\% & 0.880 & 68.3\%                        & 83.2\% & 88.6\% & 0.971 & 73.1\% & 87.6\% \\
solar       & 90.4\% & 0.897 & 87.3\%                        & 77.3\% & 87.4\% & 0.779 & 84.9\%                        & 75.3\% & 88.5\% & 0.860 & 86.3\% & 72.2\% \\
traffic     & 87.2\% & 0.936 & 58.9\%                        & 85.6\% & 86.8\% & 0.977 & 66.1\%                        & 85.1\% & 88.9\% & 1.014 & 76.2\% & 86.8\% \\
ETTh1       & 86.9\% & 1.485 & 82.6\%                        & 84.6\% & 89.5\% & 1.540 & 86.0\%                        & 88.3\% & 89.2\% & 1.529 & 84.9\% & 87.8\% \\
ETTh2       & 85.8\% & 0.995 & 78.7\%                        & 82.7\% & 85.9\% & 1.083 & 82.0\%                        & 84.1\% & 86.2\% & 1.084 & 81.5\% & 84.1\% \\
ETTm1       & 77.8\% & 1.035 & 74.9\%                        & 73.9\% & 82.8\% & 1.223 & 79.3\%                        & 80.7\% & 81.4\% & 1.187 & 71.2\% & 78.8\% \\
ETTm2       & 83.4\% & 0.822 & 77.6\%                        & 80.1\% & 82.7\% & 0.793 & 78.8\%                        & 80.4\% & 85.7\% & 0.938 & 81.6\% & 84.3\% \\
PEMS08      & 83.5\% & 0.994 & {\color[HTML]{7030A0} 23.2\%} & 80.9\% & 78.1\% & 0.958 & 50.4\%                        & 74.5\% & 86.2\% & 1.266 & 66.1\% & 84.3\% \\
PEMS03      & 78.3\% & 0.920 & 51.6\%                        & 73.5\% & 76.4\% & 0.827 & {\color[HTML]{7030A0} 47.1\%} & 70.1\% & 81.9\% & 1.164 & 67.0\% & 78.4\% \\
PEMS04      & 85.4\% & 1.063 & {\color[HTML]{7030A0} 37.9\%} & 83.7\% & 84.4\% & 0.967 & {\color[HTML]{7030A0} 49.5\%} & 81.8\% & 87.9\% & 1.453 & 73.7\% & 86.8\% \\
PEMS07      & 87.5\% & 1.016 & {\color[HTML]{7030A0} 6.5\%}  & 86.3\% & 83.0\% & 0.796 & {\color[HTML]{7030A0} 29.1\%} & 80.6\% & 88.2\% & 1.228 & 66.6\% & 87.0\% \\ \bottomrule
\end{tabular}
\end{table}

We can observe that directly using quantile regression makes it difficult to guarantee coverage. For example, the coverage for the weather dataset does not even reach 80\%. Furthermore, when examining the coverage for the worst dimension and the worst prediction time step, the results are even poorer. For instance, in the case of the PEMS08 dataset, the coverage in the worst dimension can even fall below 50\%. Similarly, for the solar dataset using SOFTS, although the overall coverage reaches 88.5\%, the coverage for the worst time step drops to just 72\%. These results suggest that time series datasets often suffer from distribution shift, meaning that the model trained on the training set is often not able to provide valid confidence intervals on the test set. Moreover, the degree of distribution shift may vary across different dimensions of the dataset. It can lead to a scenario where, although the overall coverage might still be relatively good, the coverage for the worst dimension is very low.
\section{Discussion on $MACE$}
\label{section:MACE}
While $MACE$ (Mean Absolute Coverage Error) cannot be directly observed, we approximate it by computing local coverage rates over sequential prediction windows. Specifically, we calculate the proportion of true values contained within the prediction intervals across every consecutive 100-step forecast horizon. We then measure the deviation between these local coverage rates and the target coverage level to estimate MACE.

In our experiments, only ECI, ACI, PIC, and FFDCI demonstrated consistent coverage guarantees across most datasets. Table \ref{tab:mace} therefore reports the approximated MACE values for these four algorithms.
\begin{table}[!h]
\caption{Results of approximate $MACE$}
\label{tab:mace}
\centering
\renewcommand{\arraystretch}{3}
\fontsize{8}{5}\selectfont 
\begin{tabular}{@{}c|cccc|cccc|cccc@{}}
\toprule
      Model      & \multicolumn{4}{c|}{iTransformer}                & \multicolumn{4}{c|}{Leddam}                      & \multicolumn{4}{c}{SOFTS}                       \\ \midrule
dataset     & ECI            & ACI   & PID   & FFDCI          & ECI            & ACI   & PID   & FFDCI          & ECI            & ACI   & PID   & FFDCI          \\\midrule
ETTh1       & 0.077          & 0.080 & 0.072 & \textbf{0.071} & 0.077          & 0.077 & 0.071 & \textbf{0.064} & 0.077          & 0.080 & 0.074 & \textbf{0.064} \\
ETTh2       & 0.103          & 0.106 & 0.109 & \textbf{0.098} & 0.104          & 0.104 & 0.110 & \textbf{0.098} & 0.106          & 0.103 & 0.110 & \textbf{0.097} \\
ETTm1       & 0.085          & 0.081 & 0.088 & \textbf{0.082} & 0.084          & 0.080 & 0.087 & \textbf{0.080} & 0.086          & 0.083 & 0.090 & \textbf{0.082} \\
ETTm2       & 0.102          & 0.106 & 0.103 & \textbf{0.092} & 0.102          & 0.011 & 0.102 & \textbf{0.092} & 0.105          & 0.108 & 0.105 & \textbf{0.092} \\
PEMS03      & 0.092          & 0.084 & 0.089 & \textbf{0.076} & 0.095          & 0.085 & 0.092 & \textbf{0.082} & 0.094          & 0.089 & 0.091 & \textbf{0.073} \\
PEMS04      & 0.086          & 0.078 & 0.084 & \textbf{0.064} & 0.090          & 0.077 & 0.088 & \textbf{0.068} & 0.090          & 0.081 & 0.087 & \textbf{0.072} \\
PEMS07      & 0.077          & 0.069 & 0.076 & \textbf{0.059} & 0.081          & 0.070 & 0.080 & \textbf{0.067} & 0.084          & 0.075 & 0.081 & \textbf{0.062} \\
PEMS08      & 0.083          & 0.074 & 0.083 & \textbf{0.063} & 0.084          & 0.078 & 0.081 & \textbf{0.073} & 0.087          & 0.078 & 0.086 & \textbf{0.068} \\
solar       & 0.091          & 0.075 & 0.084 & \textbf{0.072} & 0.091          & 0.078 & 0.085 & \textbf{0.072} & 0.096          & 0.080 & 0.088 & \textbf{0.076} \\
weather     & 0.107          & 0.109 & 0.111 & \textbf{0.104} & 0.108          & 0.109 & 0.111 & \textbf{0.100} & 0.113          & 0.109 & 0.117 & \textbf{0.105} \\
traffic     & \textbf{0.055} & 0.058 & 0.058 & \textbf{0.055} & \textbf{0.054} & 0.057 & 0.058 & 0.056          & \textbf{0.056} & 0.058 & 0.058 & 0.057          \\
electriticy & 0.069          & 0.076 & 0.067 & \textbf{0.064} & 0.070          & 0.756 & 0.067 & \textbf{0.064} & 0.071          & 0.077 & 0.069 & \textbf{0.066} \\ \bottomrule
\end{tabular}
\end{table}

The results in Table \ref{tab:mace} demonstrate that our algorithm achieves the smallest MACE values in most cases, indicating that its local coverage rates align more closely with the target coverage level compared to other methods.

Furthermore, Figure \ref{fig:loc_cov} presents the temporal evolution of local coverage rates across different datasets for all evaluated algorithms. The figure also shows that our method maintains local coverage rates consistently nearer to 90\%.
\newpage
\begin{figure}[!h]

    \begin{subfigure}[b]{0.48\textwidth}
        \includegraphics[width=\linewidth]{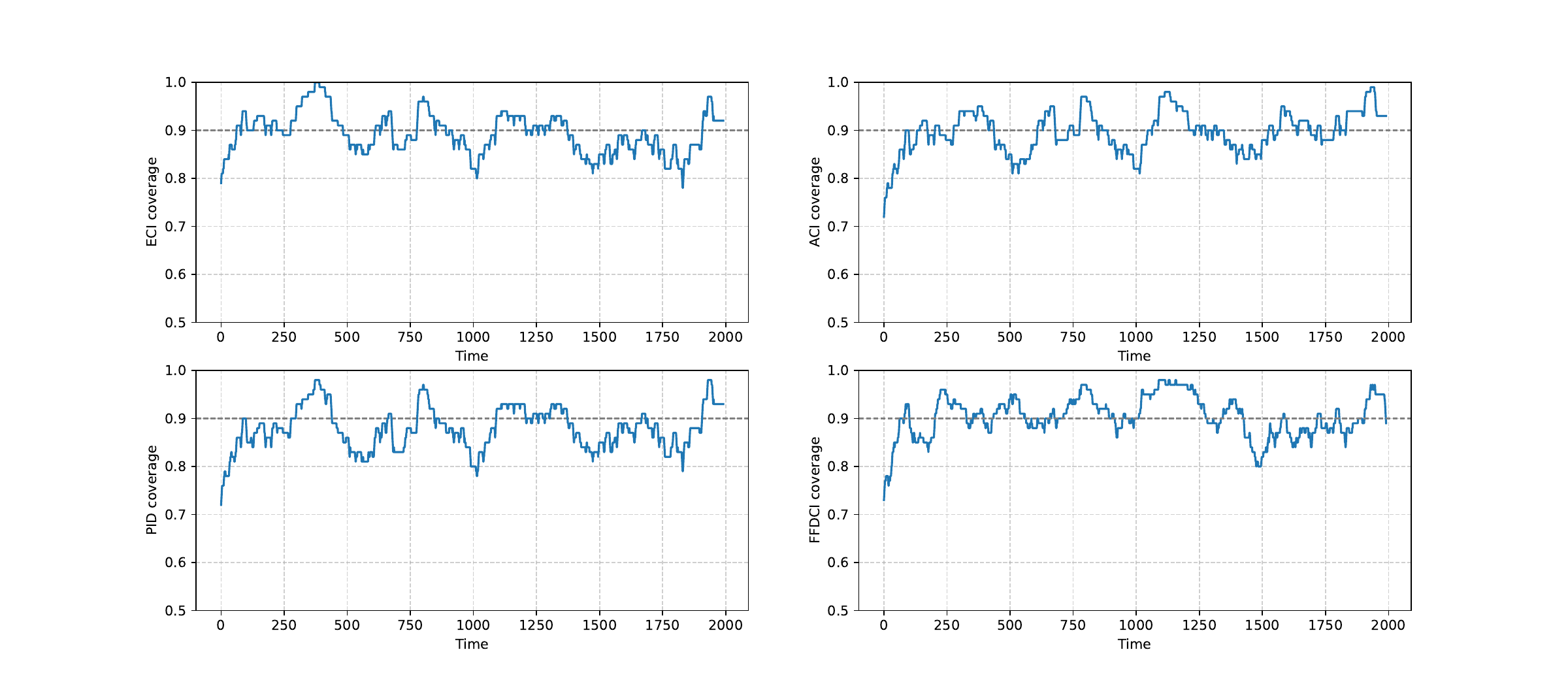}
        \caption{ETTh1}
        \label{subfig:ETTh1}
    \end{subfigure}
    \hspace{0.cm} % 添加水平间距
    \begin{subfigure}[b]{0.48\textwidth}
        \includegraphics[width=\linewidth]{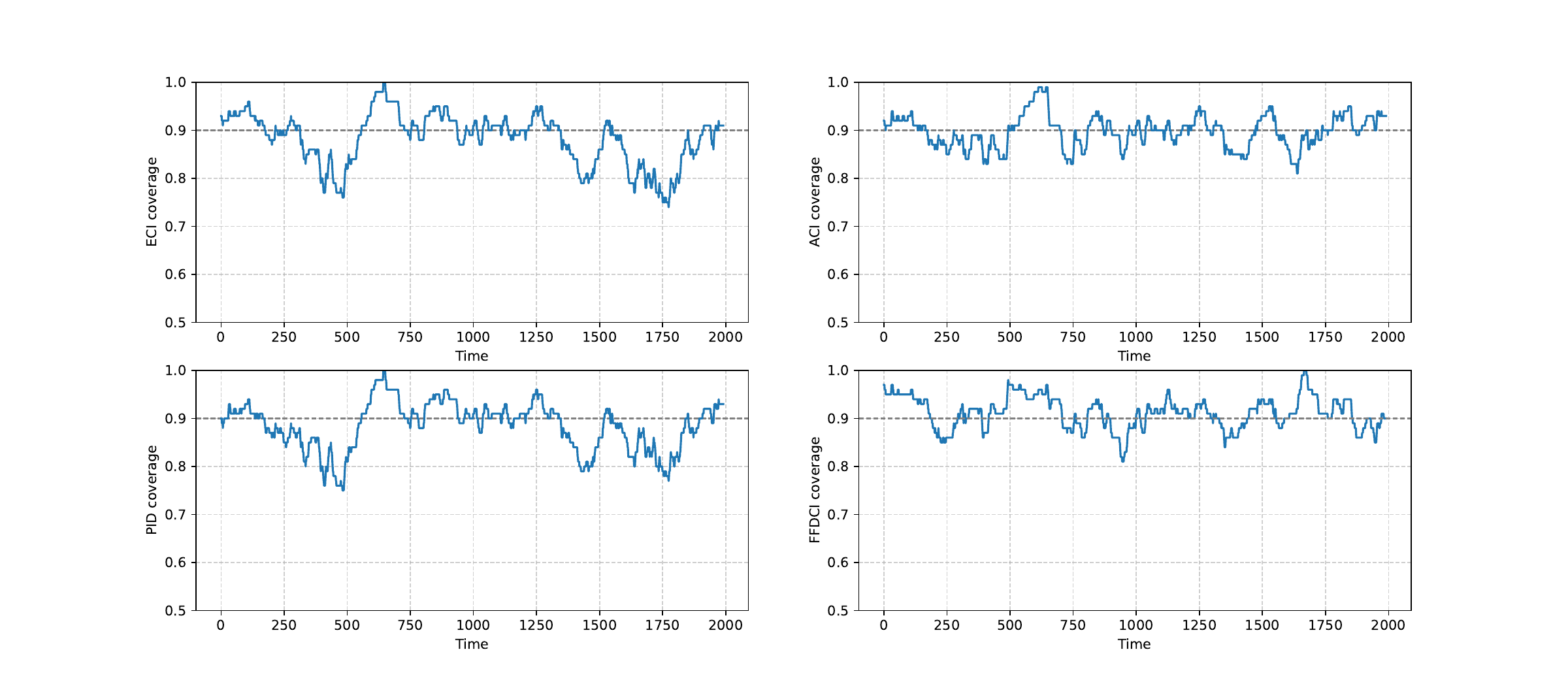}
        \caption{ETTh2}
        \label{subfig:ETTh2}
    \end{subfigure}
    
    \vspace{0.cm} % 添加垂直间距
    
    \begin{subfigure}[b]{0.48\textwidth}
        \includegraphics[width=\linewidth]{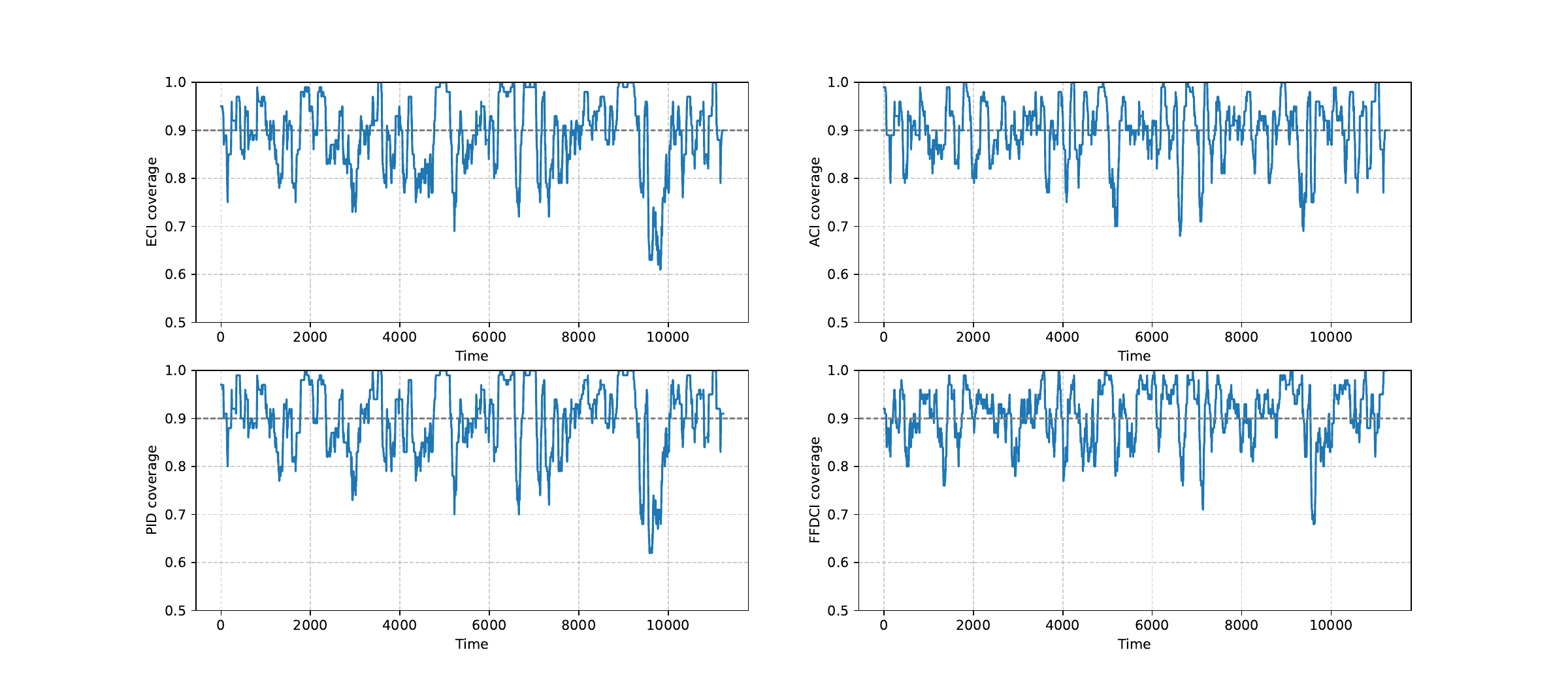}
        \caption{ETTm1}
        \label{subfig:ETTm1}
    \end{subfigure}
    \hfill % 添加水平间距
    \begin{subfigure}[b]{0.48\textwidth}
        \includegraphics[width=\linewidth]{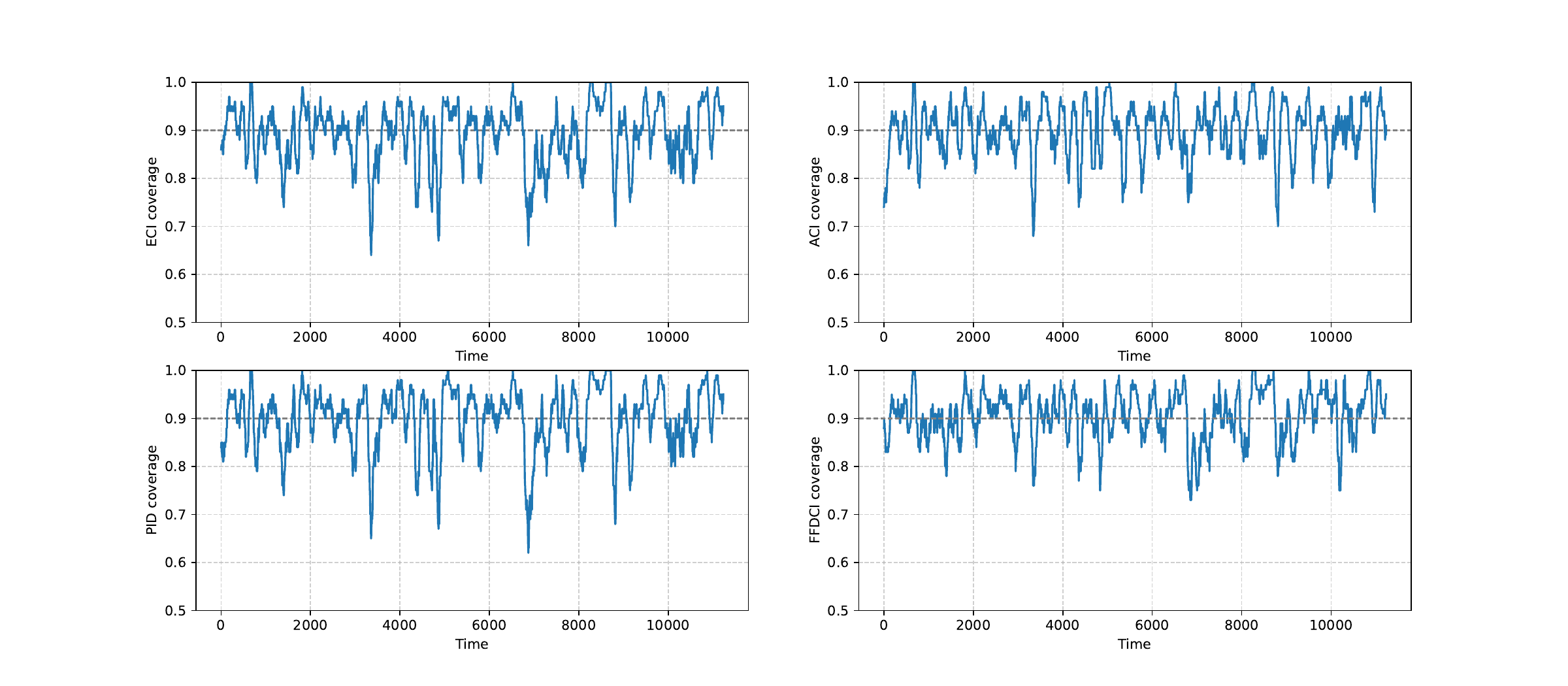}
        \caption{ETTm2}
        \label{subfig:ETTm2}
    \end{subfigure}
    \vspace{0.cm} % 添加垂直间距
    
    \begin{subfigure}[b]{0.48\textwidth}
        \includegraphics[width=\linewidth]{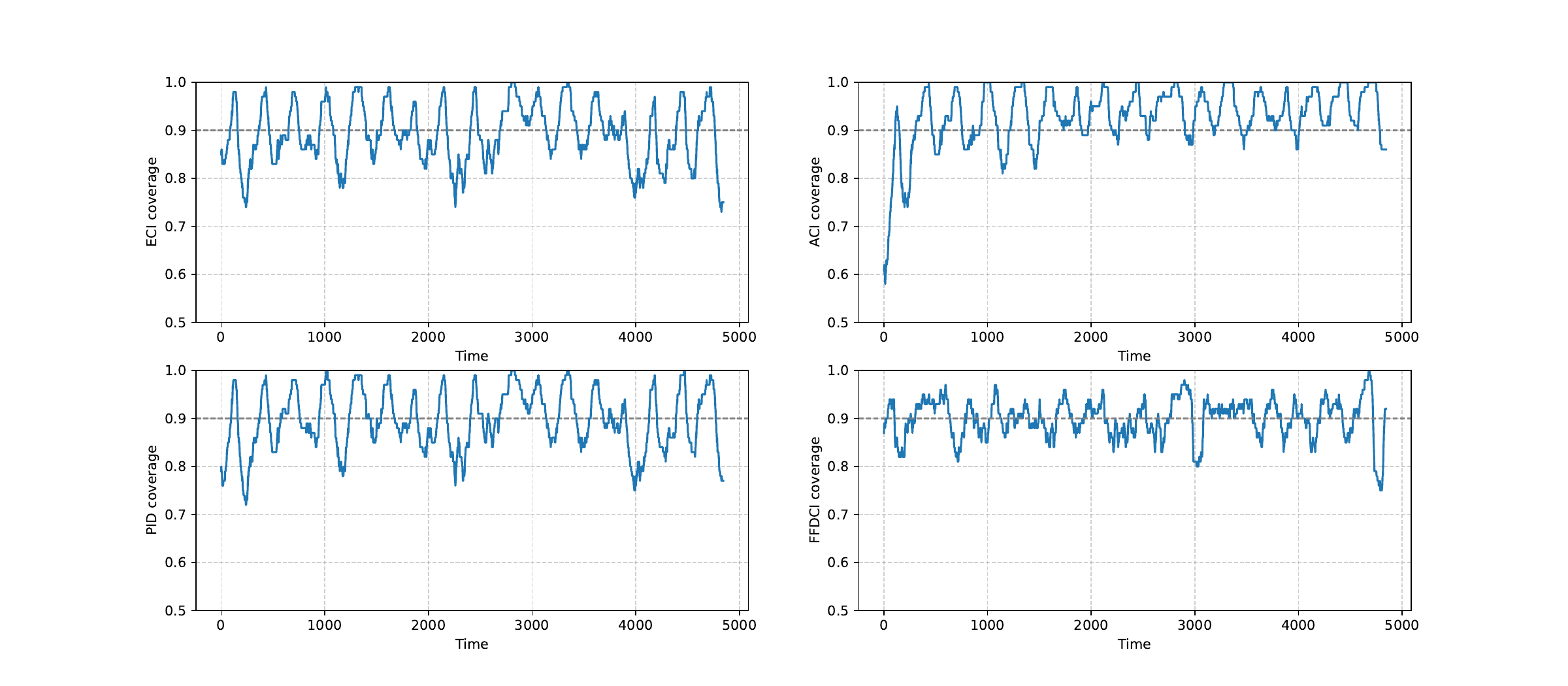}
        \caption{PEMS03}
        \label{subfig:PEMS03}
    \end{subfigure}
    \hfill % 添加水平间距
    \begin{subfigure}[b]{0.48\textwidth}
        \includegraphics[width=\linewidth]{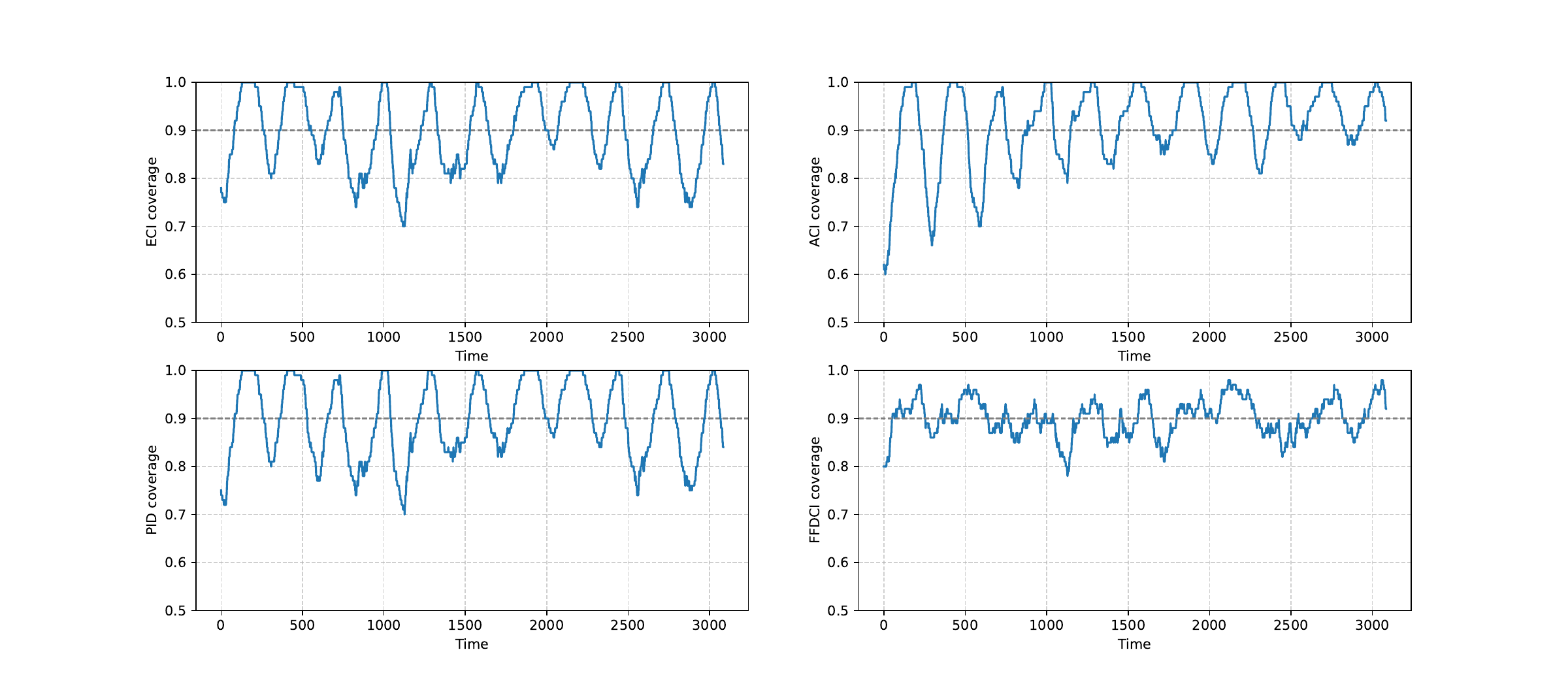}
        \caption{PEMS04}
        \label{subfig:PEMS04}
    \end{subfigure}
    \vspace{0.cm} % 添加垂直间距
    
    \begin{subfigure}[b]{0.48\textwidth}
        \includegraphics[width=\linewidth]{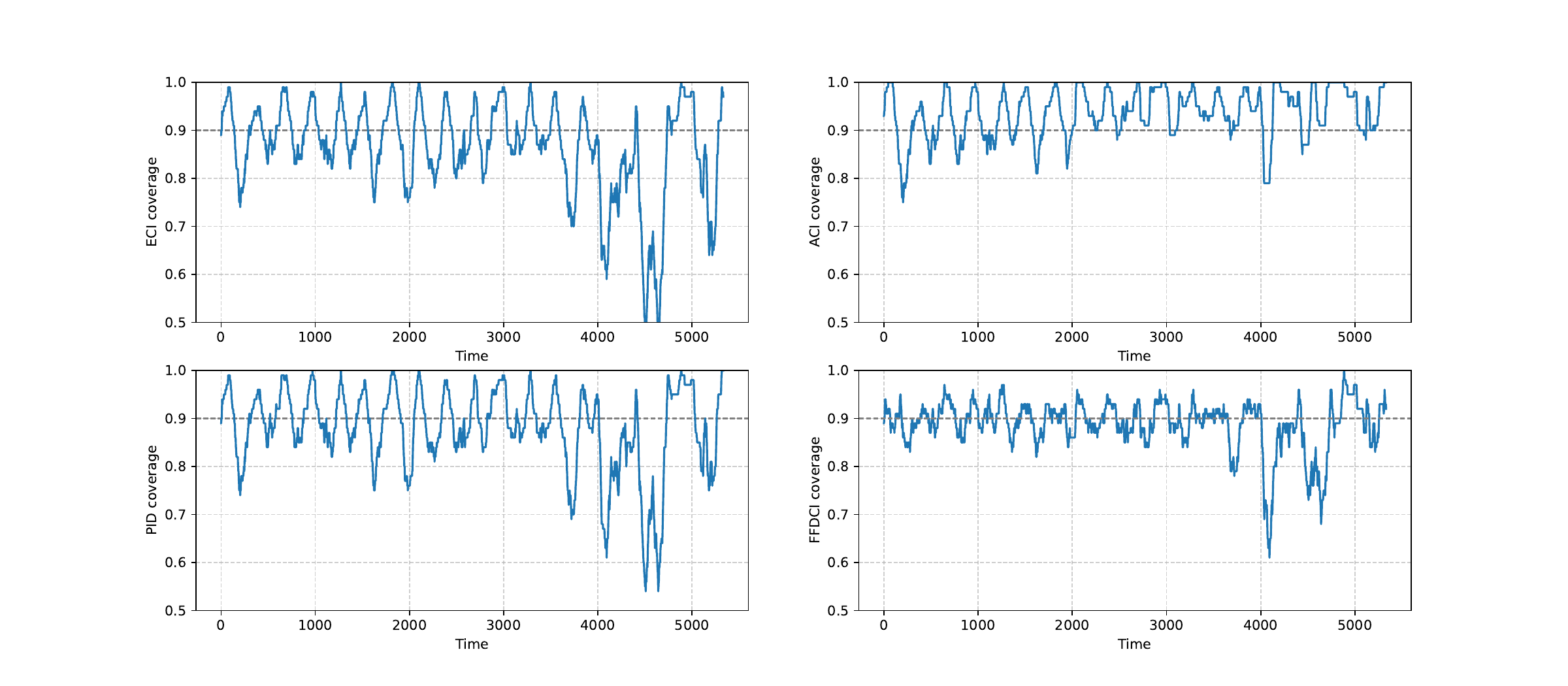}
        \caption{PEMS07}
        \label{subfig:PEMS07}
    \end{subfigure}
    \hfill % 添加水平间距
    \begin{subfigure}[b]{0.48\textwidth}
        \includegraphics[width=\linewidth]{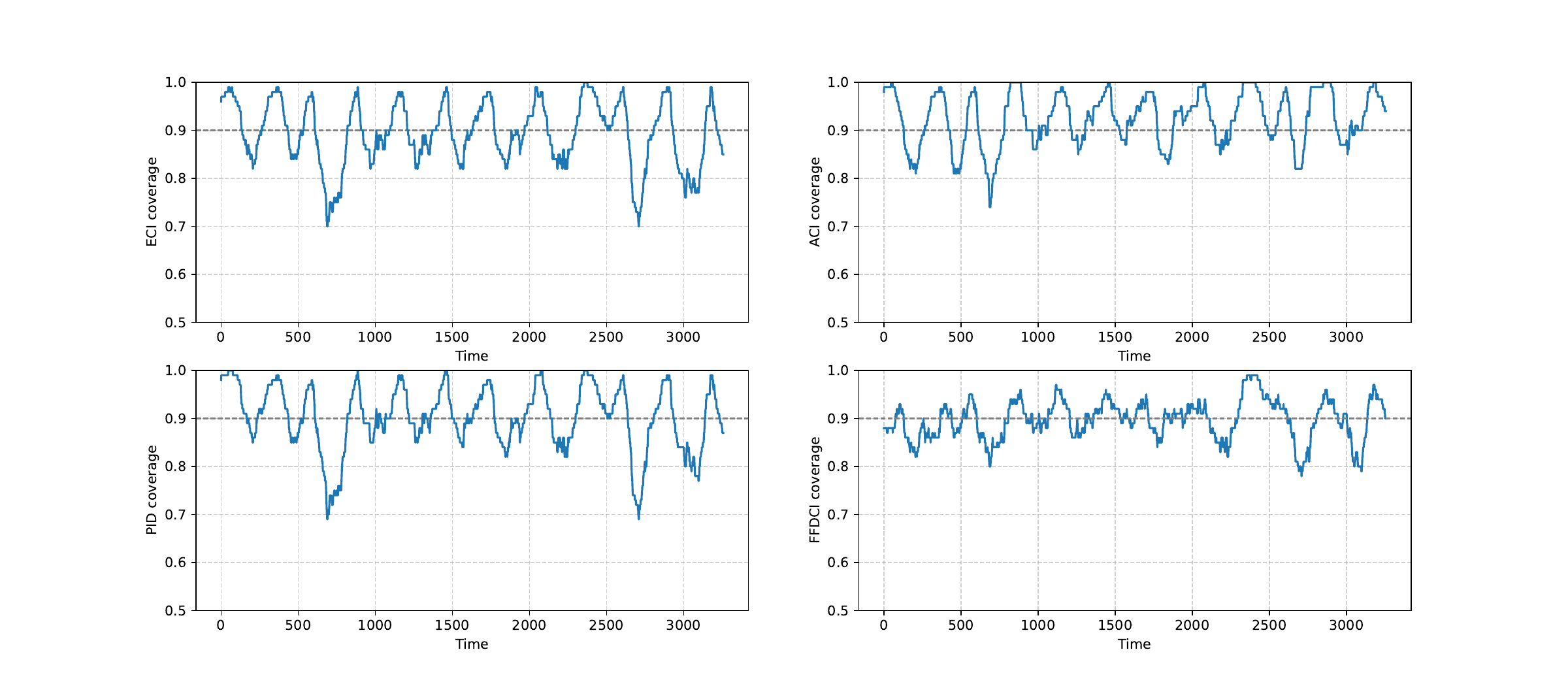}
        \caption{PEMS08}
        \label{subfig:PEMS08}
    \end{subfigure}
    \vspace{0.cm} % 添加垂直间距
    
    \begin{subfigure}[b]{0.48\textwidth}
        \includegraphics[width=\linewidth]{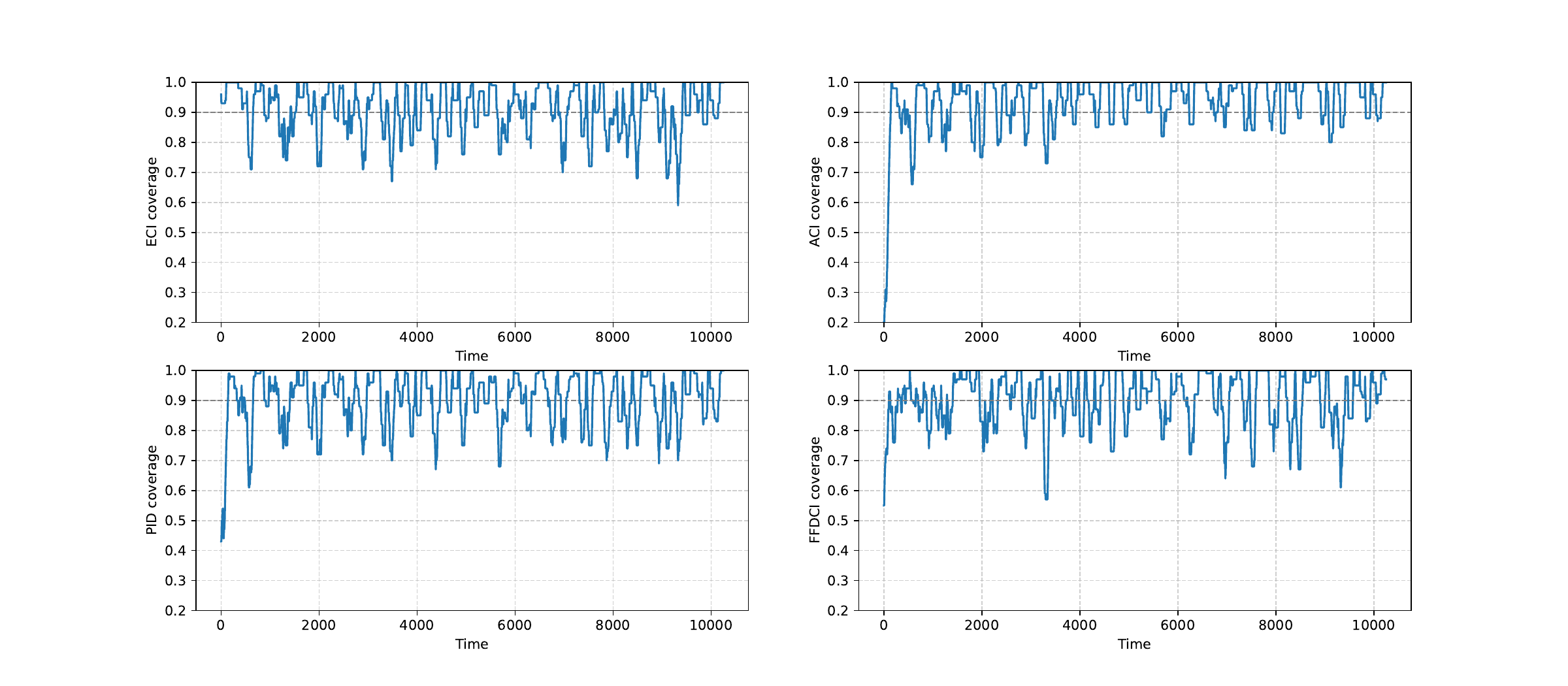}
        \caption{weather}
        \label{subfig:weather}
    \end{subfigure}
    \hfill % 添加水平间距
    \begin{subfigure}[b]{0.48\textwidth}
        \includegraphics[width=\linewidth]{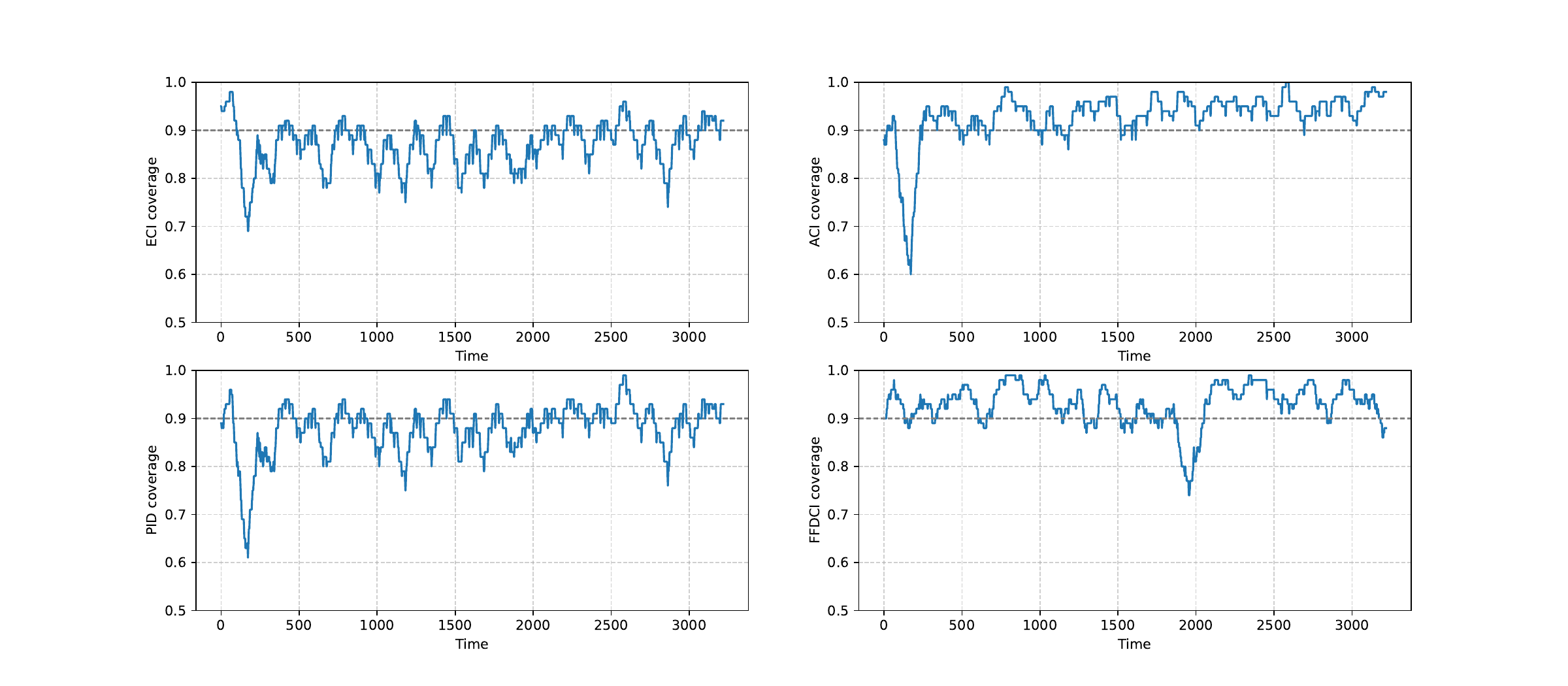}
        \caption{traffic}
        \label{subfig:traffic}
    \end{subfigure}
    \vspace{0.cm} % 添加垂直间距
    
    \begin{subfigure}[b]{0.48\textwidth}
        \includegraphics[width=\linewidth]{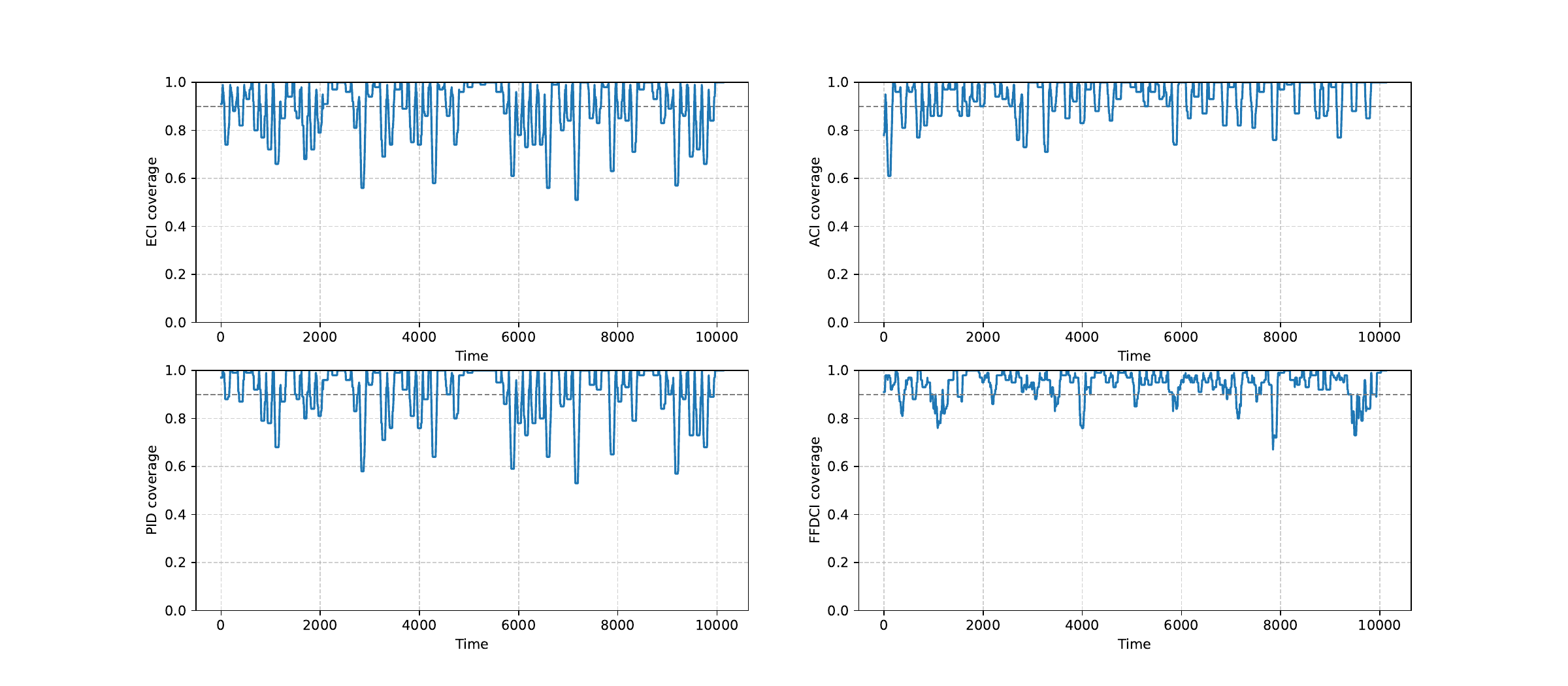}
        \caption{solar}
        \label{subfig:solar}
    \end{subfigure}
    \hfill % 添加水平间距
    \begin{subfigure}[b]{0.48\textwidth}
        \includegraphics[width=\linewidth]{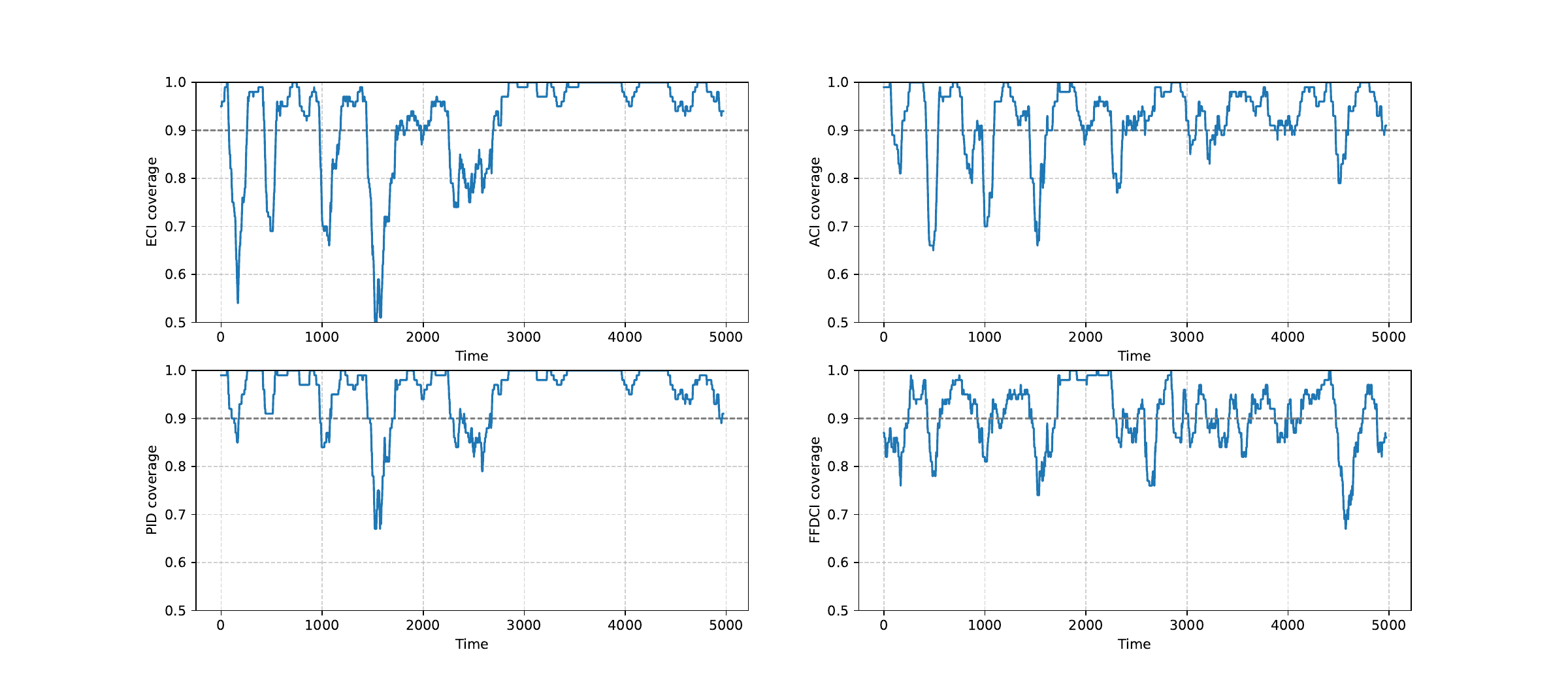}
        \caption{electricity}
        \label{subfig:electricity}
    \end{subfigure}
    
      \caption{Figure of local coverage}
      \label{fig:loc_cov}
\end{figure}
\section{Extended version of FFDCI}
We have developed several extensions to our proposed algorithm: (1) The first extension adopts the ECI approach for updating confidence interval lengths (see Equation \ref{eq:eci}); (2) The second extension replaces the MLP with an LSTM model (with a single 512-dimensional hidden layer) for error quantile prediction, better suited for time series forecasting; (3) The third extension incorporates the SFOGD methodology, adaptively determining the update rate for confidence interval lengths using accumulated second-order gradient information. Specifically, we modify the parameter $a_{t+1,i,j}$ update rule as follows:
\begin{equation}
    a_{t+1,i,j}=a_{t,i,j}+\frac{\gamma}{\sqrt{\sum_{i=1}^{t-j}(1-\mathbb{I}_{y_{t-j,i,j}\in C_{t-j,i,j}}-\alpha)^2}} (1-\mathbb{I}_{y_{t-j,i,j}\in C_{t-j,i,j}}-\alpha)
\end{equation}

The following Table \ref{tab:result_ex} presents the performance of the original algorithm and its three extended versions. We highlight the confidence interval lengths where extended versions outperform the original version in {\color[HTML]{FF0000}red}, and \textbf{bold} the interval lengths of the best-performing method among all four approaches (three extended versions plus the original). Additionally, we report the frequency of each of these four methods achieving optimal results in the last raw.
\clearpage
\renewcommand{\arraystretch}{1.75}
\setlength{\tabcolsep}{1.pt}
\begin{table}[!h]
\caption{Forecast results of extended FFDCI}
\label{tab:result_ex}
\fontsize{5.5}{5}\selectfont 
\centering
\begin{tabular}{@{}c|c|cccc|cccc|cccc|cccc@{}}
\toprule
\multicolumn{2}{c|}{Method}                   & \multicolumn{4}{c|}{DDFCI-ori}                                        & \multicolumn{4}{c|}{DDFCI-ECI}                                        & \multicolumn{4}{c|}{DDFCI-LSTM}                                       & \multicolumn{4}{c}{DDFCI-SFOGD}                                                            \\ \midrule
Dataset                       & Base  model & $Cov$    & $l$     & $Min\_d$                         & $Min\_t$  &$Cov$    & $l$     & $Min\_d$                         & $Min\_t$  &$Cov$    & $l$     & $Min\_d$                         & $Min\_t$  &$Cov$    & $l$     & $Min\_d$                         & $Min\_t$   \\ \midrule
                               & itransformer & 89.4\% & 1.061          & 88.2\%   & 89.1\%   & 89.5\% & {\color[HTML]{FF0000} 1.046}          & 88.0\%   & 88.7\%   & 89.2\% & 1.119                                 & 87.3\%   & 88.4\%   & 91.2\% & {\color[HTML]{FF0000} \textbf{1.032}} & 84.8\%                         & 88.4\%   \\
                              & leddam       & 89.3\% & \textbf{1.040} & 87.4\%   & 88.8\%   & 89.5\% & 1.074                                 & 87.2\%   & 88.8\%   & 89.3\% & 1.054                                 & 87.4\%   & 88.9\%   & 90.1\% & 1.051                                 & 81.6\%                         & 87.0\%   \\
                              & SOFTS        & 89.2\% & 1.086          & 87.2\%   & 88.9\%   & 89.4\% & 1.101                                 & 86.6\%   & 89.0\%   & 89.4\% & 1.116                                 & 87.3\%   & 89.0\%   & 90.0\% & {\color[HTML]{FF0000} \textbf{1.069}} & 81.5\%                         & 87.6\%   \\
\multirow{-4}{*}{weather}     & ave          & 89.3\% & 1.062          & 87.6\%   & 88.9\%   & 89.5\% & 1.074                                 & 87.3\%   & 88.8\%   & 89.3\% & 1.096                                 & 87.3\%   & 88.8\%   & 90.4\% & {\color[HTML]{FF0000} \textbf{1.050}} & 82.7\%                         & 87.7\%   \\\midrule
                              & itransformer & 88.9\% & 1.067          & 81.5\%   & 87.8\%   & 89.6\% & 1.076                                 & 82.2\%   & 88.6\%   & 88.2\% & {\color[HTML]{FF0000} \textbf{1.016}} & 80.1\%   & 87.7\%   & 90.8\% & 1.105                                 & 70.9\%                         & 88.5\%   \\
                              & leddam       & 88.7\% & \textbf{1.026} & 77.2\%   & 87.6\%   & 89.2\% & 1.042                                 & 77.0\%   & 88.3\%   & 88.7\% & 1.028                                 & 77.5\%   & 88.0\%   & 90.1\% & 1.060                                 & \cellcolor[HTML]{E2EFDA}67.1\% & 88.3\%   \\
                              & SOFTS        & 88.3\% & \textbf{1.048} & 78.0\%   & 87.4\%   & 89.7\% & 1.085                                 & 78.5\%   & 88.5\%   & 88.4\% & 1.050                                 & 77.6\%   & 87.7\%   & 91.0\% & 1.119                                 & \cellcolor[HTML]{E2EFDA}69.1\% & 88.7\%   \\
\multirow{-4}{*}{traffic}     & ave          & 88.6\% & 1.047          & 78.9\%   & 87.6\%   & 89.5\% & 1.068                                 & 79.3\%   & 88.5\%   & 88.4\% & {\color[HTML]{FF0000} \textbf{1.032}} & 78.4\%   & 87.8\%   & 90.6\% & 1.095                                 & \cellcolor[HTML]{E2EFDA}69.1\% & 88.5\%   \\\midrule
                              & itransformer & 89.4\% & 1.105          & 86.7\%   & 89.0\%   & 91.0\% & 1.161                                 & 86.4\%   & 90.4\%   & 89.1\% & {\color[HTML]{FF0000} \textbf{1.100}} & 85.1\%   & 88.8\%   & 93.7\% & 1.271                                 & 83.0\%                         & 92.5\%   \\
                              & leddam       & 89.1\% & \textbf{1.007} & 84.6\%   & 88.6\%   & 89.9\% & 1.042                                 & 83.2\%   & 89.0\%   & 88.7\% & 1.014                                 & 84.0\%   & 88.4\%   & 91.0\% & 1.073                                 & \cellcolor[HTML]{E2EFDA}65.6\% & 88.6\%   \\
                              & SOFTS        & 89.3\% & 1.051          & 84.3\%   & 88.9\%   & 89.7\% & {\color[HTML]{FF0000} 1.042}          & 82.9\%   & 89.2\%   & 88.8\% & {\color[HTML]{FF0000} \textbf{1.024}} & 83.9\%   & 88.5\%   & 90.8\% & 1.069                                 & \cellcolor[HTML]{E2EFDA}65.2\% & 89.3\%   \\
\multirow{-4}{*}{electricity} & ave          & 89.3\% & 1.054          & 85.2\%   & 88.8\%   & 90.2\% & 1.081                                 & 84.2\%   & 89.5\%   & 88.9\% & {\color[HTML]{FF0000} \textbf{1.046}} & 84.3\%   & 88.5\%   & 91.8\% & 1.137                                 & 71.3\%                         & 90.1\%   \\\midrule
                              & itransformer & 89.8\% & 1.054          & 89.4\%   & 89.6\%   & 91.1\% & 1.121                                 & 90.7\%   & 90.5\%   & 89.6\% & {\color[HTML]{FF0000} \textbf{1.036}} & 89.2\%   & 89.4\%   & 96.2\% & 1.258                                 & 94.8\%                         & 94.8\%   \\
                              & leddam       & 89.9\% & 1.037          & 89.5\%   & 89.6\%   & 90.7\% & {\color[HTML]{FF0000} \textbf{0.979}} & 89.8\%   & 90.2\%   & 89.8\% & 1.042                                 & 89.4\%   & 89.6\%   & 94.8\% & 1.109                                 & 92.5\%                         & 93.0\%   \\
                              & SOFTS        & 89.8\% & 1.072          & 89.4\%   & 89.5\%   & 90.5\% & {\color[HTML]{FF0000} \textbf{0.950}} & 89.6\%   & 90.0\%   & 89.4\% & {\color[HTML]{FF0000} 0.963}          & 88.9\%   & 89.1\%   & 95.0\% & 1.078                                 & 93.0\%                         & 93.9\%   \\
\multirow{-4}{*}{solar}       & ave          & 89.8\% & 1.054          & 89.4\%   & 89.6\%   & 90.8\% & {\color[HTML]{FF0000} 1.017}          & 90.0\%   & 90.2\%   & 89.6\% & {\color[HTML]{FF0000} \textbf{1.013}} & 89.1\%   & 89.4\%   & 95.3\% & 1.148                                 & 93.5\%                         & 93.9\%   \\\midrule
                              & itransformer & 89.6\% & 1.718          & 88.5\%   & 89.3\%   & 90.2\% & 1.740                                 & 88.9\%   & 89.3\%   & 89.3\% & {\color[HTML]{FF0000} \textbf{1.668}} & 88.2\%   & 89.0\%   & 89.6\% & {\color[HTML]{FF0000} \textbf{1.688}} & 85.9\%                         & 87.6\%   \\
                              & leddam       & 89.4\% & 1.588          & 88.2\%   & 89.1\%   & 89.6\% & 1.605                                 & 87.8\%   & 88.7\%   & 89.6\% & 1.622                                 & 88.3\%   & 89.4\%   & 88.3\% & {\color[HTML]{FF0000} \textbf{1.523}} & 83.1\%                         & 85.2\%   \\
                              & SOFTS        & 89.3\% & 1.565          & 87.9\%   & 89.0\%   & 89.7\% & 1.585                                 & 87.2\%   & 88.9\%   & 89.4\% & {\color[HTML]{FF0000} 1.550}          & 88.1\%   & 89.2\%   & 88.3\% & {\color[HTML]{FF0000} \textbf{1.508}} & 81.4\%                         & 86.4\%   \\
\multirow{-4}{*}{ETTm1}       & ave          & 89.4\% & 1.623          & 88.2\%   & 89.1\%   & 89.8\% & 1.643                                 & 88.0\%   & 89.0\%   & 89.4\% & {\color[HTML]{FF0000} 1.613}          & 88.2\%   & 89.2\%   & 88.7\% & {\color[HTML]{FF0000} \textbf{1.573}} & 83.5\%                         & 86.4\%   \\\midrule
                              & itransformer & 89.1\% & \textbf{1.122} & 88.0\%   & 88.4\%   & 89.6\% & 1.144                                 & 88.3\%   & 88.8\%   & 89.1\% & 1.133                                 & 88.2\%   & 88.7\%   & 91.1\% & 1.131                                 & 85.7\%                         & 88.7\%   \\
                              & leddam       & 88.8\% & 1.104          & 87.8\%   & 88.3\%   & 88.8\% & {\color[HTML]{FF0000} 1.087}          & 87.5\%   & 88.0\%   & 88.8\% & 1.113                                 & 87.9\%   & 88.3\%   & 88.6\% & {\color[HTML]{FF0000} \textbf{1.003}} & 83.5\%                         & 86.2\%   \\
                              & SOFTS        & 88.7\% & 1.114          & 87.7\%   & 88.1\%   & 88.7\% & {\color[HTML]{FF0000} 1.096}          & 87.5\%   & 87.9\%   & 88.7\% & {\color[HTML]{FF0000} 1.110}          & 87.7\%   & 88.2\%   & 88.5\% & {\color[HTML]{FF0000} \textbf{1.011}} & 83.8\%                         & 86.0\%   \\
\multirow{-4}{*}{ETTm2}       & ave          & 88.9\% & 1.114          & 87.8\%   & 88.3\%   & 89.0\% & {\color[HTML]{FF0000} 1.109}          & 87.8\%   & 88.2\%   & 88.9\% & 1.119                                 & 87.9\%   & 88.4\%   & 89.4\% & {\color[HTML]{FF0000} \textbf{1.048}} & 84.3\%                         & 87.0\%   \\\midrule
                              & itransformer & 89.9\% & 1.758          & 87.5\%   & 88.3\%   & 89.8\% & 1.758                                 & 87.0\%   & 87.5\%   & 90.4\% & 1.763                                 & 87.5\%   & 89.6\%   & 89.6\% & {\color[HTML]{FF0000} \textbf{1.747}} & 85.9\%                         & 86.6\%   \\
                              & leddam       & 88.8\% & \textbf{1.666} & 87.7\%   & 87.6\%   & 89.8\% & 1.702                                 & 88.8\%   & 87.9\%   & 90.4\% & 1.892                                 & 87.8\%   & 88.8\%   & 89.9\% & 1.705                                 & 88.7\%                         & 87.3\%   \\
                              & SOFTS        & 88.8\% & \textbf{1.624} & 87.3\%   & 86.6\%   & 90.9\% & 1.766                                 & 89.0\%   & 89.4\%   & 90.5\% & 1.879                                 & 87.4\%   & 89.6\%   & 91.3\% & 1.791                                 & 89.6\%                         & 89.5\%   \\
\multirow{-4}{*}{ETTh1}       & ave          & 89.2\% & \textbf{1.683} & 87.5\%   & 87.5\%   & 90.2\% & 1.742                                 & 88.3\%   & 88.3\%   & 90.5\% & 1.844                                 & 87.6\%   & 89.3\%   & 90.3\% & 1.748                                 & 88.1\%                         & 87.8\%   \\\midrule
                              & itransformer & 87.9\% & 1.336          & 85.7\%   & 86.9\%   & 89.7\% & 1.401                                 & 86.3\%   & 86.2\%   & 88.0\% & {\color[HTML]{FF0000} \textbf{1.258}} & 84.7\%   & 87.2\%   & 90.4\% & 1.416                                 & 85.6\%                         & 85.4\%   \\
                              & leddam       & 86.7\% & 1.133          & 82.8\%   & 85.4\%   & 87.0\% & 1.141                                 & 82.8\%   & 85.4\%   & 88.1\% & {\color[HTML]{FF0000} 1.130}          & 82.9\%   & 86.1\%   & 86.8\% & 1.119                                 & 79.0\%                         & 83.5\%   \\
                              & SOFTS        & 87.1\% & 1.113          & 82.7\%   & 85.9\%   & 87.1\% & 1.105                                 & 82.2\%   & 85.9\%   & 87.1\% & 1.103                                 & 83.2\%   & 86.2\%   & 86.5\% & 1.073                                 & 77.4\%                         & 84.3\%   \\
\multirow{-4}{*}{ETTh2}       & ave          & 87.3\% & 1.194          & 83.8\%   & 86.0\%   & 88.0\% & \textbf{1.216}                        & 83.8\%   & 85.8\%   & 87.7\% & 1.164                                 & 83.6\%   & 86.5\%   & 87.9\% & 1.203                                 & 80.7\%                         & 84.4\%   \\\midrule
                              & itransformer & 88.9\% & \textbf{1.139} & 83.5\%   & 88.3\%   & 89.3\% & 1.154                                 & 82.3\%   & 88.7\%   & 89.3\% & 1.163                                 & 84.0\%   & 88.8\%   & 89.5\% & 1.140                                 & \cellcolor[HTML]{E2EFDA}63.5\% & 87.2\%   \\
                              & leddam       & 88.4\% & 1.125          & 77.9\%   & 87.7\%   & 88.5\% & {\color[HTML]{FF0000} 1.124}          & 75.3\%   & 87.5\%   & 88.3\% & {\color[HTML]{FF0000} \textbf{1.122}} & 78.0\%   & 87.7\%   & 87.2\% & 1.063                                 & \cellcolor[HTML]{C6E0B4}52.8\% & 83.0\%   \\
                              & SOFTS        & 88.9\% & 1.274          & 84.5\%   & 88.0\%   & 90.0\% & 1.308                                 & 85.5\%   & 89.4\%   & 88.6\% & {\color[HTML]{FF0000} \textbf{1.269}} & 84.3\%   & 88.0\%   & 90.9\% & 1.326                                 & 79.8\%                         & 88.7\%   \\
\multirow{-4}{*}{PEMS03}      & ave          & 88.7\% & 1.179          & 82.0\%   & 88.0\%   & 89.3\% & 1.195                                 & 81.0\%   & 88.6\%   & 88.8\% & 1.185                                 & 82.1\%   & 88.2\%   & 89.2\% & {\color[HTML]{FF0000} \textbf{1.176}} & \cellcolor[HTML]{E2EFDA}65.3\% & 86.3\%   \\\midrule
                              & itransformer & 89.4\% & 1.146          & 81.6\%   & 88.9\%   & 90.3\% & 1.197                                 & 82.3\%   & 89.7\%   & 89.1\% & {\color[HTML]{FF0000} \textbf{1.124}} & 81.3\%   & 88.8\%   & 91.9\% & 1.236                                 & \cellcolor[HTML]{E2EFDA}64.2\% & 90.9\%   \\
                              & leddam       & 88.6\% & 1.083          & 83.3\%   & 88.0\%   & 89.3\% & 1.106                                 & 84.3\%   & 88.7\%   & 88.4\% & {\color[HTML]{FF0000} \textbf{1.071}} & 83.7\%   & 87.8\%   & 89.5\% & 1.093                                 & 76.7\%                         & 87.4\%   \\
                              & SOFTS        & 88.8\% & 1.359          & 81.0\%   & 87.9\%   & 89.2\% & 1.366                                 & 82.8\%   & 88.3\%   & 88.8\% & {\color[HTML]{FF0000} \textbf{1.334}} & 80.8\%   & 87.8\%   & 89.4\% & {\color[HTML]{FF0000} 1.354}          & 79.6\%                         & 87.1\%   \\
\multirow{-4}{*}{PEMS04}      & ave          & 88.9\% & 1.196          & 82.0\%   & 88.3\%   & 89.6\% & 1.223                                 & 83.1\%   & 88.9\%   & 88.8\% & {\color[HTML]{FF0000} \textbf{1.176}} & 81.9\%   & 88.1\%   & 90.3\% & 1.227                                 & 73.5\%                         & 88.5\%   \\\midrule
                              & itransformer & 89.9\% & 1.048          & 80.1\%   & 89.7\%   & 90.5\% & {\color[HTML]{FF0000} 1.032}          & 75.2\%   & 90.1\%   & 89.5\% & {\color[HTML]{FF0000} \textbf{1.001}} & 75.9\%   & 89.2\%   & 93.4\% & 1.103                                 & \cellcolor[HTML]{A9D08E}48.2\% & 92.2\%   \\
                              & leddam       & 89.6\% & \textbf{1.005} & 77.9\%   & 89.3\%   & 90.1\% & 1.020                                 & 75.0\%   & 89.8\%   & 89.5\% & 1.008                                 & 77.8\%   & 89.1\%   & 92.0\% & 1.058                                 & \cellcolor[HTML]{A9D08E}49.7\% & 90.7\%   \\
                              & SOFTS        & 89.5\% & \textbf{1.171} & 83.6\%   & 89.1\%   & 90.5\% & 1.205                                 & 84.5\%   & 90.1\%   & 89.2\% & 1.172                                 & 83.2\%   & 88.8\%   & 92.9\% & 1.257                                 & 74.5\%                         & 91.6\%   \\
\multirow{-4}{*}{PEMS07}      & ave          & 89.6\% & 1.075          & 80.5\%   & 89.4\%   & 90.4\% & 1.086                                 & 78.2\%   & 90.0\%   & 89.4\% & {\color[HTML]{FF0000} \textbf{1.060}} & 79.0\%   & 89.0\%   & 92.8\% & 1.139                                 & \cellcolor[HTML]{C6E0B4}57.5\% & 91.5\%   \\\midrule
                              & itransformer & 89.1\% & \textbf{1.094} & 80.3\%   & 88.6\%   & 89.7\% & 1.127                                 & 79.6\%   & 88.8\%   & 89.1\% & 1.103                                 & 80.3\%   & 88.6\%   & 90.3\% & 1.128                                 & \cellcolor[HTML]{E2EFDA}69.4\% & 88.0\%   \\
                              & leddam       & 88.3\% & 1.100          & 76.2\%   & 87.6\%   & 89.0\% & 1.120                                 & 75.0\%   & 87.8\%   & 88.2\% & 1.101                                 & 76.2\%   & 87.5\%   & 88.7\% & {\color[HTML]{FF0000} \textbf{1.090}} & \cellcolor[HTML]{E2EFDA}64.6\% & 85.7\%   \\
                              & SOFTS        & 88.6\% & 1.235          & 76.3\%   & 87.8\%   & 89.6\% & 1.260                                 & 74.3\%   & 88.5\%   & 88.2\% & {\color[HTML]{FF0000} \textbf{1.205}} & 74.8\%   & 87.7\%   & 89.9\% & 1.243                                 & \cellcolor[HTML]{E2EFDA}67.9\% & 86.9\%   \\
\multirow{-4}{*}{PEMS08}      & ave          & 88.7\% & 1.143          & 77.6\%   & 88.0\%   & 89.4\% & 1.169                                 & 76.3\%   & 88.4\%   & 88.5\% & {\color[HTML]{FF0000} 1.137}          & 77.1\%   & 87.9\%   & 89.6\% & 1.154                                 & \cellcolor[HTML]{E2EFDA}67.3\% & 86.9\%   \\\midrule
1st Count                     &              & \multicolumn{4}{c|}{12}                        & \multicolumn{4}{c|}{3}                                                & \multicolumn{4}{c|}{20}                                               & \multicolumn{4}{c}{12}                                                                     \\ \bottomrule
\end{tabular}

\end{table}

Overall, employing LSTM as the error quantile prediction model leeds to modest improvements over the original method in many scenarios. Across our comprehensive evaluation of 48 test cases, it demonstrated particularly competitive performance, achieving optimal results in 20 instances. However, it's important to note that these improvements remain relatively minor in magnitude.

The other two extension methods - the ECI-based interval length adjustment and the adaptive learning rate approach - showed limited effectiveness in enhancing the baseline performance. In fact, these modifications frequently introduced negative effects.
\section{Visualization of confidence intervals}
We plot the predicted confidence intervals based on SOFTS \cite{han2024softs} in different dataset in the following figures. In these figures, the right subfigure in the lower row represent our method. Other subfigures represent TQA-E, ACI, ECI, LPCI and PIC respectively.

\begin{figure}[h]
    \centering
    \includegraphics[width=1.\linewidth]{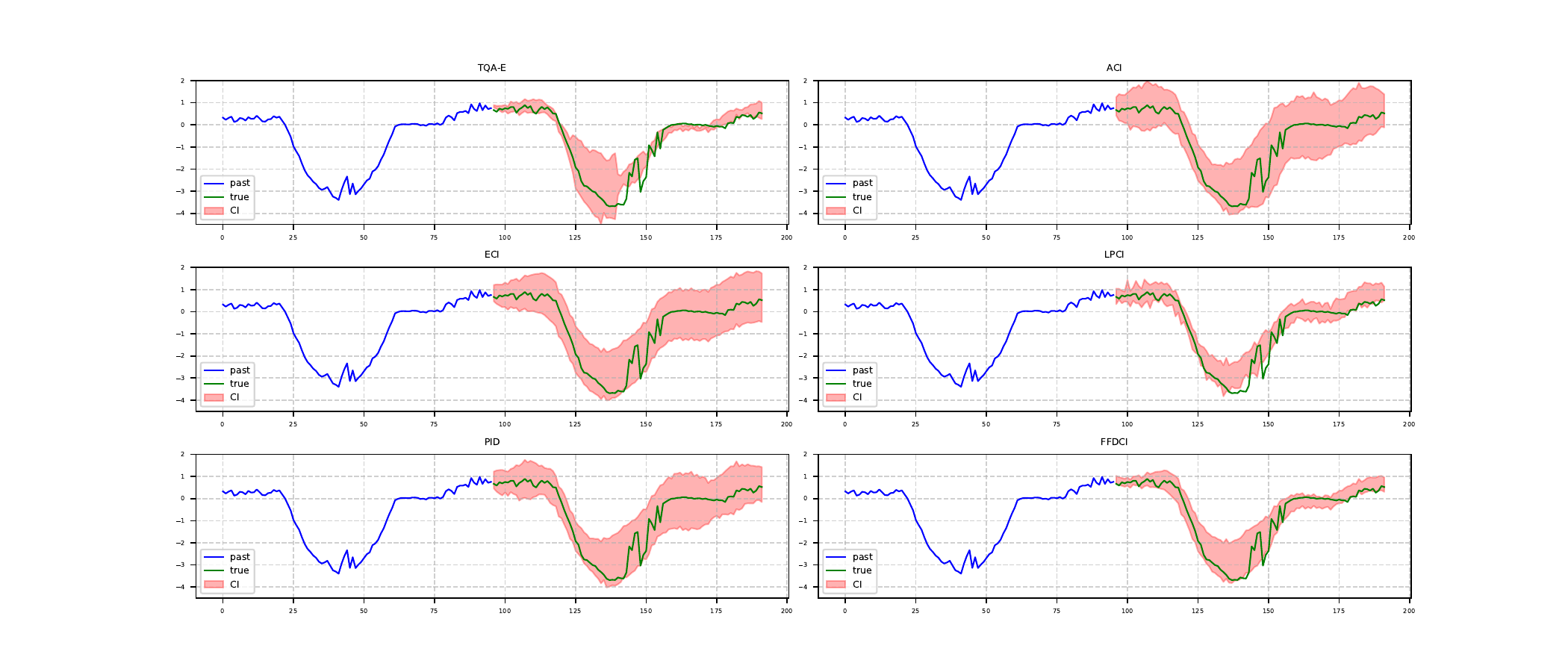}
    \caption{Predicted confidence intervals in ETTm1 dataset}
    \label{fig:et1}
\end{figure}
\vspace{-10pt} % 减少图片之间的间距

\begin{figure}[h]
    \centering
    \includegraphics[width=\linewidth]{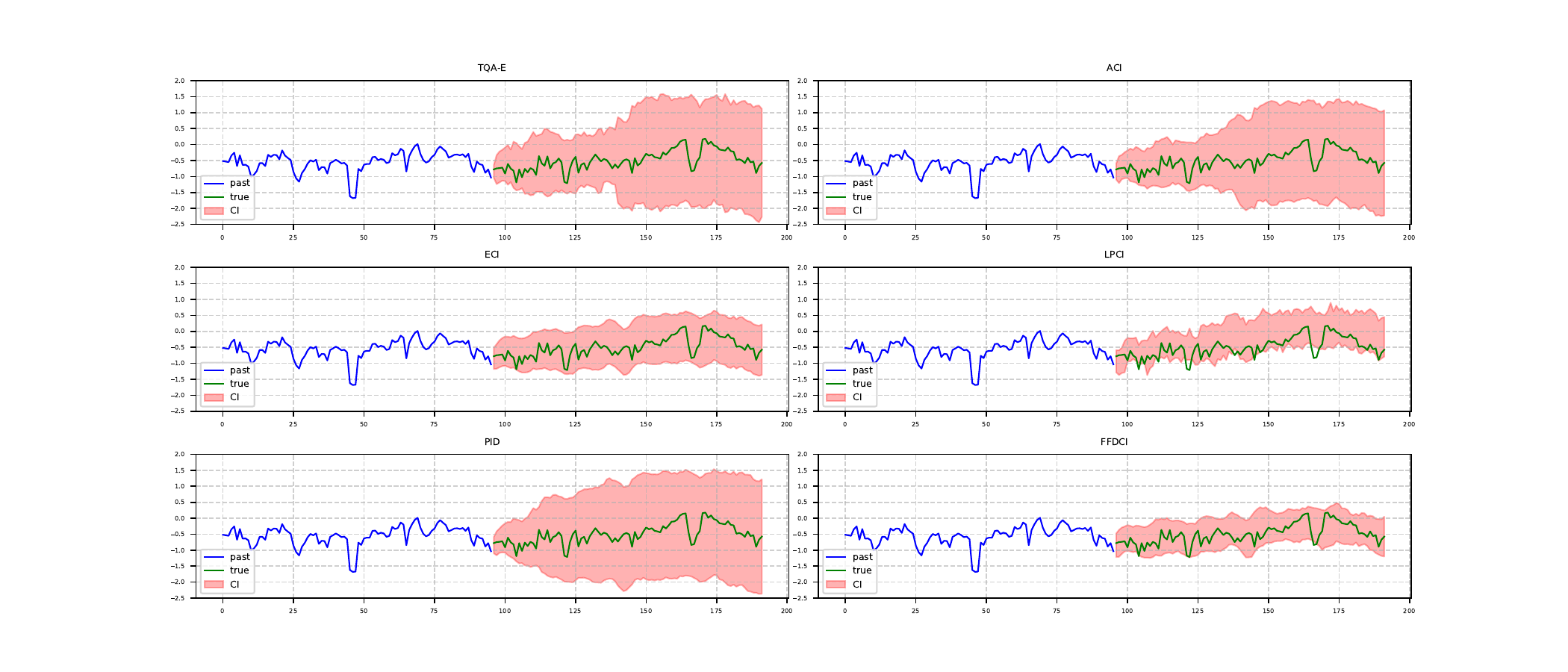}
    \caption{Predicted confidence intervals in ETTm2 dataset}
    \label{fig:et2}
\end{figure}
\vspace{-10pt}

\begin{figure}[h]
    \centering
    \includegraphics[width=\linewidth]{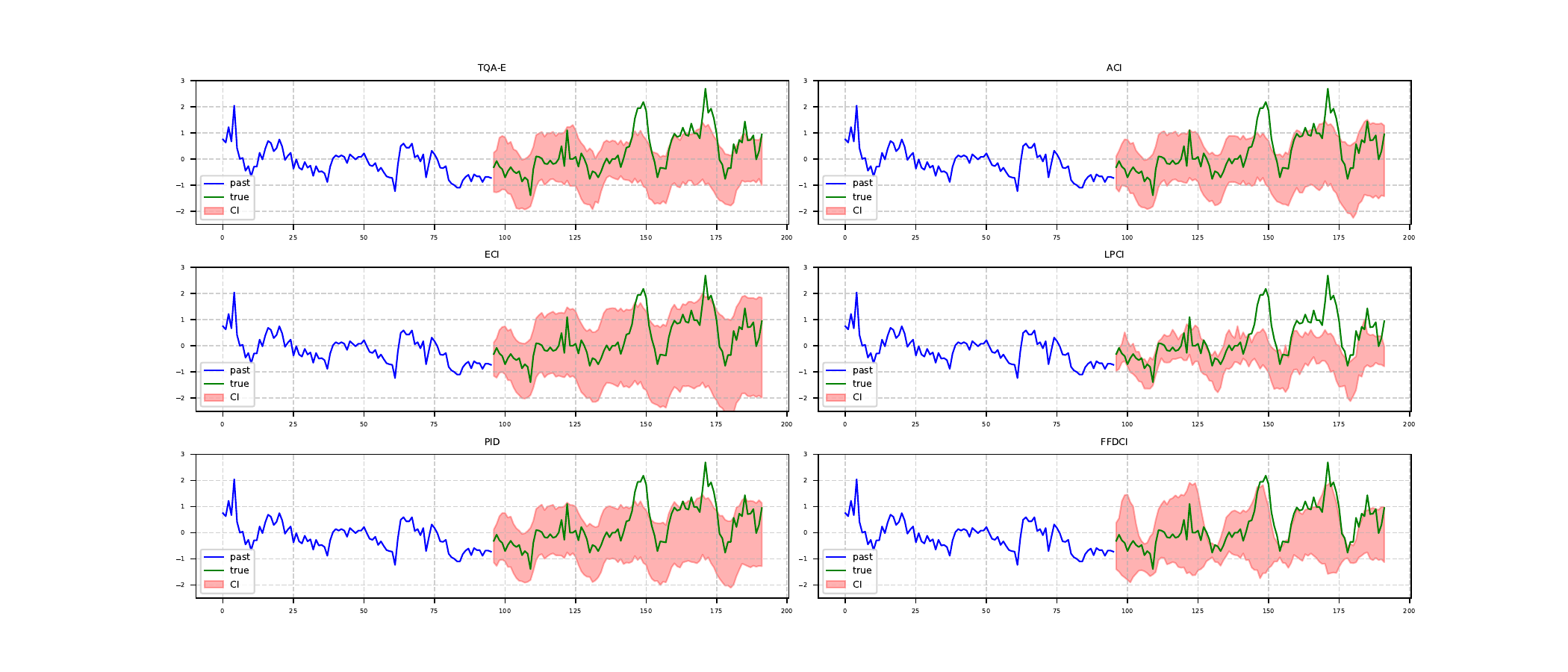}
    \caption{Predicted confidence intervals in ETTh1 dataset}
    \label{fig:et3}
\end{figure}
\vspace{-10pt}

\begin{figure}[h]
    \centering
    \includegraphics[width=\linewidth]{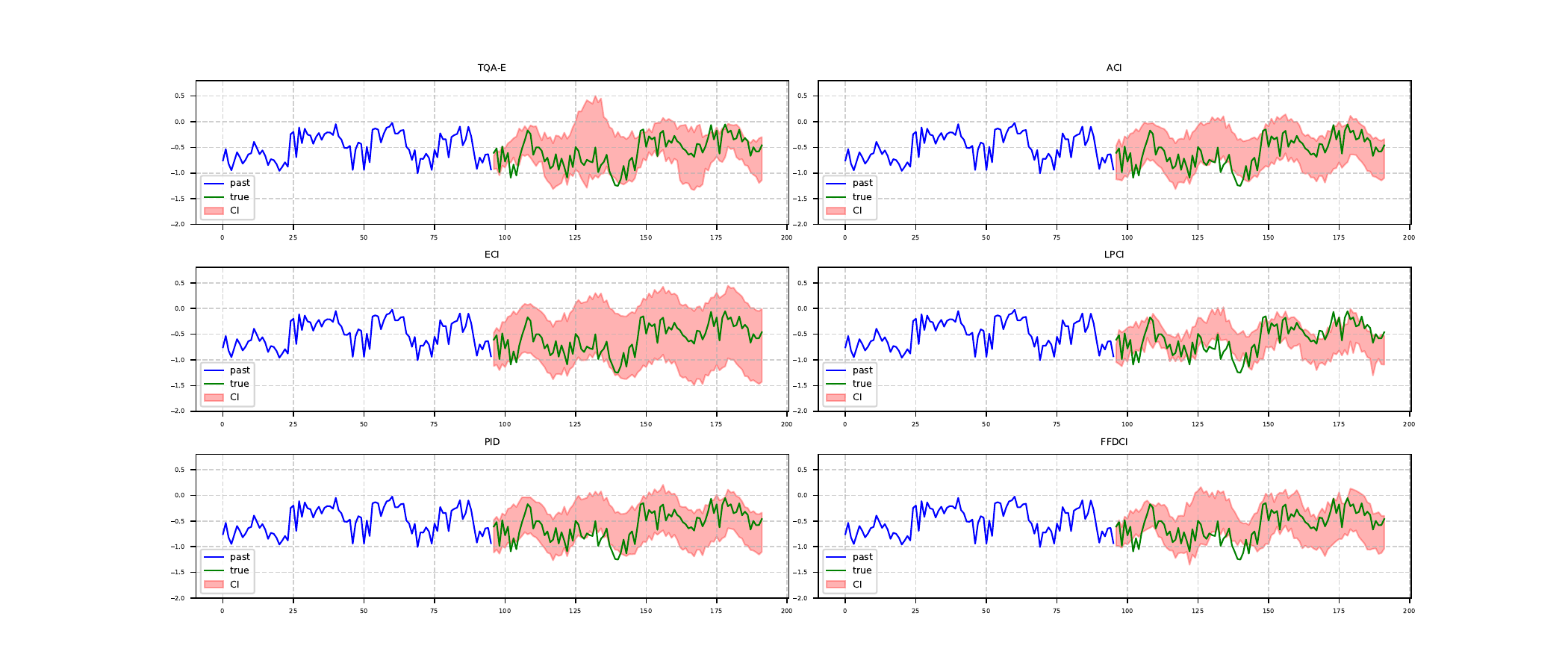}
    \caption{Predicted confidence intervals in ETTh2 dataset}
    \label{fig:et4}
\end{figure}
\begin{figure}
    \centering
    \includegraphics[width=\linewidth]{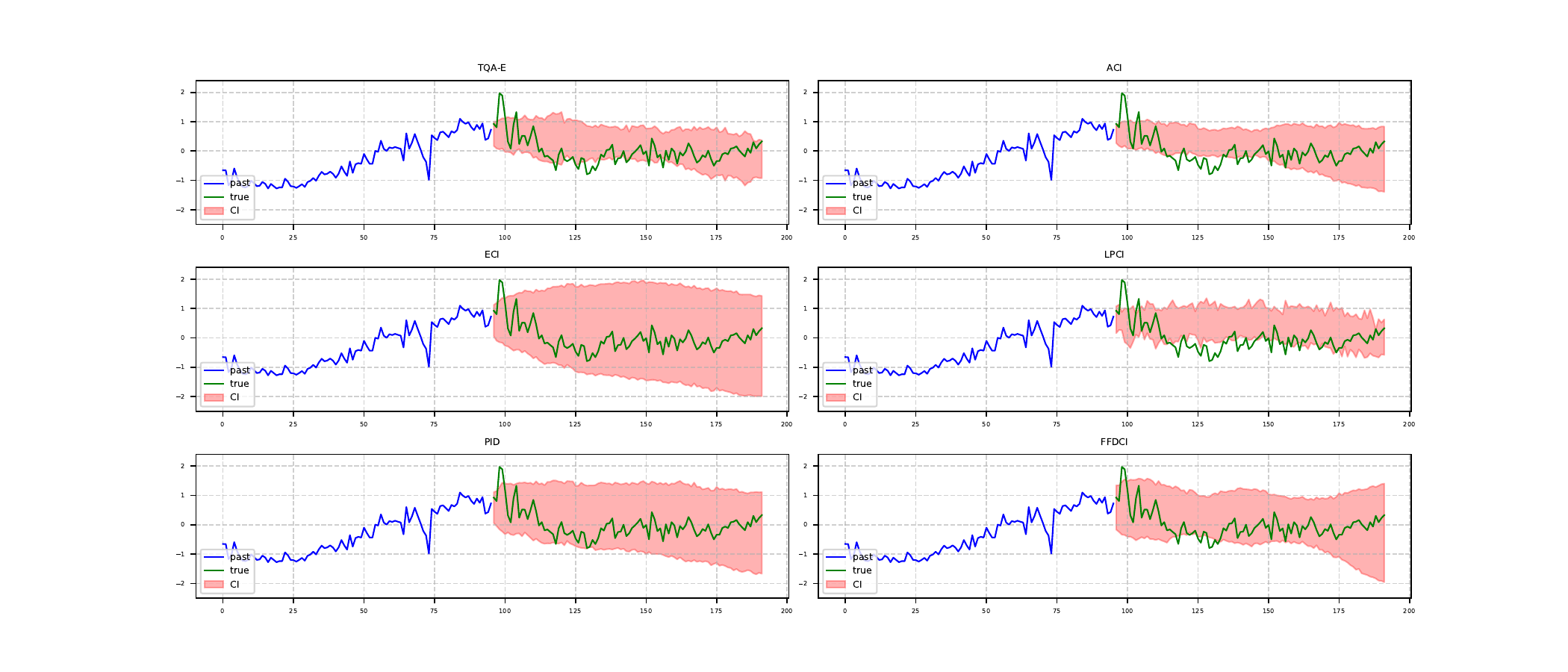}
    \caption{Predicted confidence intervals in weather dataset}
    \label{fig:enter-label}
\end{figure}
\begin{figure}
    \centering
    \includegraphics[width=\linewidth]{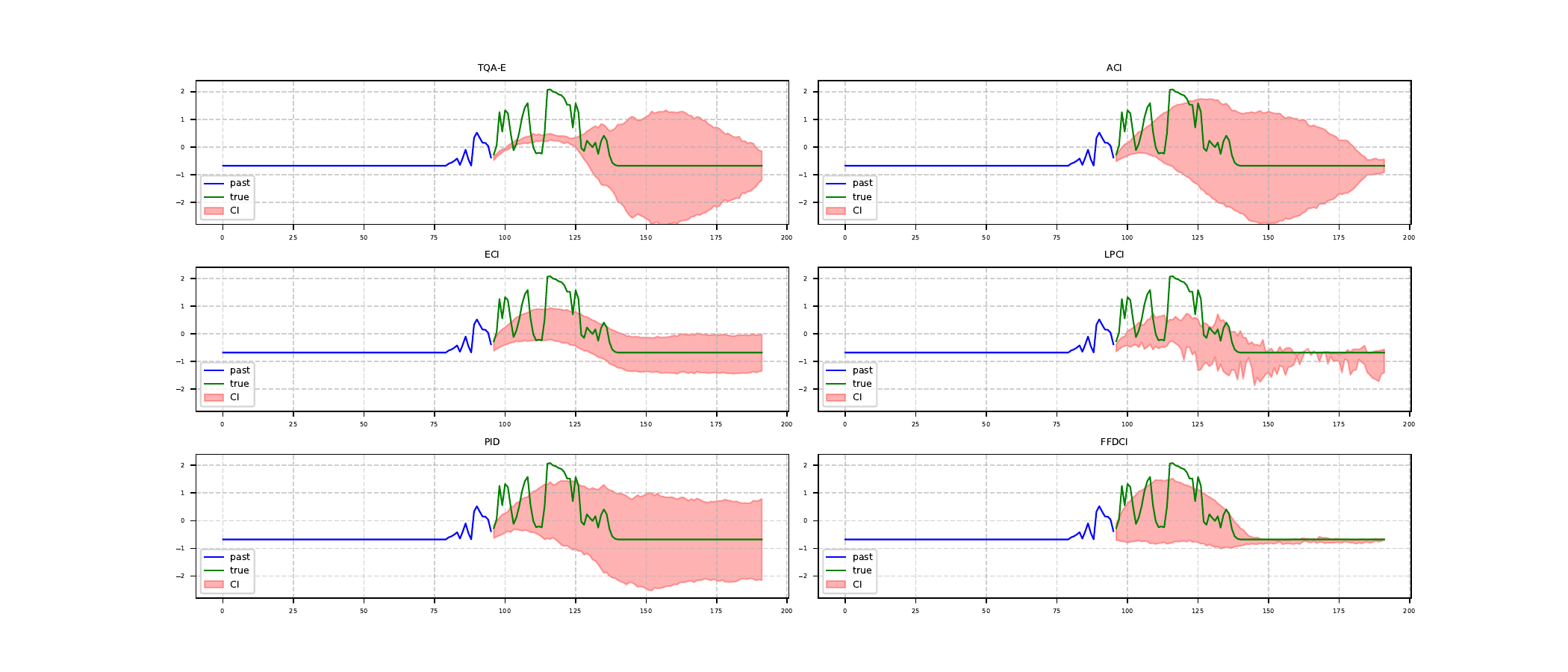}
    \caption{Predicted confidence intervals in solar dataset}
    \label{fig:enter-label}
\end{figure}
\begin{figure}
    \centering
    \includegraphics[width=\linewidth]{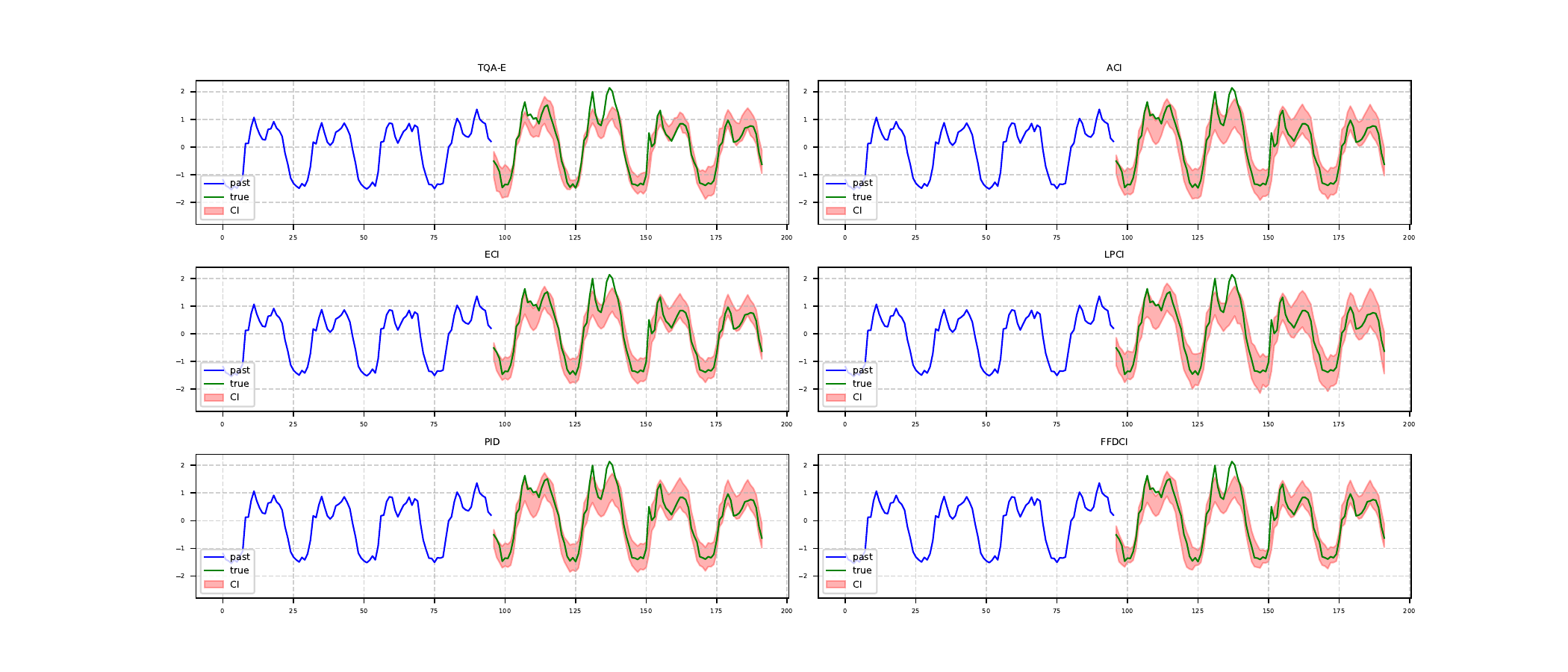}
    \caption{Predicted confidence intervals in electricity dataset}
    \label{fig:enter-label}
\end{figure}
\begin{figure}
    \centering
    \includegraphics[width=\linewidth]{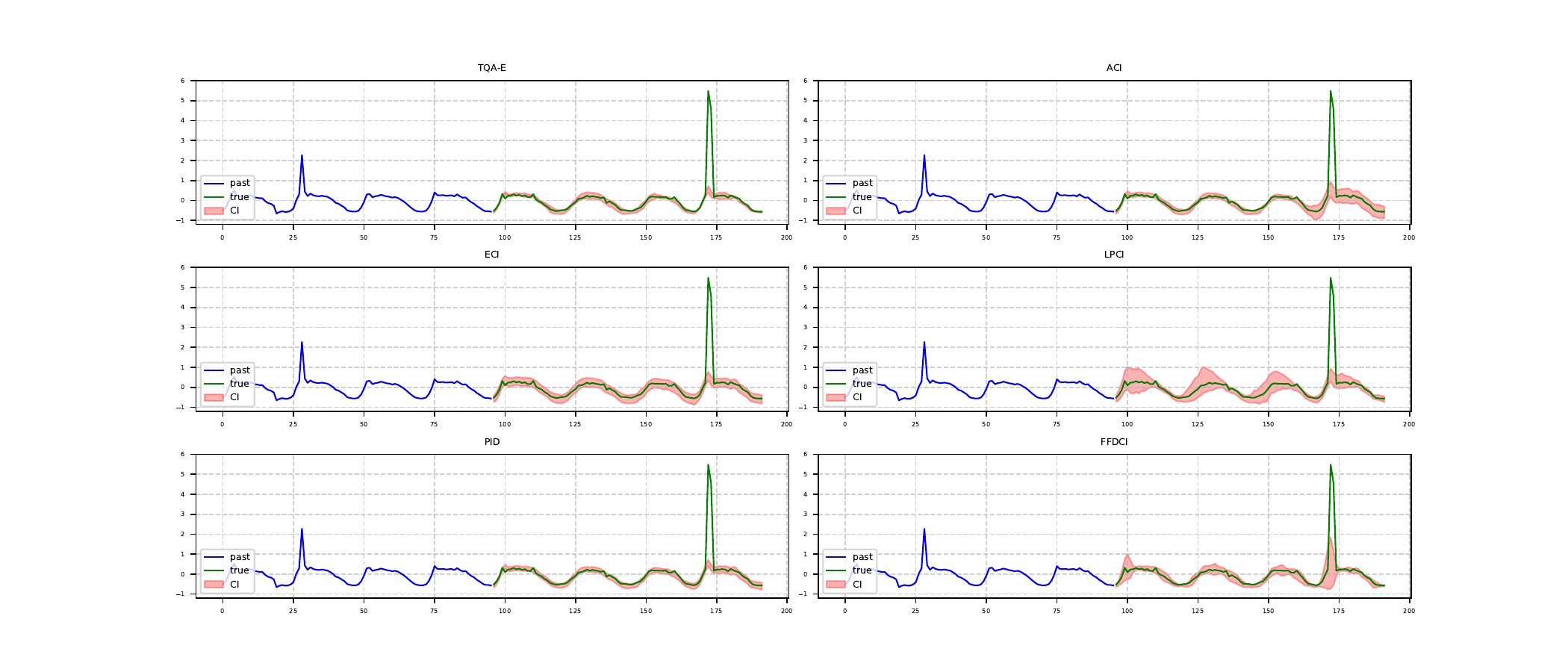}
    \caption{Predicted confidence intervals in traffic dataset}
    \label{fig:enter-label}
\end{figure}
\begin{figure}
    \centering
    \includegraphics[width=\linewidth]{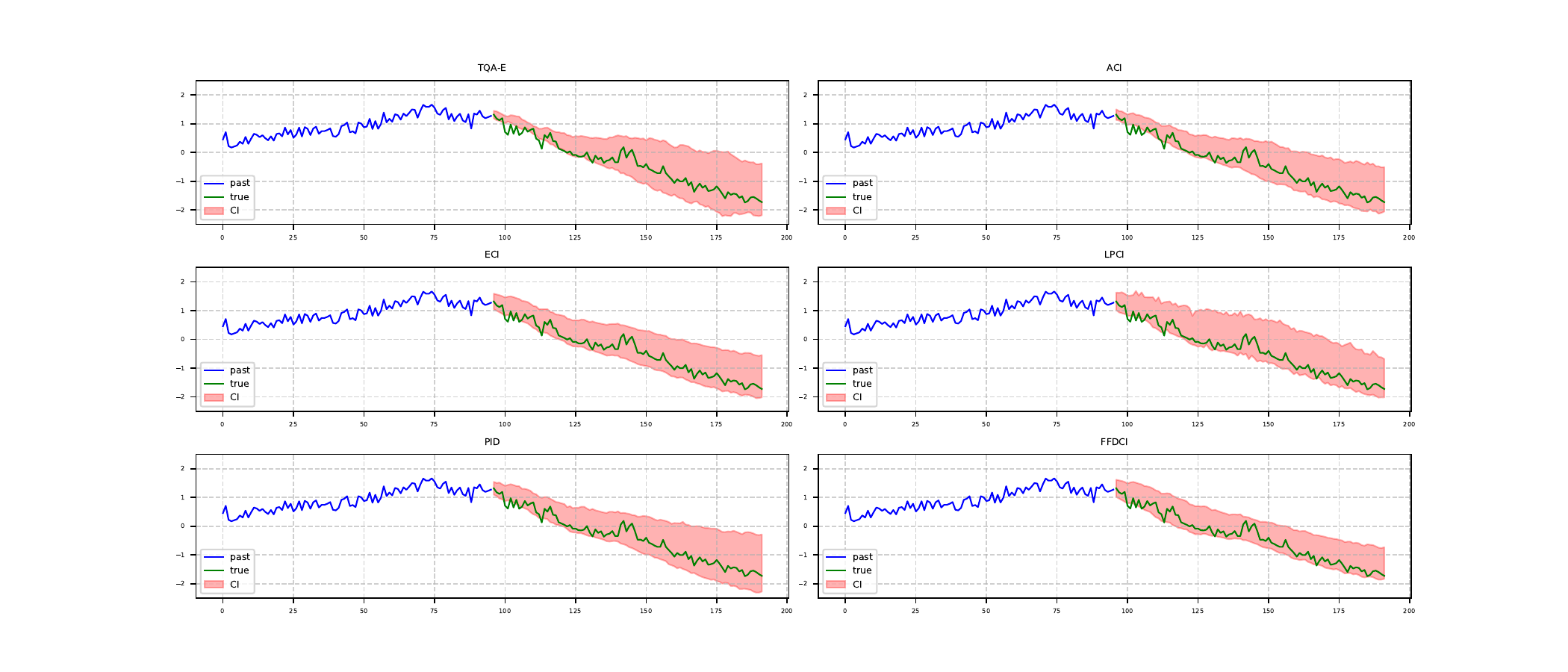}
    \caption{Predicted confidence intervals in PEMS03 dataset}
    \label{fig:enter-label}
\end{figure}
\begin{figure}
    \centering
    \includegraphics[width=\linewidth]{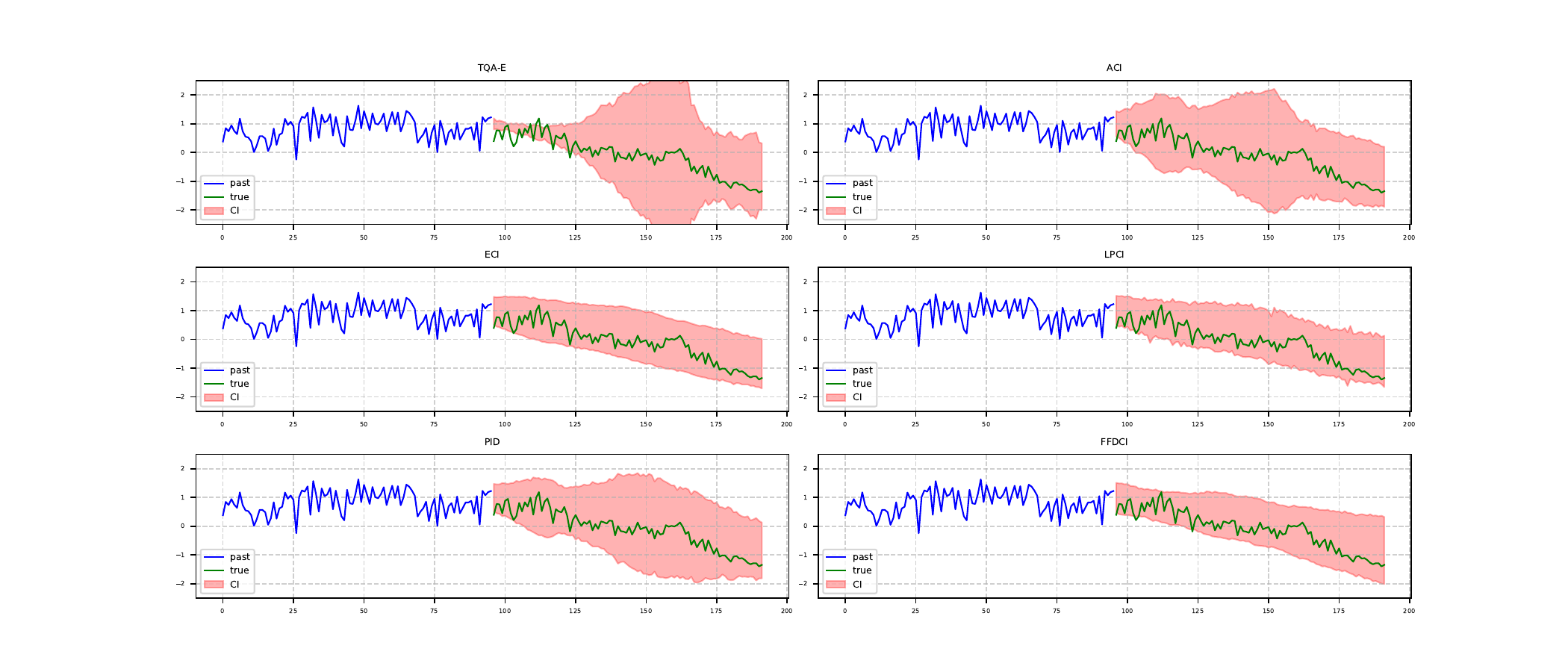}
    \caption{Predicted confidence intervals in PEMS04 dataset}
    \label{fig:enter-label}
\end{figure}
\begin{figure}
    \centering
    \includegraphics[width=\linewidth]{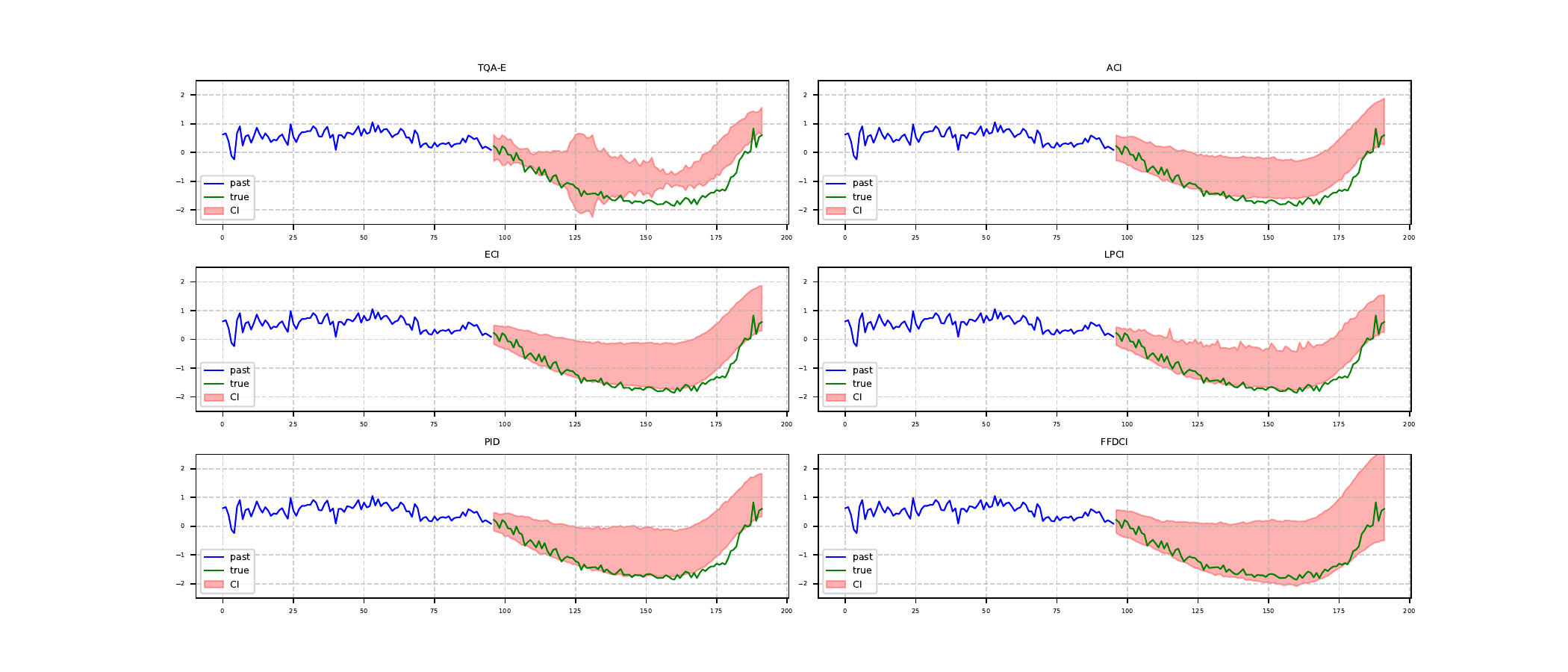}
    \caption{Predicted confidence intervals in PEMS07 dataset}
    \label{fig:enter-label}
\end{figure}
\end{document}